\documentclass{article}
\usepackage[utf8]{inputenc}
\usepackage{amsmath,amssymb,amsthm,amsfonts}
\usepackage{times}
\usepackage{graphicx}
\usepackage{color}
\usepackage{bbm}
\usepackage{latexsym}
\usepackage{dsfont}
\usepackage{verbatim}
\usepackage{booktabs,array}
\usepackage{tikz}
\usepackage{subfig}
\usepackage{xfrac}
\usepackage[ruled,vlined]{algorithm2e}
\usepackage{amstext} 
\newcolumntype{L}{>{$}l<{$}} 

\usepackage{hyperref}
\definecolor{bl}{RGB}{30,30,150}
\hypersetup{colorlinks=true, citecolor=bl, linkcolor=bl, urlcolor=bl}

\usetikzlibrary{tikzmark,calc,,arrows,shapes,decorations.pathreplacing}
\tikzset{every picture/.style={remember picture}}
\usetikzlibrary{fit,shapes.misc}
\usetikzlibrary{positioning,backgrounds}

         \def\L{\Lambda }       

       \newcommand{\sM}{{\cal M}}

    \newcommand{\R}{{\mathbb R}}

\newcommand{\bq}{\begin{equation}}
\newcommand{\eq}{\end{equation}}
\newcommand{\bpm}{\begin{pmatrix}}
	\newcommand{\epm}{\end{pmatrix}}

\usepackage{siunitx}
\usepackage{enumitem}

\usepackage{nicefrac}
\usepackage{graphicx}
\usepackage[margin=3.5cm]{geometry}
\usepackage{color}
\usepackage{xcolor}
\usepackage{bm}
\usepackage[normalem]{ulem}

\newcommand{\RR}{\mathbb{R}}
\newcommand{\CC}{\mathbb{C}}

\newcommand{\eg}{{\em e.g.}}
\newcommand{\ie}{{\em i.e.}}

\newcommand{\x}{\mathsf{x}}
\newcommand{\Jac}{\operatorname{Jac}}
\newcommand{\y}{\mathsf{y}}
\renewcommand{\L}{\mathcal{L}}
\renewcommand{\Pi}{\pi}
\newcommand{\rrmp}{rrmp}
\newcommand{\Crit}{\mathrm{Crit}}

\usepackage{graphbox}
\usepackage{cleveref}

\newtheorem{theorem}{Theorem}[section]
\newtheorem{corollary}[theorem]{Corollary}
\newtheorem{lemma}[theorem]{Lemma}
\newtheorem{proposition}[theorem]{Proposition}
\theoremstyle{definition}
\newtheorem{definition}[theorem]{Definition}

\newtheorem{example}[theorem]{Example}
\newtheorem{remark}[theorem]{Remark}

\renewenvironment{abstract}%
{\begin{trivlist}\item[]{\emph{Abstract.}}\ }
	{\end{trivlist}}

\newenvironment{keywords}%
{\begin{trivlist}\item[]{\emph{Key words.}}\ }
	{\end{trivlist}}

\newenvironment{AMS}%
{\begin{trivlist}\item[]{\emph{AMS subject classifications.}}\ }
	{\end{trivlist}}

\newcommand{\email}[1]{#1}

\title{Geometry of Linear Convolutional Networks}

\author{Kathl\'en Kohn\thanks{Department of Mathematics, KTH, Stockholm, Sweden 
(\email{kathlen@kth.se}).}
\and Thomas Merkh\thanks{Department of Mathematics, UCLA, CA, USA 
(\email{tmerkh@g.ucla.edu}).}
\and Guido Mont\'ufar\thanks{Departments of Mathematics and Statistics, UCLA, CA, USA; Max Planck Institute for Mathematics in the Sciences, Leipzig, Germany 
(\email{montufar@math.ucla.edu}).}
\and Matthew Trager\thanks{Amazon, New York, NY, USA 
(\email{matthewtrager@gmail.com}). This work was done outside of Amazon.} 
  }

\date{\today}

\begin{document}

\maketitle

\begin{abstract}%
We study the family of functions that are represented by a linear convolutional neural network~(LCN). These functions form a semi-algebraic subset of the set of linear maps from input space to output space. In contrast, the families of functions represented by fully-connected linear networks form algebraic sets. We observe that the functions represented by LCNs can be identified with polynomials that admit certain factorizations, and we use this perspective to describe the impact of the network's architecture on the geometry of the resulting function space. We further study the optimization of an objective function over an LCN, analyzing critical points in function space and in parameter space, and describing dynamical invariants for gradient descent. Overall, our theory predicts that the optimized parameters of an LCN will often correspond to \emph{repeated filters} across layers, or filters that can be decomposed as repeated filters. We also conduct numerical and symbolic experiments that illustrate our results and present an in-depth analysis of the landscape for small architectures. 
\end{abstract}

\begin{keywords}
Function space description of neural networks, linear network, Toeplitz matrix, circulant matrix, algebraic statistics, Euclidean distance degree, semi-algebraic set, gradient flow, discriminant, critical point, tensor. 
\end{keywords}

\begin{AMS}
68T07,  
14P10,  
14J70,  
90C23,  
62R01.  
\end{AMS}

\section{Introduction}

A neural network is a parameterized family of functions. 
The class of representable functions, also known as the \emph{neuromanifold} or \emph{function space}, is determined by the architecture of the network. Optimizing an objective function over such a parametrized set usually depends on both the set and the parametrization. 
The objective function is typically convex in function space but non-convex in parameter space. This observation has guided numerous recent works on parameter optimization in neural networks; see, \eg, 
\cite{NEURIPS2018_5a4be1fa, 
du2018gradient,             
pmlr-v97-allen-zhu19a,      
pmlr-v119-dukler20a}.       
However, the geometry of function space is not well understood in general. 
Fully-connected networks, also known as dense networks, have been studied in significantly more detail than convolutional networks, which are the topic of this article. Whereas weight matrices in fully-connected networks have independent and unconstrained entries, weight matrices in convolutional networks are sparse and have repeated entries. 

Neural networks with linear activation functions, called linear networks, represent linear functions as compositions of linear functions. 
An important motivation for studying such networks is to better understand the effects of the parameterization when optimizing an objective function. 
We will see that it also can help us better understand the effects that the geometry of function space has on the optimization problem. 
Although neural networks are usually presented in a parameterized form, it is in principle possible to characterize the representable functions implicitly as the solutions to certain constraints. 
Such descriptions have been particularly fruitful in algebraic statistics, \eg, \cite{sullivant2018algebraic,SeigalMontufar,722fd1ac5816467ebadbd10f15ff2a74,10.1007/978-3-030-43120-4_29,ccelik2020wasserstein}. 
In the case of fully-connected linear networks, the function space is characterized by rank constraints depending on the layer widths. A study of critical points and local optimizers in parameter and function space for fully-connected linear networks has been carried out in \cite{geometryLinearNets}. 
Here we follow a similar perspective,
analyzing the function space and the optimization of parameters for
\emph{linear convolutional networks} (LCNs).

\paragraph{Main results} 
In \Cref{section:LCNs} we introduce LCNs and present our general setup. 
Each layer of an LCN is represented by a convolutional matrix or, more generally, a convolutional tensor. 
We observe that composing convolutional layers corresponds to polynomial multiplication. 

In \Cref{sec:discriminants} we recall some general properties of discriminants and multiplicity root patterns for polynomials, which we use in later sections.

In \Cref{sec:geometry} we show that the function space of an LCN (\ie, the set of linear maps that can be represented by the network) is a semi-algebraic set, whose natural ambient space is the vector space of convolutional tensors of a certain format. 
In contrast to fully-connected linear networks, whose function space is cut out by polynomial equations only, the function space of LCNs is determined by polynomial equations and inequalities. 
For one-dimensional convolutions with stride one, the function space is a full-dimensional subset of its natural ambient space and its boundary is contained in the discriminant hypersurface of univariate polynomials of a certain degree. The function space is equal to its natural ambient space if and only if at most one of the filters has even size. 
The latter case is interesting since optimizing a convex loss in function space reduces to a convex optimization problem on a vector space. 
For higher-dimensional convolutions or larger strides, the function space is typically a lower-dimensional subset of its natural ambient space. 

In \Cref{sec:optimization-function} we study the optimization problem for a general loss and one-dimensional convolutions with stride one. 
Different LCNs can have the same function space, but different parameterizations can influence the optimization behavior. 
Hence, we distinguish \emph{pure critical points} given by the optimization on function space from \emph{spurious critical points} induced only by the parameterization. 
We show that the critical points in parameter space are either global minima or correspond to polynomials with repeated roots. The latter actually correspond to critical points of the loss in function space restricted to subsets of polynomials with particular real root multiplicity patterns that we characterize combinatorially depending on the LCN architecture. 
We further describe the gradient flow dynamics in parameter space and provide invariants that allow us to formulate it as a local Riemannian gradient in function space. 

In \Cref{sec:squareloss} we take a closer look at the the square loss. 
We provide upper bounds on the number of critical points in function space in terms of \emph{Euclidean distance degrees}, and observe that for certain normalizations of the training data, the square loss becomes the \emph{Bombieri norm}, which is expected to have fewer critical points. 

In \Cref{sec:experiments} we conduct numerical experiments illustrating our theoretical results. 
These demonstrate how the type of training data, the geometry of function space, and the parametrization lead to different types of critical points and solutions found using gradient descent.

\paragraph{Overview of previous works on linear networks}

\emph{Loss surface.}  
The properties of the optimization problem in linear networks have been studied in \cite{BALDI198953,NIPS1988_123,10.1109/72.392248}. 
For the square loss in fully-connected linear networks there is a unique local and global minimum, up to equivalence, and all other critical points are saddles. 
A proof of this statement for the deep case was given in \cite{NIPS2016_6112}. 
For deep residual linear networks, \cite{DBLP:conf/iclr/HardtM17} showed that the square loss has no critical points other than global minima. 
Further, \cite{DBLP:journals/corr/LuK17} and \cite{47812} argue that depth alone does not create bad local minima, although it induces a non-convex loss surface. 
The analytic form of the critical points for the square loss in fully-connected linear networks was discussed in \cite{zhou2018critical}. 
For arbitrary differentiable convex losses and fully-connected linear networks having layers at least as wide as the input or output layers, \cite{pmlr-v80-laurent18a} showed that all local minima are global. 
A geometric analysis was given in \cite{geometryLinearNets}, showing that the absence of non-global local minima in linear networks is expected for arbitrary smooth convex losses only if the architecture can express all linear maps (as in \cite{pmlr-v80-laurent18a}) or for the quadratic loss and arbitrary architectures (as in \cite{NIPS2016_6112}). 
Taking an algebraic standpoint, \cite{9397294} studied fully-connected linear networks and upper bounds on the number of critical points of the square loss with weight norm regularization. 
The above works focus mostly on fully-connected networks. 
A paper dating back to 1989~\cite{NIPS1988_123} asked whether non-global local minima for the square loss exist in the case non-fully-connected (\ie, locally-connected) multi-layer linear networks, and this seems to have remained an open question.
We show that LCNs can indeed have non-trivial local minima in function space and in parameter space.

\emph{Gradient dynamics.} 
An overview of non-convex optimization in low-rank matrix factorization (two-layer linear networks) was given in \cite{8811622}. 
Exact solutions to the nonlinear gradient dynamics for the square loss in fully-connected deep linear networks with orthogonal input data have been obtained in \cite{DBLP:journals/corr/SaxeMG13}. 
Also for the square loss in fully-connected linear networks, \cite{arora2018a} showed linear convergence to a global minimum when the hidden layers are at least as wide as the input or output layers, weight matrices at initialization are approximately balanced, and the initial loss is less than for any rank deficient solution. 
The work \cite{pmlr-v80-arora18a} shows that 
depth amounts to a pre-conditioner 
which may accelerate convergence. 
A detailed analysis of gradient flow optimization and convergence for the square loss in fully-connected linear networks was presented in \cite{DBLP:journals/corr/abs-1910-05505}. Using a balancedness condition they obtain a Riemannian gradient in function space and show that the flow always converges to a critical point, which for almost all initializations is the global optimum in function space restricted to the set of rank-$k$ matrices for some $k$. 
We characterize gradient flow invariants in LCNs and obtain corresponding descriptions in function space. The critical points can also be characterized in function space, in terms of real root multiplicity patterns rather than rank constraints. 

\emph{Linear convolutional networks and implicit bias of gradient descent.}  
LCNs have been~studied in \cite{dft} for the case of 1D convolutions with circulant weight matrices (with stride~1) having full-width filters and a single output. Such networks can express any linear map. 
That work focuses on binary classification with linearly separable data and asks which of the infinitely many linear classifiers with zero loss are obtained by gradient descent. 
For the exponential loss, it is shown that for fully-connected networks gradient descent converges to the hard margin support vector machine linear classifier independently of the depth of the network (as was previously observed for the direct parameterization in \cite{JMLR:v19:18-188}), whereas for convolutional networks it is biased towards linear classifiers that are sparse in frequency domain depending on the depth. 
In contrast, we consider LCNs with arbitrary filter sizes (and also provide first geometric insights for LCNs with larger strides or higher-dimensional convolutions). 
As we will see, the filter sizes play an important role. 
Our results suggest that gradient descent in LCNs will be biased towards solutions with repeated filters (or repeated factors of filters), even when these are not global optimizers of the training objective. 

\paragraph{Notation}
Throughout the paper, we use zero-based indices for our vectors and matrices. 
For $n\in\mathbb{Z}_{\ge 0}$, we write $[n]=\{0,\ldots, n-1\}$. 
We write $\RR[\x,\y]_{d}$ (resp. $\CC[\x,\y]_{d}$) for the space of bivariate homogeneous polynomials of degree $d$ with real (resp.  complex) coefficients.

\section{Linear convolutional networks} 
\label{section:LCNs}

A \emph{feedforward neural network} is a family of functions $f: \RR^{d_0} \rightarrow \RR^{d_L}$ that are compositions of linear and nonlinear maps:  
\begin{equation}\label{eq:feedforward}
f(x) = (\alpha_L \circ \rho \circ \alpha_{L-1} \circ \rho \circ \ldots  \rho \circ \alpha_1)(x).
\end{equation}
Here $\alpha_l: \RR^{d_{l-1}} \rightarrow \RR^{d_{l}}$ is an affine map and each $\rho: \RR^{d_*} \rightarrow \RR^{d_*}$ is a (typically nonlinear) \emph{activation map} that acts on each element of a vector. The composition of consecutive affine and activation maps is called a \emph{layer} of the network. 
The coefficients of the affine functions $\alpha_i$ are called the \emph{weights} and \emph{biases} of the network and serve as trainable parameters.

A \emph{linear network} is a feedforward network where the activation function $\rho$ is the identity. In this setting, the end-to-end map $f$ in~\cref{eq:feedforward} is affine, however the parameterization of the representable functions provided by the network's weights is nonlinear. Thus, linear networks preserve some of the nonlinear aspects of general feedforward networks while being theoretically more tractable. 

A \emph{linear convolutional network (LCN)} is a linear network with the additional constraint that the affine functions in~\cref{eq:feedforward} have a convolutional structure. The convolutional structure can be seen as constraining general affine maps to satisfy certain linear equations, corresponding to conditions of \emph{restricted connectivity} (some weight coefficients must be zero) and \emph{weight sharing} (some weight coefficients must be repeated). Although an LCN function $f$ can also be seen as a particular fully-connected linear network, the family of all representable functions using an LCN architecture is a strict subset of the family of representable functions using a fully-connected architecture.
In particular, the study of LCNs cannot be reduced to the study of general linear networks. We will point out several important differences between the geometry and optimization properties of LCNs and those of fully-connected networks.

In this paper, we consider convolutional networks with a single filter per layer and no biases. For most of our discussion, we focus on convolutions applied to one-dimensional signals, which means that the input $x$ of the network function is a vector (we discuss the case where $x$ is an image or higher-dimensional in \Cref{sec:higher_dimension}, but we defer general definitions to that section). In this setting, a \emph{convolution} is a map $\alpha_l: \RR^{d_{l-1}} \rightarrow \RR^{d_{l}}$ of the form 
\begin{equation}
(\alpha_l(x))_i = \sum_{j \in [k_l]} w_{l,j} \cdot x_{is_l+j},
\label{eq:convolutionDefinition}
\end{equation}
where $w_l = (w_{l,0},\ldots,w_{l,k_l-1}) \in \RR^{k_l}$ is a \emph{filter} of \emph{width} $k_l$ and $s_l \in \mathbb{Z}_{>0}$ is the $\emph{stride}$ associated with the convolution.\footnote{The expression in~\cref{eq:convolutionDefinition} is actually the \emph{cross-correlation} between $x$ and $w_l$ or equivalently the convolution of $x$ with the \emph{adjoint filter} $w^*_l = (w_{l,k_l-1},\ldots,w_{l,0})$. We refer to it simply as a ``convolution'' following common machine learning terminology. 
Exchanging all filters with their adjoints would have no effect on our discussion.} 
We are assuming ``no padding'' in the way the boundaries of the domain are handled. 
This tacitly assumes that input dimension and stride fit with the output dimension. 
We compare this approach with using ``circular padding'' in \Cref{subsec:toep_vs_circulant} below. 

The linear map $\alpha_l$ in~\cref{eq:convolutionDefinition} can be represented as a generalized Toeplitz matrix (generalized in the sense that standard Toeplitz matrices would correspond to stride $s_l=1$). For example,
for stride $s_l = 2$, filter size $k_l=3$, and input size $d_{l-1}=7$, the matrix takes the form 
\begin{equation}
W_l = 
\begin{bmatrix}
w_{l,0}&w_{l,1}&w_{l,2}&&&\\
&&w_{l,0}&w_{l,1}&w_{l,2}&\\
&&&&w_{l,0}&w_{l,1}&w_{l,2}\\
\end{bmatrix}. 
\label{eq:toeplitz}
\end{equation}
We will call generalized Toeplitz matrices \emph{convolutional matrices}. 
For convolutional matrices as in~\cref{eq:toeplitz}, the input and output sizes $d_{l-1}$, $d_{l}$, filter size $k_l$, and stride $s_l$ are related by
\begin{equation}
\label{eq:dimensionConstraintsToeplitz}
    d_{l} = \frac{d_{l-1} - k_l}{s_l} + 1.
\end{equation} 

An $L$-layer LCN function is given by $f(x) = \overline W x$, where $\overline W \in \RR^{d_L \times d_0}$ can be written as $\overline{W} = W_LW_{L-1}\cdots W_{2}W_{1}$ and each $W_l \in \R^{d_{l} \times d_{l-1}}$ is a convolutional matrix.
In this work, we study the \emph{function space}, \ie, the set of linear maps that can be represented by an LCN as the parameters vary:
\begin{equation*}
    \sM_{\bm{d},\bm{k},\bm{s}} = \Big\{ \overline{W} \in \R^{d_L \times d_0} \colon  \overline{W} = \prod_{l=1}^L W_l, \; W_l \in \R^{d_{l} \times d_{l-1}} \text{ convolutional} \Big\}.
\label{eq:functionSpace}    
\end{equation*}
Here $\bm d = (d_0, \ldots, d_L)$ denotes the dimensions, $\bm k = (k_1, \ldots, k_L)$  the filter sizes, and 
$\bm s  = (s_1, \ldots, s_L)$ the strides of the $L$ layers. 
We refer to the triple $(\bm d, \bm k, \bm s)$ as the \emph{architecture} of the LCN. 
We tacitly assume in the following 
that the dimensions $\boldsymbol{d}$ are compatible by satisfying~\cref{eq:dimensionConstraintsToeplitz}. 
We also note that by \cref{eq:dimensionConstraintsToeplitz} the function space $\sM_{\bm{d},\bm{k},\bm{s}}$ is fully determined by fixing $d_L$, $\bm k$, and $\bm s$.
As such, we may omit any of the indices $\boldsymbol{d}$, $\boldsymbol{k}$ and $\bm s$  when clear from the context.
The function space $\sM_{\bm{d},\bm{k},\bm{s}}$ is the image of the \textit{parameterization map} 
\begin{equation} \label{eqn:Wbar}
\mu_{\bm{d},\bm{k},\bm{s}}\colon \mathbb{R}^{k_1}\times \cdots\times \mathbb{R}^{k_L} \longrightarrow \sM_{\bm{d},\bm{k},\bm{s}} ;\quad (w_1,\ldots, w_L) \longmapsto \overline{W},
\end{equation}
that takes $L$ filters to the product of their associated convolutional matrices.

\begin{remark} Although we use $\sM_{\bm{d},\bm{k},\bm{s}}$ to denote the function space of LCNs, the sequence of dimensions $\bm d$ does not play an essential role. More precisely, 
since the product of convolutional matrices is again a convolutional matrix (see \Cref{prop:nonzerodiagsExtension}),
we can also define the parameterization map~\cref{eqn:Wbar} at the level of filters $
\mu_{\bm{k},\bm{s}}\colon \mathbb{R}^{k_1}\times \cdots\times \mathbb{R}^{k_L} \to \RR^{k} ;\; (w_1,\ldots, w_L) \mapsto \overline w, 
$
where $\overline{w}$ is the filter of the product convolutional matrix.
\end{remark}

\subsection{Compositions of convolutional matrices} 

It is well known that the composition of convolution operations is again a convolution. More precisely, for two convolutional maps $\alpha_1: \RR^{d_0} \rightarrow \RR^{d_1}$ and $\alpha_2: \RR^{d_1} \rightarrow \RR^{d_2}$ with filters $w_1$ and $w_2$ respectively and satisfying~\eqref{eq:dimensionConstraintsToeplitz}, the composition is given by
\begin{align}
\label{eq:convolutionalMatricesComposition}
\begin{split}
(\alpha_2 \circ \alpha_1) (x)_i &= \sum_{j \in [k_2]} w_{2,j} \cdot \alpha_1(x)_{is_2+j} = \sum_{j \in [k_2]} w_{2,j} \cdot \sum_{\ell \in [k_1]} w_{1,\ell} \cdot x_{(is_2 + j)s_1+\ell}\\ 
&= \sum_{j \in [k_2]} \sum_{\ell \in [k_1]} w_{2,j} w_{1,\ell} \cdot x_{is_2 s_1 + (js_1 + \ell)} = \sum_{m \in [(k_2-1)s_1 + k_1]} u_{m} \cdot x_{is_2s_1 + m}. 
\end{split}
\end{align}
Here the resulting filter $u$ has size $(k_2-1)s_1 + k_1$ and has entries 
\begin{align}
\label{eq:filterProduct}
    u_m = \sum_{\substack{j \in [k_2],  \,\ \ell \in [k_1]\\ js_1 + \ell = m}} w_{2,j} \, w_{1, \ell}.
\end{align}

\begin{proposition}
\label{prop:nonzerodiagsExtension}
The composition of $L$ convolutions with filter sizes ${\bm k} = (k_1,\ldots,k_L)$ and strides ${\bm s} = (s_1,\ldots, s_L)$ is a convolution with 
filter size $k = k_1 + \sum_{l=2}^L (k_l-1) \prod_{m=1}^{l-1} s_m$ and stride $s = \prod_{l=1}^L s_l$. 
If all strides $s_i$ are equal to $\tilde s$, then  
$k=d_0 - (d_L-1)s$ and $s = \tilde s^L$.
\end{proposition}

\begin{proof}
The expression in~\cref{eq:convolutionalMatricesComposition} shows that the composition of two convolutions with filter sizes $k_1, k_2$ and strides $s_1, s_2$ is again a convolution, with filter size $k=k_1 + (k_2-1)s_1$ and stride $s_2 s_1$. If the strides $s_1,s_2$ are equal to $\tilde s$,
then by~\cref{eq:dimensionConstraintsToeplitz}, we have $k_1 = d_0 -  (d_1-1) \tilde s$ and $k_2 = d_1 - (d_2-1) \tilde s$, which yields $k = k_1 + (k_2-1) \tilde s = d_0 - (d_2-1)\tilde s^2$.
The statement then follows by induction, whereby at the $l$-th iteration we multiply the convolution represented up to layer $l-1$ with the $l$-th layer convolution. 
\end{proof}

\noindent By \Cref{prop:nonzerodiagsExtension}, $\sM_{\bm d, \bm k, \bm s}$ is contained in the vector space $\sM_{(d_0,d_L),k,s}$ of convolutional $d_L\times d_0$ matrices with filter size $k$ and stride $s$. A natural question is whether this containment is strict.

\begin{definition}%
\label{def:filling}\rm
An architecture $(\bm d, \bm k, \bm s)$ is \emph{filling} if $\sM_{\bm d, \bm k, \bm s} = \sM_{(d_0,d_L),k,s}$ holds, where $k$ and $s$ are as in \Cref{prop:nonzerodiagsExtension}. 
\end{definition}

Filling LCN architectures might be desirable since in function space training using a convex loss is simply 
convex optimization in a vector space.
 We emphasize that the vector space $\sM_{(d_0,d_L),k,s}$ is properly contained in $\RR^{d_L \times d_0}$, since convolutional matrices have restricted patterns of non-zero entries as well as rows with repeated entries. 
For non-filling architectures, we will see in \Cref{sec:geometry} that $\sM_{\bm d, \bm k, \bm s}$ has additional constraints in the form of polynomial equalities and inequalities. 
 Below we give examples of filling and non-filling LCN architectures. 
 For fully-connected linear networks, the only constraints on the space of representable matrices are rank constraints (vanishing of certain minors, which are polynomial equations). For LCNs, the matrix sizes are decreasing by \cref{eq:dimensionConstraintsToeplitz}, and such rank constraints do not occur.

\begin{example}\label{example:minimal}\rm
Consider the LCN architecture with $\bm d = (3,2,1)$, $\bm k = (2,2)$, $\bm s = (1,1)$. The corresponding function space $\sM \subset \RR^{1 \times 3}$ consists of matrices 
\begin{equation*}
    \overline{W} = W_2W_1 = \begin{bmatrix} c&d \end{bmatrix} \begin{bmatrix} a&b&0\\0&a&b \end{bmatrix} =  \begin{bmatrix} ac & ad + bc &  bd \end{bmatrix} =: \begin{bmatrix} A & B  &  C \end{bmatrix} . 
\end{equation*}
We can see that this architecture is not filling, \ie, that the containment $\sM \subsetneq \RR^{1 \times 3}$ is strict. Indeed, the entries in $\overline{W}$ are given by the coefficients $(A,B,C)$ of quadratic polynomials that factor into two real linear forms: 
\begin{equation*}
    (a \x + b)(c \x + d) = ac\x^2 + (ad + bc) \x + bd = A\x^2 + B \x + C.
\end{equation*}
These correspond to quadratic equations with one or more real roots (or linear equations with $A=0$). 
Hence the set of representable coefficients $(A,B,C)$ is characterized by the discriminant condition $B^2 - 4AC \ge 0$. 
See \Cref{fig:zwei} (left). 
\end{example}

\begin{example} \label{example:larger}\rm
Consider the LCN with $\bm d = (5,3,2)$, $\bm k = (3,2)$, $\bm s = (1,1)$. Then 
\begin{equation*}
    \overline{W} = W_2W_1 = 
    \begin{bmatrix} d&e&0\\
    0&d&e\end{bmatrix}
    \begin{bmatrix}
    a&b&c&0&0\\
    0&a&b&c&0\\
    0&0&a&b&c
    \end{bmatrix}
    =
    \begin{bmatrix}
    A&B&C&D&0\\
    0&A&B&C&D
    \end{bmatrix},
\end{equation*}
where $(A,B,C,D) = (ad, \, bd + ae, \, cd + be, \, ce)$. 
The set $\sM$ is the set of all possible $\overline{W}$ where its entries are related by the above equation as $(a,b,c,d,e)$ vary over $\R^5$. 
Rephrasing this question as which real cubic polynomials factor into a real quadratic and linear factor, \ie,
\begin{equation*}
    (a\x^2 + b\x + c)(d\x + e) = ad\x^3 + (bd+ae)\x^2 + (cd + be)\x + ce = A\x^3 + B\x^2 + C\x + D,
\end{equation*}
one can immediately draw on the fundamental theorem of algebra to conclude that every choice of $(A,B,C,D)$ has a corresponding choice of $(a,b,c,d,e)$. From this, we note that $\sM$ is filling, so $\sM_{(5,3,2),(3,2),(1,1)} = \sM_{(5,2),4,1}$.
\end{example}

\subsection{Polynomial multiplication}

The previous examples illustrate a useful identification between  convolutional matrices and  polynomials. More precisely, we associate a filter $w \in \RR^k$ with the polynomial
\begin{equation}\label{eq:polynomial_identification}
\pi(w) = w_0 \x^{k-1} + w_1 \x^{k-2}\y + \cdots + w_{k-2}\x\y^{k-2} + w_{k-1}\y^{k-1} \in \RR[\x,\y]_{k-1}.
\end{equation}
This is a formal polynomial, where the monomials $\x^i\y^j$ are used to index the weights of the filter. To better handle situations where leading coefficients vanish, it is convenient to use homogeneous bivariate polynomials instead of univariate affine polynomials as in the two previous examples. Recall that a root of a homogeneous polynomial $p(\x,\y)$ is an equivalence class of pairs $(r_\x, r_\y) \in \CC^2 \setminus \{ (0,0) \}$ such that $p(r_\x, r_\y) = 0$, up to complex scalar factors (\ie, a point in $\mathbb C\mathbb P^1$). 
With some abuse of notation, we will also use $\pi$ to map a convolutional matrix $W$ with filter $w$ to the polynomial in~\cref{eq:polynomial_identification}. 

\begin{proposition}\label{prop:deep_poly_multiplication}
Consider convolutional matrices $W_1,\ldots, W_L$ with compatible sizes and stride one.  
Then the convolutional matrix $\overline W = W_L \cdots W_1$ satisfies $\pi(\overline W) =\pi(W_L) \cdots \pi(W_1)$. 
\end{proposition}

\begin{proof}
For the product of two convolutional matrices, the statement follows from \cref{eq:filterProduct}.
The general case follows by induction. 
\end{proof}

\begin{corollary}\label{cor:polynomial_space_identification}
Let $(\boldsymbol d, \boldsymbol k, \boldsymbol s)$ be an LCN architecture of depth $L$ with $\boldsymbol{s} = (1, \ldots, 1)$. 
Then $\pi$ identifies $\mathcal{M}_{\boldsymbol d, \boldsymbol k, \boldsymbol s}$ with the set of polynomials that factor into polynomials of degrees $k_1 - 1, \ldots, k_L-1$, that is,  $\pi(\mathcal{M}_{\boldsymbol d, \boldsymbol k, \boldsymbol s}) = \left\lbrace p = q_L \cdots q_1 \colon q_i \in \mathbb{R}[\x,\y]_{k_i-1}  \right\rbrace \subseteq \mathbb{R}[\x,\y]_{k-1}$ where $k = \sum_{i=1}^L (k_i -1) + 1$ is the size of the end-to-end filter.
\end{corollary}

Since the ordering of the factors does not affect the product polynomial, the function space of an LCN with stride-one 1D convolutional layers is the same for any permutation of the layers. 
This also holds for higher-dimensional convolutions, as we will see in \Cref{sec:higher_dimension}. 

\begin{remark}
\label{rem:polynomialCorrespondenceHigherStride}
\Cref{prop:deep_poly_multiplication} can be generalized to deal with arbitrary strides $\boldsymbol s$. 
For example, for two convolutional matrices $W_2,W_1$ with strides $s_2,s_1$, the product $W_2 W_1$ can be identified with 
the polynomial $p = \pi(W_2 W_1) \in \RR[\x,\y]$ that factors as $p = q_2q_1$, where $q_1 = \pi(W_1)$ as in~\cref{eq:polynomial_identification}, while $q_2$ is defined as follows: 
if $w$ denotes the filter of $W_2$, we see from \cref{eq:filterProduct} that  
\[q_2 = w_0 \x^{s_1(k_2-1)} + w_1 \x^{s_1(k_2-2)}\y^{s_1}  + \cdots + w_{k_2-2} \x^{s_1} \y^{s_1(k_2-2)} + w_{k_2-1}\y^{s_1(k_2-1)}. 
\]
\end{remark}

\subsection{Toeplitz vs.\ circulant matrices}
\label{subsec:toep_vs_circulant}

Discrete convolutions over finite-dimensional data can be defined in slightly different ways depending on how they deal with the boundary of the input. Until now, we have been discussing a ``no padding'' setting where the convolution is not applied across the boundary, and the dimension of the output is smaller than the dimension of the input (see \cref{eq:toeplitz}). An alternative approach would be to consider \emph{cyclic convolutions}, where the input space $\RR^d$ is viewed as representing ``cyclic signals'' $\mathbb{Z}/d \rightarrow \RR$. In this section we point out that the two approaches are effectively equivalent in terms of the function space. 

Let $W_1,\ldots, W_L$ be the weight matrices of an LCN. These matrices are generalized Toeplitz matrices and must have compatible sizes \cref{eq:dimensionConstraintsToeplitz}.
If $W_l$ has size $d_{l} \times d_{l-1}$ and corresponding filter $w_l$ of size $k_l$, then we construct a \emph{generalized circulant} matrix of size $d_{0} \times d_{0}$ as follows.
If the stride is $s_l = 1$, the circulant matrix is
\[
\begin{bmatrix} w_{l,0} & w_{l,1} & \ldots & w_{l, d_0-1} \\
w_{l, d_0-1} & w_{l,0} & \ldots & w_{l, d_0-2} \\
\vdots &\vdots&\ddots& \vdots \\
w_{l,1} & w_{l,2} &\ldots & w_{l,0} \\
\end{bmatrix} , 
\]
where $w_{l,i} = 0$ for $i \geq k_l$. 
For Toeplitz matrices with larger strides $s_l$, we can similarly define a $d_0 \times d_0$ generalized circulant matrix of stride $s_l$, where each row is obtained by shifting the previous row by $s_l$ steps.

The generalized Toeplitz matrix $W_l$ is the top-left $d_l \times d_{l-1}$  block of its associated generalized circulant matrix. 
Under this identification, the multiplication of generalized Toeplitz matrices with compatible sizes is the multiplication of the corresponding generalized circulant matrices. 
In particular, the resulting filter $\overline{w}$ of the Toeplitz matrix product $\overline{W}$ also results from the product of the circulant matrices. Therefore both formulations are equivalent. 

\begin{example} \rm
For the Toeplitz matrices in \Cref{example:minimal}, 
the corresponding product of circulant matrices~is
\begin{equation*}
\begin{bmatrix}
c&d&0\\
0&c&d\\
d&0&c\\
\end{bmatrix}
\begin{bmatrix}
a&b&0\\
0&a&b\\ 
b&0&a\\ 
\end{bmatrix}
= \begin{bmatrix}
ac&ad+bc&bd\\
bd&ac&ad+bc\\ 
ad+bc&bd&ac\\ 
\end{bmatrix}. 
\end{equation*}
\end{example}

\section{Discriminants and multiple roots}
\label{sec:discriminants}
In this section, we review some properties of the structure of real polynomials with repeated roots, recalling the definitions of \emph{discriminants}, \emph{multiple root loci}, and \emph{real root multiplicity patterns} (\emph{\rrmp} for short).
We will use these notions in the next sections to study the function space of LCNs and to investigate the critical points of the associated loss functions.
A reader who is more interested in the applications of these concepts to LCNs might skip this section on a first read and return to it later to better understand the structure of repeated roots in polynomials.

The \emph{discriminant} of a polynomial $p=a_n\x^n+a_{n-1}\x^{n-1}\y +\cdots+a_1 \x\y^{n-1} +a_0\y^n \in \mathbb{R}[\x,\y]$ is a polynomial in the coefficients $a_n,\ldots, a_0$ that vanishes if and only if $p$ has a (complex) double root \cite[Chapter~4]{10.5555/1197095}. 
The iterated singular loci of the discriminant hypersurface are the \emph{coincident/multiple root loci}.

\begin{definition} \rm 
For a partition $\lambda = (\lambda_1, \ldots, \lambda_r)$ of $k-1$,
the \emph{multiple root locus} $\Delta_{\lambda} \subset \mathbb{R}[\x,\y]_{k-1}$ is the set of homogeneous polynomials that have $r$ (not necessarily distinct) complex roots with multiplicities $\lambda_1, \ldots, \lambda_r$. That is, $\Delta_\lambda$ is the set of polynomials $P \in \mathbb{R}[\x,\y]_{k-1}$ that can be written as $P=Q_{1}^{\lambda_1} \cdots Q_r^{\lambda_r}$, where each $Q_i \in \CC[\x,\y]_1$ is a complex linear form.
\end{definition}
For instance, the discriminant hypersurface is $\Delta_{(2,1,\ldots,1)}$. 
The singular locus of a multiple root locus is a union of certain higher-order multiple root loci; see \cite{kurmann2012some} for a combinatorial description.
\footnote{For two partitions $\lambda$ and $\mu$ of $k-1$, $\Delta_{\mu}$  is contained in $\Delta_{\lambda}$  if and only if $\mu$ is a coarsening of $\lambda$. In \cite[Proposition 2.1(i)]{kurmann2012some} it is shown that $\Delta_{\mu}$ is contained in the singular locus of $\Delta_{\lambda}$ if and only if either there are at least two ways of obtaining $\lambda$ as a splitting of $\mu$ or there is a splitting such that one of the elements $\mu_i$ in $\mu$ is split into subsets which do not all have the same size. 
An illustrative example is given in~\cite[Example 2.4]{kurmann2012some}.} 
We are also interested in distinguishing between real and complex roots. This more fine-grained information is captured by the  \emph{real root multiplicity pattern}.

\begin{definition} \rm
A homogeneous polynomial $P \in\RR[\x,\y]_{k-1}$ has \emph{real root multiplicity pattern}, short \emph{\rrmp}, $(\rho \mid \gamma) = (\rho_1, \ldots, \rho_r \mid \gamma_1, \ldots, \gamma_c)$ if it can be written as \begin{equation}
\label{eq:realFactorization}
P = p_1^{\rho_1} \cdots p_r^{\rho_r} q_1^{\gamma_1} \cdots q_c^{\gamma_c}, \end{equation}
where $p_i \in \RR[\x,\y]_1$ and $q_j \in \RR[\x,\y]_2$ are irreducible and pairwise linearly independent. 
\end{definition}

Recall that each factor $q_j$ has two complex conjugate roots. 
Polynomials with \rrmp\, $(\rho \mid \gamma)$ are smooth points on the multiple root locus $\Delta_{\lambda_{\rho|\gamma}}$, 
where $\lambda_{\rho|\gamma} = (\rho_1, \ldots, \rho_r, \gamma_1, \gamma_1, \ldots, \gamma_c, \gamma_c)$.

\begin{example}\rm \label{ex:rrmps}
We describe the possible \rrmp\, of a quadratic, cubic, and quartic polynomial.
The discriminant polynomial 
$\Delta$ of  $a\x^2+b\x\y+c\y^2$ distinguishes three \rrmp's:
\vspace{.2cm}

\begin{center}
\small 
\begin{tabular}{ll} 
{$11|0$}& 2 distinct real roots: $\Delta = b^2-4ac > 0$. \\
{$2|0$}& 1 double real root: $\Delta =0$. \\
{$0|1$}& no real roots: $\Delta<0$. 
\end{tabular}
\end{center}

\vspace{.2cm}

\noindent
For  cubic polynomials $a\x^3+b\x^2\y+c\x\y^2+d\y^3$ with real coefficients, there are four \rrmp's shown below, discussed in \cite{ARNON198837,BlinnHowto,Gonzalez2015RootCO}. 
\vspace{.2cm}

\begin{center}
\small 
\begin{tabular}{ll} 
    {$111|0$} & 3 distinct real roots: $\Delta = b^2 c^2 - 4 a c^3 -4 b^3 d - 27 a^2 d^2 + 18 a b c d  > 0$. \\
     {$12|0$} & 1 single real root and 1 double real root: $\Delta = 0$.\\
     {$3|0$} & 1 triple real root: $\delta_1 = 3ac-b^2 = 0$ and $\delta_2 = 9ad-bc= 0$ and $\delta_3 = 3bd-c^2= 0$. \\
    {$1|1$} & 1 real root and 2 complex conjugate roots: $\Delta < 0$. 
\end{tabular} 
\end{center}
\vspace{.2cm}

\noindent
Following \cite{ARNON198837}, a polynomial $f_n=\x^n+a_1\x^{n-1}\y+a_2\x^{n-2}\y^2+ \cdots +a_n\y^n$  can be translated to a polynomial $g_n=f_n(\x-\nicefrac{a_1}{n}\y,\y) = \x^n+b_2\x^{n-2}\y^2+b_3\x^{n-3}\y^3+\cdots+b_n\y^n$. The two polynomials $f_n$ and $g_n$ have the same \rrmp. 
For a quartic polynomial $g_4=\x^4+p\x^2\y^2+q\x\y^3+r\y^4$, letting 
\begin{equation*}
\begin{split}
\delta(p,q,r) =& 256 r^3 - 128p^2r^2 + 144 pq^2r +16p^4r -27q^4 -4p^3q^2\\     
\delta'(p,q,r) =& 8pr - 9q^2 -2p^3, 
\end{split}
\end{equation*}
where $\delta$ corresponds to the discriminant, the different \rrmp's are as follows: 

\begin{center}
\begin{tabular}{ll} 
     {$1111|0$} & 4 distinct real roots: $\delta>0$ and $\delta'>0$ and $p<0$. \\  
     {$112|0$} & 2 distinct real roots and 1 double real root: $\delta=0$ and $\delta'>0$ and $p<0$. \\  
     {$22|0$} & 2 real double roots: $\delta=0$ and $\delta'=0$ and $p<0$ and $q=0$. \\ 
     {$13|0$} & 1 single real root and 1 triple real root: $\delta=0$ and $\delta'=0$ and $p<0$ and $q\neq0$. \\ 
     {$4|0$} & 1 quadruple real root: $\delta=0$ and $\delta'=0$ and $p=0$. \\ 
     {$11|1$} & 2 distinct real roots and 2 distinct complex conjugate roots: $\delta<0$. \\ 
     {$2|1$} & 1 double real root and 2 distinct complex conjugate roots: $\delta=0$ and $\delta'<0$. \\ 
     {$0|2$} & 2 complex conjugate double roots: $\delta=0$ and $\delta'=0$ and $p>0$. \\ 
    {$0|11$} & 4 distinct complex roots: $\delta>0$ and ($\delta'\leq 0$ or $p>0$). 
\end{tabular}
\end{center}
\end{example}

\section{Function space of linear convolutional networks}
\label{sec:geometry}

We show that the function space of an LCN is a semi-algebraic set, 
discuss its dimension, and whether it is filling (\Cref{def:filling}).

\subsection{Convolutional matrices of stride one} 

We first focus on one-dimensional convolutions and stride one.  
Recall that a semi-algebraic set is a solution set to a finite set of polynomial equations and polynomial inequalities, or a finite union of such sets.

\begin{theorem}%
\label{thm:fillingCircular}
Let $(\bm d, \bm k, \bm s)$ be an LCN architecture with $\bm s= (1,\ldots,1)$ and depth $L$. 
Then $\sM_{\bm d, \bm k, \bm s}$ is a full-dimensional semi-algebraic subset of $\sM_{(d_0, d_L), k, s}$, where 
$k= \sum_{i=1}^L k_i - L + 1$ and $s=1$. 
In particular, $\sM_{\bm d, \bm k, \bm s}$ is cut out from $\sM_{(d_0, d_L), k, s}$ by inequalities only. 
Moreover, the architecture is filling if and only if at most one of the filter sizes $k_1,\ldots, k_L$ is even. 
\end{theorem}

This result explains the difference between \Cref{example:minimal,example:larger}.
To prove it, we use \Cref{cor:polynomial_space_identification}:
We identify $\sM_{(d_0, d_L), k, s}$ with the space $\RR[\x,\y]_{k-1}$ of polynomials of degree  $k-1$, 
and the subset $\sM_{\bm d, \bm k, \bm s}$ with the polynomials that factor as a product of polynomials of degrees $k_1-1,\ldots, k_{L}-1$. 
Under this identification, we have the following.

\begin{lemma}
\label{lem:realRoots}
Let $\bm s= (1,\ldots,1)$. A polynomial $p \in \RR[\x,\y]_{k-1}$ 
corresponds to a convolutional matrix in $\sM_{\bm d, \bm k, \bm s}$ if and only if $p$ has at least $e := |\{k_i \colon k_i \mbox{ is even}\}|$ real roots (counted with multiplicity), \ie, if the rrmp $(\rho \, | \, \gamma)$ of $p$ satisfies $\sum_{i=1}^r \rho_i \ge e$. 
\end{lemma}

\begin{proof} By \Cref{cor:polynomial_space_identification}, it suffices to show that a polynomial $p$ of degree $ k-1$ has at least $e$ real roots if and only if $p$ is a product of polynomials of degrees $k_1-1,\ldots, k_{L}-1$. 
One direction is immediate. A factorization of $p$ of the desired type  has $e$ odd-degree factors, and any polynomial $p$ with $e$ odd-degree factors has at least $e$ real roots. 

For the other direction, 
let $e'$ denote the number of real roots of $p$.
The irreducible factors of $p$ are $e'$ linear terms and $(k-1 - e')/2$ quadratic terms.
If $e' \geq e$, these factors can be multiplied to build a set of factors $q_1, \ldots, q_L$ of  degrees $k_1-1,\ldots,k_L-1$ as follows. We first assign one of the $e'$ linear terms to each of the $e$ factors $q_i$ of odd degree and then we use the remaining $(e'-e) + (k-1 - e')/2$ irreducible factors of $p$ to fill up the factors $q_1, \ldots, q_L$ until their degrees are indeed $k_1-1,\ldots,k_L-1$.
Note that this procedure works as the degrees of the irreducible factors of $p$ and the degrees of $q_1, \ldots, q_L$ both sum up to the same value~$k-1$.
\end{proof}

\begin{example} \rm
\label{ex:quartic}
When $\bm s = (1,\ldots,1)$ and $\bm k = (k_1,\ldots,k_L)$ are chosen such that the resulting filter size is $k = 5$, by \Cref{lem:realRoots} the matrices in $\sM_{\bm d, \bm k, \bm s}$ correspond to quartic polynomials with a particular real root structure, specified by the individual filter sizes $k_i$:
\begin{enumerate}[leftmargin=*]
    \item If $L=2$ and $(k_1,k_2)=(3,3)$, then every polynomial of degree four is in $\pi(\sM_{\bm d, \bm k, \bm s})$. 
    \item \label{ex:quartic2} If $L = 2$ and $(k_1, k_2) = (4,2)$, then a quartic polynomial $p$ corresponds to a matrix in the function space if and only if it has two or four real roots.
    Hence, the complement of the function space corresponds to all quartics with no real roots.
    \item \label{ex:quartic3} If $L = 3$ and $(k_1, k_2, k_3) = (3,2,2)$, this architecture has the same function space as the previous one. However, the parameterization of the function space is different: Here $\pi(\sM_{\bm d, \bm k, \bm s})$ arises by taking the product of two linear and one quadratic factor, whereas in the previous architecture one considers the product of a linear and a cubic factor. 
    \item \label{ex:quartic4} If $L=4$ and $(k_1, k_2, k_3, k_4) = (2,2,2,2)$, a polynomial of degree four is in $\pi(\sM_{\bm d, \bm k, \bm s})$ if and only if all its roots are real.  
\end{enumerate}
\end{example}

\begin{proof}[Proof of \Cref{thm:fillingCircular}] 
We first consider the case where all filter sizes are odd.
Using \Cref{lem:realRoots}, we see that $e=0$, so every polynomial of degree $k-1$ corresponds to a matrix in the function space $\sM_{\bm d, \bm k, \bm s}$.
If exactly one filter size is even, then $k-1$ is odd.
Hence, every polynomial $p$ of degree $k-1$ has at least one real root.
In \Cref{lem:realRoots} we have $e=1$, which shows that $p \in \pi(\sM_{\bm d, \bm k, \bm s})$.
We conclude that the architecture is filling if at most one filter size is~even. 
Finally, assume that two or more filter sizes are even. In this case, we consider a polynomial $p$ of degree $k-1$ that has zero (if $k-1$ is even) or one (if $k-1$ is odd) real root. Then $p$ does not belong to $\pi(\sM_{\bm d, \bm k, \bm s})$ by \Cref{lem:realRoots}. This shows that the architecture is not filling.
That the function space is semi-algebraic follows from the fact that it has a polynomial parametrization (by Tarski-Seidenberg). 
Finally, we argue that the set of polynomials representable by a given LCN is full dimensional in $\mathbb{R}[\x,\y]_{k-1}$. 
For this, it is enough to notice that small perturbations of the coefficients of a generic polynomial will not affect its root structure. 
We will discuss this in more detail in the following statements. 
\end{proof}

When the architecture of an LCN is not filling, it is interesting to study the boundary of the function space. 

\begin{proposition} \label{prop:boundaryCondition}
Let $(\bm d, \bm k, \bm s)$ be an LCN architecture with $e := |\{k_i \colon k_i \text{ is even}\}| \geq 2$ and $\bm s= (1,\ldots,1)$. 
In the Euclidean topology on $\sM_{(d_0,d_L), k, s}$, the function space $\sM_{\bm d, \bm k, \bm s}$ is closed and its boundary consists of all convolutional matrices that correspond to polynomials with \rrmp\, $(\rho \mid \gamma)$ satisfying 
$\sum_{i=1}^r \rho_i \geq e$ and $|\{\rho_i \colon \rho_i \mbox{ is odd}\}| \leq e-2$. 
\end{proposition}
\begin{proof}
A sequence of polynomials, each having at least $e$ real roots counted with multiplicity, will also have at least $e$ real roots in the limit.
Hence, by \Cref{lem:realRoots}, the function space is closed. 
Using the identification $\pi$ with polynomials, the boundary  of $\sM_{\bm d, \bm k, \bm s}$ consists of all polynomials in $\sM_{\bm d, \bm k, \bm s}$ that are limits of sequences of polynomials in its complement.
\\ \indent
Given a sequence of polynomials $P^{(j)}$ of degree $(k-1) = \sum_{i=1}^L (k_i-1)$ in the complement  of $\sM_{\bm d, \bm k, \bm s}$ that converges to a polynomial in $\sM_{\bm d, \bm k, \bm s}$, we may assume (by restricting to a sub-sequence) that all $P^{(j)}$ have the same \rrmp\, $(\rho \mid \gamma)$. 
Since $e$ and $(k-1)$ have the same parity and $P^{(j)} \notin \sM_{\bm d, \bm k, \bm s}$, we have $\sum_{i=1}^r \rho_i \leq e-2$.
In particular, the number of odd $\rho_i$ is at most $e-2$.
In the limit of the $P^{(j)}$, pairs of complex conjugate roots might become real double roots (which does not affect the number of odd $\rho_i$) or distinct real roots might become equal (which cannot increase the number of odd $\rho_i$).
Hence, the limit of the $P^{(j)}$ has \rrmp\, $(\rho' \mid \gamma')$ with at most $e-2$ odd $\rho'_i$ and $\sum \rho'_i \geq e$ (since we assumed the limit to be in $\sM_{\bm d, \bm k, \bm s}$).
\\\indent
Conversely, given a polynomial $P$ with \rrmp\, $(\rho \mid \gamma)$ satisfying $|\{\rho_i \colon \rho_i \mbox{ is odd}\}| \leq e-2$ and $\sum_{i=1}^r \rho_i \geq e$, we write $\rho_i = 2\alpha_i+\beta_i$, where $\beta_i \in \{0,1\}$.
We consider its factorization as in \cref{eq:realFactorization} and perturb each factor $p_i^{\rho_i}$ to a polynomial with $\alpha_i$ pairs of complex conjugate roots, plus one real root if $\beta_i=1$. 
That way, $P$ gets perturbed to a polynomial $\tilde P$ with exactly $\sum \beta_i$ real roots. 
We have that $\sum \beta_i = |\{\rho_i \colon \rho_i \mbox{ is odd}\}| \leq e-2$, so $\tilde P \notin \sM_{\bm d, \bm k, \bm s}$. 
\end{proof}
\begin{corollary}
Let $(\bm d, \bm k, \bm s)$ be an LCN architecture with $\bm s= (1,\ldots,1)$ and two or more even filter sizes.
The Zariski closure of the Euclidean boundary of the function space $\sM_{\bm d, \bm k, \bm s}$ is the discriminant hypersurface $\pi^{-1}(\Delta_{(2,1,\ldots,1)})$.
\end{corollary}
\begin{proof}
Let $e := |\{k_i \colon k_i \text{ is even}\}|$.
Since the \rrmp\, $(\rho\mid\gamma)$ of any polynomial on the Euclidean boundary of $\sM_{\bm d, \bm k, \bm s}$ satisfies by \Cref{prop:boundaryCondition} that at least one of the $\rho_i$ is larger than $1$, the boundary is contained in the discriminant hypersurface. 
Moreover, any polynomial with one real double root, $e-2$ pairwise distinct real single roots, and all other roots non-real with multiplicity one, is both on the boundary of $\sM_{\bm d, \bm k, \bm s}$ and a smooth point of the discriminant.
\end{proof}
\begin{example}\rm \label{ex:boundaries}
The following table lists all non-filling architectures with stride one and where the end-to-end filter has size $3,4,5$ (up to permutations of the filters):
\begin{center}
    \small
    \begin{tabular}{c|ccc}
 $\bm k$ &function space $\mathcal M_{\bm k}$   & complement & Euclidean boundary  \\ \hline
 $(2,2)$ & $11|0$, $2|0$ & $0|1$ & $2|0$ \\
 $(2,2,2)$ & $111|0$, $12|0$, $3|0$ & $1|1$ & $12|0$, $3|0$ \\
 $(3,\!2,\!2)$ \!or\! $(4,\!2)$ & $1111|0$, $112|0$, $22|0$, $13|0$, $4|0$, $11|1$, $2|1$ & $0|2$, $0|11$ & $2|1$, $22|0$, $4|0$ \\
 $(2,2,2,2)$ & $1111|0$, $112|0$, $22|0$, $13|0$, $4|0$ & $11|1$, $2|1$, $0|2$, $0|11$ & $112|0$, $13|0$, $22|0$, $4|0$
 \\
 \hline 
\end{tabular} 
\end{center} 
\vspace{.35cm}

\noindent
For each architecture $\bm k$, the function space $\mathcal{M}_{\bm k}$ consists of all polynomials with one of the \rrmp's listed above. 
We described the possible \rrmp's for quadratic, cubic, and quartic polynomials in \Cref{ex:rrmps}.
Similarly, the complement and boundary of $\mathcal{M}_{\bm k}$ are as shown above.
\end{example}

When exactly two filter sizes are even, the non-filling geometry of $\sM_{\bm d, \bm k, \bm s}$ has a particularly simple description. 

\begin{proposition}
\label{prop:twoConvexCones}
Let $(\bm d, \bm k, \bm s)$ be an LCN architecture with $\bm s= (1,\ldots,1)$ and exactly two even filter sizes.
The complement of $\pi(\sM_{\bm d, \bm k, \bm s})$ in $\RR[\x,\y]_{k-1}$ is a union of two open convex cones, each consisting of all positive or respectively negative polynomials. 
\end{proposition}
\begin{proof}
If exactly two filters have even size, then $k-1$ is even, and  a polynomial of degree $k-1$ is in $\pi(\sM_{\bm d, \bm k, \bm s})$ if and only if it has at least two real roots (by \Cref{lem:realRoots}). 
Hence, the complement of $\pi(\sM_{\bm d, \bm k, \bm s})$ consists of all polynomials of even degree $k-1$ without real~roots. 
\end{proof}
\noindent This behaviour is exhibited in \Cref{example:minimal}, illustrated on the left  of \Cref{fig:zwei}, and in \Cref{ex:quartic} cases 2 and 3 (see also \Cref{ex:boundaries} first and third rows). 
In general, if strictly more than two filter sizes are even, the complement of $\pi(\sM_{\bm d, \bm k, \bm s})$ is \emph{not} a union of two convex cones, as the following example demonstrates.

\begin{example}\rm
Consider the case $|\{k_i \colon k_i \mbox{ is even}\}| = 3$ with stride one. 
A polynomial of odd degree $k-1$ is in $\pi(\sM_{\bm d, \bm k, \bm s})$ if and only if $p$ has at least three real roots (by \Cref{lem:realRoots}).
Hence, the complement $\RR[\x,\y]_{k-1} \setminus \pi(\sM_{\bm d, \bm k, \bm s})$ is the set of all polynomials of degree $k-1$ with exactly one real root. 
Even when restricting ourselves to degree-$(k-1)$ polynomials with one real root such that the remaining degree-$(k-2)$ factor is positive, this set does not form a convex cone.
For instance, if $k-1 = 3$, consider the polynomials $(\x-\y)\x(\x+\y)+2\y^3$ and $(\x-\y)\x(\x+\y)-2\y^3$:
Both have exactly one real root such that the remaining quadratic factor is positive. 
However, their sum clearly has three real roots. 
\end{example}

\subsection{Generalization to larger strides}
\label{sec:largerStrides}

In this section, we consider one-dimensional convolutions that may have stride larger than one. In this setting, the identification with polynomials holds as described in  Remark~\ref{rem:polynomialCorrespondenceHigherStride}. 
Specifically, we write 
\begin{equation*}
\pi_s: \RR^k \longrightarrow \RR[\x,\y]_{(k-1)s}, \quad w \longmapsto w_{0} \x^{(k-1)s} + w_1 \x^{(k-2)s}\y^s + \cdots + w_{k-2} \x^s\y^{(k-2)s} + w_{k-1} \y^{(k-1)s}. 
\end{equation*} 
This generalizes the map $\pi = \pi_1$ defined in~\cref{eq:polynomial_identification}. 
If $W_2$ and $W_1$ are convolutional matrices with strides $s_2$ and $s_1$, then $\pi(W_2 W_1) = \pi_{s_1}(W_2) \pi(W_1)$. 
Note that $\pi_{s_1}$ is applied to $W_2$, while $s_2$ does not have any effect on the filter of $W_2 W_1$ (only on its stride which is given by $s_2 s_1$). 
Based on this fact, we deduce the following two results.

\begin{lemma}
\label{lem:specialCases}
\begin{enumerate}[leftmargin=*]
\item[a)] If $\bm s \!=\! (s_1,\ldots, s_L)$ and $\bm s' \!=\! (s_1,\ldots, s_{L-1},1)$, then
$\pi(\mathcal M_{\bm d, \bm k, \bm s}) = \pi(\mathcal M_{\bm d, \bm k, \bm s'})$.
\item[b)] If $\bm k = (k_1, \ldots, k_{L-1}, 1)$ and $\bm k' = (k_1, \ldots, k_{L-1})$, $\bm d' = (d_0, \ldots, d_{L-1})$, $\bm s' = (s_1, \ldots, s_{L-1})$,
then $\pi(\mathcal M_{\bm d, \bm k, \bm s}) = \pi(\mathcal M_{\bm d', \bm k', \bm s'})$. 
\end{enumerate}
\end{lemma}

\begin{proof}
This follows from the fact that $\pi(W_L \cdots W_1) = \pi_s(W_L) \pi_1(W_{L-1}\cdots W_1)$, where $s = s_{L-1} \ldots s_1$ is the stride of $W_{L-1}\cdots W_1$. Part a) holds since this fact is independent of $s_L$. Part~b) holds since $k_L=1$ implies that $\pi_s(W_L)$ is a degree-0 polynomial, which does not play a role in the factorization of the end-to-end filter. 
\end{proof}

\begin{corollary}
Let $s_{1} = \cdots = s_{L-1} = 1$ and $s_L$ arbitrary. 
Then the LCN architecture 
$(\bm d, \bm k, \bm s)$ satisfies the properties in \Cref{thm:fillingCircular} (with $s=s_L$) and \Cref{prop:boundaryCondition,prop:twoConvexCones}. 
\end{corollary}

The corollary is immediately implied by \Cref{lem:specialCases}~a). 
This covers the case where only the last layer has an arbitrary stride.
Similarly, \Cref{lem:specialCases}~b) says that if the last filter size is one, we can understand the function space from the smaller architecture that only consists of the first $L-1$ layers. 
For the remaining cases, we obtain the following result. 

\begin{proposition} 
\label{prop:nonfillingstride}
If $k_L > 1$ and $s_i > 1$ for some $i \leq  L-1 $,  the LCN architecture is non-filling. 
In fact, the function space is a lower-dimensional semi-algebraic subset of~$\sM_{(d_0,d_L),k,s}$. 
\end{proposition} 

\begin{proof} 
Consider the last product $W_L \cdot W'$ where $W' = W_{L-1} \cdots W_1$.
We write $k' = k_1 + \sum_{l=2}^{L-1} (k_l-1)\prod_{m=1}^{l-1} s_m $ for the filter size of $W'$ and $s' = s_1 \cdots s_{L-1} > 1$ for its stride.
Even if $W'$ is an arbitrary convolutional matrix with filter size and stride $(k',s')$ (\ie, even if the LCN architecture from the first $L-1$ layers is filling),
the product $W_L \cdot W'$ is \emph{not an arbitrary} convolutional matrix with filter size $k = k'+(k_L-1)s'$ and stride $s = s's_L$.
Otherwise, every polynomial $p \in \RR[\x,\y]_{k-1}$  could be factored into a real polynomial $p'$ of degree $k'-1$ in $\x,\y$ and a real polynomial $p_L$ of degree $k_L-1$ in $\x^{s'},\y^{s'}$, by \Cref{rem:polynomialCorrespondenceHigherStride}. 
Modding out global scaling, we may assume that the polynomial $p$ and its two factors $p'$ and $p_L$ are monic, so their remaining degrees of freedom (\ie, the number of their remaining coefficients) are $k-1$ for $p$, $k'-1$ for $p'$, and $k_L-1$ for $p_L$.
Since $k_L > 1$ and $s' > 1$, we have that $(k'-1)+(k_L-1) < k'+(k_L-1)s'-1 = k-1$,
so the monic polynomials $p'$ and $p_L$ do not have enough degrees of freedom for their product to yield all monic polynomials $p$.
This dimension count shows that $\dim (\mathcal M_{\bm d, \bm k, \bm s}) \leq k'+k_L -1 < k = \dim(\sM_{(d_0,d_L),k,s}). $
\end{proof}

For the non-filling architectures in \cref{prop:nonfillingstride}, 
$\sM_{\bm d, \bm k, \bm s}$ has a strictly smaller dimension than 
$\sM_{(d_L, d_0), k,s}$ and is thus cut out by polynomial equations and inequalities.
This is in contrast to the non-filling architectures appearing for stride one in \Cref{thm:fillingCircular}: in that case, 
$\sM_{\bm d, \bm k, \bm s}$ has the same dimension as 
$\sM_{(d_0,d_L),k,s}$ and is cut out only by polynomial inequalities. 

\begin{example} 
\rm
Consider the LCN with $\bm d = (5,2,1)$, $\bm k =(3,2)$, $\bm s =(2,1)$. 
Note $s_2=1$ is inconsequential by \Cref{lem:specialCases} a). 
This setting is similar to \Cref{example:larger} where we considered $\bm s = (1,1)$. In that case, the architecture was filling, but this is no longer the case for stride size $2$ as we will see. 
In the current example, the function space $\sM\in \mathbb{R}^{1\times 5}$ consists of matrices 
\begin{align}
\label{eq:strideTwo}
    \overline{W} = W_2 W_1 = \begin{bmatrix}
        d & e
    \end{bmatrix}  \begin{bmatrix} 
        a & b & c & 0 & 0 \\
        0 & 0 & a & b & c
    \end{bmatrix} 
    = \begin{bmatrix} 
        ad & bd & ae+cd & be & ce 
    \end{bmatrix}.
\end{align}
 We see that the coefficients in~\cref{eq:strideTwo} are given by
\begin{equation}\label{eq:poly_stride}
\pi(\overline W) = (d \x^2 + e \y^2)(a \x^2+b\x\y+c\y^2). 
\end{equation}
The architecture is filling if and only if every real quartic polynomial 
$A\x^4 + B\x^3\y + C\x^2\y^2 + D\x\y^3 + E\y^4$ 
can be factored into two real quadratic polynomials as above. 
However, it is easy to verify that the coefficients of the polynomial in~\cref{eq:poly_stride} 
satisfy the equation $AD^2+B^2E = BCD$. 
In fact, we can fully characterize the function space as follows: 
\begin{align}
    \label{eq:strideTwoFunctionSpace}
     \mathcal{M} = \{ (A,B,C,D,E) \in \mathbb{R}^5 \colon AD^2+B^2E = BCD \text{ and }  C^2 \geq 4AE \}.
\end{align}
Indeed, the coefficients in \cref{eq:strideTwo} satisfy $(ae+cd)^2-4adce = (ae-cd)^2 \geq 0$. 
For the other containment, we consider a tuple $(A, \ldots, E)$ in the right-hand side of \cref{eq:strideTwoFunctionSpace} and explicitly find a factorization as in~\cref{eq:poly_stride}. 
For the case $E=D=0$, set $a=A$, $b=B$, $c=C$, $d=1$, $e=0$. 
For the case $E=0, D\neq0$, set $a= C/D$, $b=1$, $c=0$, $d=B$, $e=D$, which works since $AD=BC$. 
It remains to show the case $E \neq 0$. We may assume without loss of generality  that $E=1$ and we set $c=1, e=1, b = D$. To determine $a$ and $d$, we distinguish 
four cases:
\begin{itemize}
    \item $D \neq 0, B \neq 0$: Set $d = B/D$ and $a = AD/B$.  
    \item $D \neq 0, B = 0$: Set $d = 0$ and $a = C$.  
    \item $D = 0, A \neq 0$: 
    At least one of the two values $\frac{C}{2} \pm \sqrt{\frac{C^2}{4}-A}$ is non-zero since $A \neq 0$. Set $d$ to be such a non-zero value and $a = A/d$. 
    \item $D = 0, A=0$: Set $a = C$ and $d=0$.  
\end{itemize}
Using the equation and inequality in the right-hand side of \cref{eq:strideTwoFunctionSpace}, we verify that these four cases provide factorizations as in \cref{eq:poly_stride}. 
\end{example}

\subsection{Higher dimensional convolutions}
\label{sec:higher_dimension}

We briefly discuss convolutions over inputs having a $D$-dimensional structure, $D>1$ (\eg, raster images). The output of each layer is now a tensor of order $D$. Thus, we identify $\RR^{d_l} \cong \RR^{\Lambda_l}$ with $\Lambda_l = [d^1_l] \times \cdots \times [d^D_l]$. Each filter is also a tensor $w_l \in \mathbb{R}^{\lambda_l}$ with $\lambda_l = [k^1_l]\times \cdots \times[k^D_l]$. The convolution of $x \in \RR^{\Lambda_{l-1}}$ with $w_l \in \RR^{\lambda_l}$ is 
\begin{equation}\label{eq:D_conv}
(\alpha_l(x))_i = \sum_{j\in \lambda_l} w_{l,j} \cdot x_{is_l+j}, \quad i \in \Lambda_l,
\end{equation}
where $s_l  \in \mathbb{Z}_{>0}$ is the stride of the convolution (cf.\ \cref{eq:convolutionDefinition} for 1D convolutions). The map~\eqref{eq:D_conv} is linear in $x$ and can be represented as a \emph{convolutional tensor} $T_l \in \RR^{\Lambda_{l} \times \Lambda_{l-1}}$ of order $2D$. The dimensions of the tensor $T_l$ satisfy~\eqref{eq:dimensionConstraintsToeplitz} along each dimension: $d_l^{h} = \frac{d_{l-1}^{h} - k_l^{h}}{s_l} + 1$.

\begin{example} \label{ex:tensorConv} \rm
Let $D=2$ and let us consider a filter $w$ of size $2 \times 2$ with stride one. 
The convolution of an input $ x= \left[\begin{smallmatrix}* & * \\ * \circ& * \circ\\ \circ& \circ\end{smallmatrix}\right]$ 
of size $3 \times 2$ with the filter $w$
yields an output of size $2 \times 1$:
Its first resp.\ last entry is the inner product of $w$ with the first resp.\ last four entries of $x$ (marked by $*$ resp. $\circ$). 
The associated convolutional tensor $T$ has size $2 \times 1 \times 3 \times 2$ and its entries (shown in two slices) are as follows:
\begin{align*}
    T_{00::} = \begin{bmatrix}
    w_{00} & w_{01} \\  w_{10} & w_{11} \\ 0 & 0 
    \end{bmatrix}
    \quad \text{ and }
    \quad 
    T_{10::} = \begin{bmatrix}
    0 & 0 \\ w_{00} & w_{01} \\  w_{10} & w_{11}
    \end{bmatrix}.
\end{align*}
\end{example}

In the following, we assume that all strides are one. Similar to the one-dimensional case, we can represent convolutions using polynomials. If $w \in \RR^{\lambda_l}$, we define
\begin{equation}\label{eq:poly_tensor}
\pi(w) = \sum_{i \in \lambda_l}  w_{i_1, \ldots, i_D} \, \x_1^{i_1} {\y_1^{k_l^1-1-i_1}} \cdots \x_D^{i_D} {\y_D^{k_l^D-1-i_D}}  \,\, \in \,\, \RR[\x_1,{\y_1},\ldots,\x_D, {\y_D}],
\end{equation}
where the weights of the filter are the coefficients of a multivariate polynomial  {that is multi-homogeneous, that is, it has degree $k_l^j-1$ in each pair $\x_j,\y_j$.}
{For instance, in \Cref{ex:tensorConv} we have 
$\pi(w) = w_{00} {\y_1\y_2} +w_{10} \x_1 {\y_2} + w_{01} {\y_1} \x_2 + w_{11} \x_1\x_2$.}
If $T$ is a convolutional tensor with filter $w$, we also write $\pi(T)$ for the polynomial in~\eqref{eq:poly_tensor}. 
The following is a direct generalization of~\Cref{prop:deep_poly_multiplication}. 

\begin{proposition} Consider a collection of convolutional tensors $T_1,\ldots, T_L$ with compatible sizes and stride one.  
The tensor $\overline T = T_L \circ \cdots \circ T_1$ corresponding to the composition of the associated linear maps~\eqref{eq:D_conv} is a convolutional tensor and satisfies $\pi(\overline T) =\pi(T_L) \cdots \pi(T_1)$.
\end{proposition}
\begin{proof} The proof follows from the  composition of two maps as in~\eqref{eq:D_conv} with stride one:
\begin{align*}
\begin{split}
(\alpha_2 \circ \alpha_1) (x)_i &= \sum_{j \in \lambda_2} w_{2,j} \cdot \alpha_1(x)_{i+j} = \sum_{j \in \lambda_2} w_{2,j} \cdot \sum_{\ell \in \lambda_1} w_{1,\ell} x_{(i + j)+\ell}\\ 
&= \sum_{j \in \lambda_2} \sum_{\ell \in \lambda_1} w_{2,j} w_{1,\ell} x_{i + (j + \ell)} = \sum_{m \in \mu} u_{m} \cdot x_{i + m},
\end{split}
\end{align*}
where $\mu = [k_1^1+k_2^1-1] \times \cdots \times [k_1^D + k_2^D -1]$ and $u \in \RR^\mu$ is the filter with entries
$
    u_m = \sum_{j \in \lambda_1,  \ell \in \lambda_2, j + \ell = m} w_{2,j} w_{1, \ell}.
$
From this expression it follows that $\pi(u) = \pi(w_2) \pi(w_1)$.
\end{proof}

\begin{corollary} 
\label{prop:tensorFunctionSpace}
Given an LCN architecture with $L$ layers, inputs of format $[d_0^1]\times \cdots[d_0^D]$, and stride one, we write $(k_i^1, \ldots, k_i^D)$ for the filter size in the $i$-th layer. 
The function space can be identified with the family of polynomials $P \in \RR[\x_1,{\y_1},\ldots,\x_D, {\y_D}]$ that can be factored as 
$P = Q_L \cdots Q_1$, 
{where $Q_i$ is homogeneous of degree $k_i^j-1$ in the variables $\x_j,\y_j$.}
\end{corollary}

The natural ambient space of the LCN function space is the set of polynomials $P\in \RR[\x_1,{\y_1},\ldots,\x_D, {\y_D}]$ 
that are homogeneous of degree $k^j-1$ in  $\x_j,\y_j$
where $k^j:= (k_1^j-1) + \cdots + (k_L^j-1)+1$. 
We denote this vector space by 
$\mathcal{M}_{(d_0,d_L),k,s}$, where 
$d_0 = (d^1_0,\ldots, d^D_0)$, $ d_L = (d^1_L,\ldots, d^D_L)$, $k=(k^1,\ldots, k^D)$, $s =(1,\ldots, 1)$. 
Using~\Cref{prop:tensorFunctionSpace}, we show that for higher-dimensional convolutions and at least two layers (and excluding trivial cases where filters have size one) the LCN function space is always a lower-dimensional semi-algebraic subset of $\mathcal{M}_{(d_0,d_L),k,s}$ and in particular is not filling.

\begin{corollary}
Consider an LCN architecture with stride one, $D > 1$, and $L > 1$, where no layer has trivial filter size $(1, \ldots, 1)$.
If there are distinct $j_1, j_2 \in \{ 1, \ldots, D \}$ such that $k^{j_1} > 1$ and $k^{j_2} > 1$, the architecture is not filling. In fact, the function space is a lower-dimensional semi-algebraic subset of $\mathcal{M}_{(d_0,d_L),k,s}$.
\end{corollary}

\begin{proof}

The statement is a consequence of the fact that multivariate polynomials are generically irreducible. 
Since at least two layers have non-trivial filter sizes, the function space is contained in the set of polynomials that are reducible over $\mathbb{C}$, which is Zariski closed and thus lower-dimensional in the ambient space $\mathcal{M}_{(d_0,d_L),k,s}$. 
As before, the fact that the function space is semi-algebraic follows from Tarski-Seidenberg. 

For completeness, we give a formal argument showing that a generic polynomial $P$ in the ambient space $\mathcal{M}_{(d_0,d_L),k,s}$ is irreducible. 
We  may assume without loss of generality that $k^j > 1$ for all $1 \le j \le n$ and that $k^j = 1$ for all $n < j \le D$. 
The latter means that the variables $\x_j, \y_j$ for $n<j\leq D$ do not appear in the polynomial $P$; see \cref{eq:poly_tensor}. 
Hence, the zero locus of $P$ is a hypersurface in $(\mathbb{P}^1)^n$ of multidegree $(k^1-1, \ldots, k^n-1)$. 
Given the assumptions of the architecture, note that $n \ge 2$. 
To proceed, we consider the Segre-Veronese embedding $(\mathbb{P}^1)^n \hookrightarrow \mathbb{P}^{N-1}$ that maps $(\x_j:\y_j)_{j=1}^n$ to the $N$-tuple of all monomials that are homogeneous of degree $k^j-1$ in each pair $\x_j,\y_j$. 
Under this map, the polynomial $P$ becomes linear in the new coordinates of $\mathbb{P}^{N-1}$. 
In turn, the zero locus of $P$ is a hyperplane section of $(\mathbb{P}^1)^n$ embedded in $\mathbb{P}^{N-1}$.  
Now, Bertini's Theorem \cite[Ch. II Thm. 8.18 + Ch. III Rem. 7.9.1]{hartshorne} states that, for smooth complex projective varieties of dimension at least two (in our case, $(\mathbb{P}^1)^n$ with $n\geq2$), generic hyperplane sections are irreducible. 
This implies that a generic polynomial $P$ in $\mathcal{M}_{(d_0,d_L),k,s}$ is irreducible over $\mathbb{C}$.
\end{proof}

\begin{example}\rm
Consider an LCN architecture with $3 \times 3$ matrices as inputs (\ie, $D=2$ and $d_0 = (3,3)$), $L=2$ layers, filter sizes $(k_1^1, k_1^2) = (2,2) = (k_2^1, k_2^2)$, and strides one.
The two convolutions are maps $\RR^{3 \times 3} \to \RR^{2 \times 2} \to \RR^{1 \times 1}$.
Their composition is a convolution with filter size $3 \times 3$.
Hence, the natural ambient space of the LCN function space is the 9-dimensional vector space of polynomials in two {pairs of} variables such that the degree in each {pair is two},
and the function space is the subset of polynomials that factor into two {polynomials that are linear in each pair}.
The function space has codimension two in the ambient space.
A \texttt{Macaulay2}~\cite{m2} computation reveals that its Zariski closure has degree 10 and that its vanishing ideal is minimally generated by 39 homogeneous polynomials of degree~6.
\end{example}

\section{Optimization} 
\label{sec:optimization-function} 
We study the optimization of an objective function, or \emph{loss}, using LCNs with 1D convolutions and stride one. 
A {loss} $\L(\theta)$ is a function of the filters  $\theta=(w_1,\ldots,w_L)$ that can be written as a composition $\L = \ell \circ \mu$, where  $\mu = \mu_{\bm{d},\bm{k},\bm{s}}$ is the polynomial map described in~\cref{eqn:Wbar} and $\ell$ is a smooth function in matrix space (so $\L$ only depends on the end-to-end convolutional matrix $\overline W = \mu(\theta)$). 
This includes not only the square loss discussed further in \Cref{sec:squareloss}, but also classification losses (such as cross-entropy) or any other smooth function  that depends only on the output of the LCN. A tuple of filters $\theta=(w_1,\ldots,w_L)$ is said to be \emph{critical} for the loss $\L$ if all partial derivatives of $\L$ vanish at $\theta$. 

In~\Cref{sec:critPts} we analyze critical points of the loss function, explaining their relation to polynomials with repeated roots. In \Cref{sec:training-dynamics}, we discuss the training dynamics and invariants in parameter space for LCNs, which allow us to formulate local Riemannian gradient dynamics in function space.

\subsection{Critical points in parameter and function space}
\label{sec:critPts}
In this section, we discuss how critical points of an LCN 
correspond to polynomials that have repeated roots. 
We provide the proofs of the following proposition and theorems at the end of this section.

\begin{proposition} 
\label{prop:critical-points}
Let $\L = \ell \circ \mu$ be a loss function for an LCN with stride one, where  $\ell$ is a smooth and convex function on the space of convolutional matrices. If $\theta = (w_1,\ldots,w_L)$ is critical for $\L$, then one of the following holds: 
\begin{enumerate}
    \item $\theta$ is a \emph{global minimum} for $\L$, or
    \item for some $i \neq j$, the homogeneous polynomials $\Pi(w_i)$ and $\Pi(w_j)$ have a common factor.
\end{enumerate}
\end{proposition}

Polynomials with multiple roots form in fact the boundaries of non-filling architectures (as shown in \Cref{prop:boundaryCondition}), however we will see that they arise as critical points more generally: \emph{Every} critical point $\theta$ for the loss of an LCN corresponds to a critical point $\mu(\theta)$ of $\ell$ restricted to one of the multiple root loci $\Delta_\lambda$, \ie, $\mu(\theta) \in \Crit(\ell |_{\Pi^{-1}(\Delta_{\lambda})} )$, {even if $\mu(\theta)$ does not necessarily lie on the boundary of the function space.}
In other words, $\mu(\theta)$ is a smooth point of $\Pi^{-1}(\Delta_\lambda)$ and the differential $d \ell$ at $\mu(\theta)$ vanishes on the tangent space $T_{\mu(\theta)}\Pi^{-1}(\Delta_\lambda)$.

\begin{theorem}
\label{thm:criticalPointsAreOnMultipleRootLoci}
Let $\L = \ell \circ \mu$ be a loss function for an LCN with stride one. Let $\theta = (w_1,\ldots,w_L) \in \Crit(\L)$ and let $\overline w = \mu(\theta) \in \RR^{k}$ be the end-to-end filter. Let $\lambda = (\lambda_1, \ldots, \lambda_r)$ be the partition of $k-1$ corresponding to the root structure of $\Pi(\overline w)$. Then $\overline w \in \Crit(\ell |_{\Pi^{-1}(\Delta_{\lambda})} )$. 
\end{theorem}

To reverse \Cref{thm:criticalPointsAreOnMultipleRootLoci}, we characterize which critical points on the multiple root loci come from critical points of a given LCN. 
\begin{theorem} 
\label{thm:MultipleRootLocigivecriticalpoints}
Let $\mathcal L = \ell \circ \mu$ be a loss function for an LCN with stride one.
Let  $\lambda$ be  a partition of $k-1$ and let $\overline w \in \Crit(\ell |_{\Pi^{-1}(\Delta_{\lambda})} )$. If there exists $\theta = (w_1,\ldots,w_L)$ such that $\mu(\theta) = \overline w$ and each polynomial $\Pi(w_1),\ldots,\Pi(w_L)$ has no repeated roots, then $\theta \in \Crit(\L)$.
\end{theorem}
Moreover, if the loss $\ell$ in function space is sufficiently generic (\eg, if the training data is sufficiently generic), then we do not expect there to be any further critical points of the LCN other than the ones described in \Cref{thm:MultipleRootLocigivecriticalpoints}.
We provide a formal argument for this assertion when $\ell=\ell_{X,Y}$ is the square loss (defined in \cref{eq:objective}) with generic training data.

\begin{theorem} 
\label{thm:noDoubleRootsInFilters}
Let $X\in\mathbb{R}^{d_0\times N}$ and $Y\in \mathbb{R}^{d_L\times N}$ be generic data matrices with $N \geq d_0$ and let $\mathcal L = \ell_{X,Y} \circ \mu$ be the square loss for an LCN with stride one.
Then every $\theta = (w_1,\ldots,w_L) \in \Crit(\L)$ satisfies that each polynomial $\Pi(w_1),\ldots,\Pi(w_L)$ has no repeated roots.
\end{theorem} 

To summarize, a critical point $\overline{w}$ of a generic loss $\ell$ on a multiple root locus $\Pi^{-1}(\Delta_\lambda)$ comes from a critical point of the LCN if and only if
$\Pi(\overline{w})$ can be factored according to the LCN architecture such that no factor has a double root.
Whether a polynomial with \rrmp\, $(\rho\mid\gamma)$ can be factored according to the LCN architecture without double roots in any of the factors depends on a purely combinatorial property of $(\rho\mid\gamma)$ and the partition $(k_1-1, \ldots, k_L-1)$ of $k-1$ that is given by the architecture.
Indeed, the factorization of the polynomial is equivalent to the placing of $\rho_i$ balls of size $1$ and color $i$ (for $1 \leq i \leq r$) and $\gamma_j$ balls of size $2$ and color $r+j$ (for $1 \leq j \leq c$) into $L$ bins of sizes $k_1-1, \ldots, k_L-1$ such that no bin contains two balls of the same color. If this placement of balls into bins is possible, we say that $(\rho\mid\gamma)$ is \emph{compatible} with the LCN architecture.

All in all, the previous \Cref{thm:criticalPointsAreOnMultipleRootLoci,thm:MultipleRootLocigivecriticalpoints,thm:noDoubleRootsInFilters} show that the critical points of an LCN with a generic loss $\ell$ correspond exactly to the critical points on multiple root loci with an \rrmp\, that is compatible with the architecture:

\begin{corollary}
\label{cor:rrmp}
Let $X\in\mathbb{R}^{d_0\times N}$ and $Y\in \mathbb{R}^{d_L\times N}$ be generic data matrices with $N \geq d_0$ and let $\mathcal L = \ell_{X,Y} \circ \mu$ be the square loss for an LCN with stride one.
Let $(\rho \mid \gamma)$ be such that $\lambda = \lambda_{\rho|\gamma} $ is a partition of $k-1$.
If $(\rho \mid \gamma)$ is compatible with the LCN architecture, then for every $\overline w \in \Crit(\ell |_{\Pi^{-1}(\Delta_{\lambda})} )$  with \rrmp\, $(\rho \mid \gamma)$ there exists some $\theta \in \mu^{-1}(\overline{w}) \cap \Crit(\L)$.
Otherwise, every $\overline w \in \Crit(\ell |_{\Pi^{-1}(\Delta_{\lambda})} )$  with \rrmp\, $(\rho \mid \gamma)$ satisfies $\mu^{-1}(\overline{w}) \cap \Crit(\L) = \emptyset$.
\end{corollary}

\begin{example} 
\label{ex:partitionVsArchitecture} \rm
We compare the LCN architectures with $\bm k = (3,2,2)$ and $\bm k' = (4,2)$. We discussed in \Cref{ex:quartic} that the function space associated with these architectures is the same. The polynomial $\Pi(\overline{w}) = p_1 p_2^3$, where $p_i\in \mathbb{R}[\x,\y]_1$, has \rrmp\, $(13 \mid 0)$.
We can find a factorization compatible with the architecture $\bm k = (3,2,2)$ where each polynomial associated with a filter does not have any repeated roots: $\Pi(\overline{w}) = (p_1 p_2) \cdot p_2 \cdot p_2$.
In contrast, for $\bm k' = (4,2)$,  any factorization of $\Pi(\overline w)$ according to the architecture will be such that the cubic polynomial associated with the filter of size $4$ will have at least one pair of repeated roots.

For both architectures, we now list all compatible $(\rho \mid \gamma)$ by providing a feasible factorization; the remaining ones are marked with a dash: 

\noindent
    \begin{tabular}{c|ccccccccc}
        $\rho|\gamma$ & $1111|0$ & $112|0$ & $22|0$ & $13|0$ & $4|0$ & $11|1$ & $2|1$ & $0|2$ & $0|11$   \\ \hline
         \small $(3,2,2)$ & \small $p_1p_2 {\cdot} p_3 {\cdot} p_4$ & \small $p_1p_2 {\cdot} p_3 {\cdot} p_3$ & \small $p_1p_2 {\cdot} p_1 {\cdot} p_2$ & \small $p_1p_2 {\cdot} p_2 {\cdot} p_2$ & $-$ & \small $q_1 {\cdot} p_1 {\cdot} p_2$ & \small $q_1 {\cdot} p_1 {\cdot} p_1$ & $-$ & $-$ \\
         \small  $(4,2)$ & \small $p_1p_2p_3{\cdot}p_4$  & \small $p_1p_2p_3{\cdot}p_3$ & $-$
          & $-$  & $-$ & \small $p_1q_1{\cdot}p_2$ & \small $p_1q_1{\cdot}p_1$ & $-$ & $-$
    \end{tabular}
\end{example}

To prove the theorems above, we make use of the composition rule of differentials: 
\begin{equation*}
d \L = d \ell \circ d \mu_{\bm{d},\bm{k},\bm{s}}.
\end{equation*}
Even if the parameterization map $\mu$ is surjective (\ie, the function space $\mathcal{M}=\mathcal M_{\bm d, \bm k, \bm s}$ is filling), the differential $d \mu$ (equivalently, the Jacobian matrix of $\mu$) might not have maximal rank everywhere. This implies that the parameterization may give rise to critical points for $\L$ in addition to the global minimum even if the function space $\mathcal M$ is a vector space and $\ell$ is convex. 
Following~\cite{geometryLinearNets}, we call this kind of critical points \emph{spurious}. 
In contrast, a critical point $\theta$ for $\L$ is $\emph{pure}$ if $\mu(\theta)$ is a critical point of the loss in function space, \ie, of $\ell |_{\mathcal{M}}$. 

Given an end-to-end filter $\overline w \in \RR^k$ of an LCN, we describe the set of parameters $\theta$ such that $\mu(\theta) = \overline{w}$, \ie, the fiber of the parameterization.
We say that two parameters $\theta = (w_1,\ldots, w_L)$ and $\theta'=(w_1',\ldots,w_L')$ are equivalent up to scalar factors if there exists  $(\kappa_1,\ldots,\kappa_L) \in (\RR^*)^L$ with $\prod_{i=1}^L \kappa_i = 1$ such that $w_i = \kappa_i w_i'$.

\begin{lemma} \label{lem:fiber}
Let $\overline w \in \RR^k \setminus\lbrace 0 \rbrace$ be an end-to-end filter of an LCN with stride one. Then the fiber $\mu^{-1}(\overline w)$ consists of a finite set of scaling equivalence classes. Moreover, either all points in the same equivalence class are critical or none are.
\end{lemma}

\begin{proof} The fact that $\mu^{-1}(\overline w)$ consists of a finite set of scaling equivalence classes follows from the uniqueness of the decomposition of $\pi(\overline w) \in \RR[\x,\y]$ into irreducible factors: Each equivalence class corresponds to a way of aggregating factors of $\pi(\overline w)$ into polynomials of degrees $k_1-1,\ldots,k_L-1$, where $(k_1,\ldots,k_L)$ are the filter widths of the LCN. For the second claim, we use the fact that if $\theta$ and $\theta'$ are equivalent under rescaling, then the image of the differential $d\mu$ at these points is the same. 
This will follow from our characterization of the image of $d \mu$ in~\eqref{eq:image_differential_ideal}. Hence,  $d\mathcal L(\theta) = d \ell(\overline w) \circ d \mu(\theta) = 0$ if and only if \mbox{$d\mathcal L(\theta') = d \ell(\overline w) \circ d \mu(\theta') = 0$.}
\end{proof}

\begin{example}\rm 
We discuss the parameters corresponding to particular points in the function space of the LCN with $\bm k = (2,2)$ from \Cref{example:minimal}.
Using the identification $\pi$ with polynomials,
a point in the interior of the function space is a quadratic polynomial with two distinct real roots: $P = p_1p_2$, where $p_1 = a\x+b\y$ and $p_2= c\x+d\y$ are not scalar multiples of each other.
The fiber $\mu^{-1}(P)$ of the parameterization has four connected one-dimensional components:
$\{ (\kappa p_1, \frac{1}{\kappa} p_2) : \kappa > 0 \}$, 
$\{ (\kappa p_1, \frac{1}{\kappa} p_2) : \kappa < 0 \}$, 
$\{ (\kappa p_2, \frac{1}{\kappa} p_1) : \kappa > 0 \}$, and 
$\{ (\kappa p_2, \frac{1}{\kappa} p_1) : \kappa < 0 \}$.
\\\indent
The points on the boundary of the function space are exactly those $P=p_1p_2$ where $p_1$ and $p_2$ are linearly dependent. In other words, the parameters mapping to boundary points are those which satisfy $bc - ad = 0$.
We will see in \Cref{prop:rank_differential} that these are exactly the parameters where the Jacobian of the parameterization $\mu$ drops rank.
\\\indent
Hence, a non-zero point on the boundary of the function space is of the form $P = \pm p^2$ for some linear term $p \neq 0$. 
Thus, the fiber $\mu^{-1}(P)$ has two connected one-dimensional components:
$\{ (\pm \kappa p, \frac{ 1}{\kappa} p) : \kappa > 0 \}$ and 
$\{ (\pm \kappa p, \frac{1}{\kappa} p) : \kappa < 0 \}$.
Finally, the fiber $\mu^{-1}(0)$ over the origin is the union of the two planes
$\{ (p, 0) : p \in \RR[\x,\y]_1 \}$ and $\{ (0, p) : p \in \RR[\x,\y]_1 \}$.
\end{example}

\begin{remark}
\label{rem:fiberstructure}
In the case of fully-connected linear networks, spurious critical points are always \emph{saddles}. More precisely, for a fully-connected network, a loss function $\L = \ell \circ \mu$ where $\ell$ is a smooth convex function has non-global minima if and only if the restriction of $\ell$ on the function space has non-global minima~\cite[Proposition~10]{geometryLinearNets}.
In Example~\ref{ex:badMinimum}, we will see that this is \emph{not} true for LCNs, and it is possible for the loss function $\L$ to have non-global minima even if the function space is filling.

The intuitive geometric reason for this different behavior between fully-connected and convolutional linear networks lies in the very different structure of the fibers. 
For fully-connected networks, as shown in~\cite[Proposition~9]{geometryLinearNets},
it is possible to perturb a spurious critical point within the fiber to obtain another point that is not a critical point for the loss. 
This in turn implies that the original spurious critical point could not have been a minimum for the loss. In the case of LCNs, however, this argument does not apply, by \Cref{lem:fiber}.
 Thus, for LCNs, it is possible for spurious critical points to be (non-global) local minima. 
\end{remark}
\begin{example} 
\label{ex:badMinimum}\rm
Let us consider an LCN with filter sizes $\bm k = (2,3)$. This is a filling architecture.
In terms of polynomials, its function space is the set of all cubic polynomials. 
The two layers of the LCN factor such a polynomial into a linear and a quadratic term. 
We fix the target polynomial $u = (\x+\y)(\x^2 + 1/10 \, \y^2)$.
In the space of polynomials $\RR[\x,\y]_{3}$, we use the loss $\ell(p)$ that is the squared sum of the coefficients of $p-u$. 
We will see that $\L = \ell \circ \mu$ has a non-global minimum, although $\ell$ clearly only has one local and global minimum. 
\\\indent
Almost every cubic polynomial $p$ can be factored as $p = q_1 q_2$ where $q_1 = \x+a\y$ and $q_2 = b\x^2+c\x\y+d\y^2$. 
We assume that $q_1$ is monic (\ie, the coefficient for $\x$ is $1$) since there is otherwise a scaling ambiguity among $q_1$ and $q_2$ which does not affect the product $q_1q_2$ nor the property of critical points of $\L$. Our objective function $\L$ thus becomes
\begin{align*}
    \L: \RR^4 \longrightarrow \RR;\quad
    (a,b,c,d) \longmapsto (b-1)^2 + (ab+c-1)^2 + (ac+d-\nicefrac{1}{10})^2 + (ad-\nicefrac{1}{10})^2.
\end{align*}
A direct computation reveals that $\L$ has 10 critical points over the complex numbers:
Three of them (1 real, 2 complex) correspond to factorizations of the target polynomial $u$, so that $\mathcal L=0$. 
The other seven (3 real, 4 complex) are spurious critical points  and yield factorizations of seven different cubic polynomials.
One of the latter three real spurious critical points is
\begin{align*}
    a \approx 0.0578445483987, b \approx 1.0000187825172, c \approx 0.941829719725, d \approx 0.0511336556138.
\end{align*}
The Hessian of $\L$ at this point is positive definite, showing that it is indeed a local minimum. 
\end{example}

We now investigate the existence of spurious critical points for LCNs by studying the rank of the differential of the parameterization map $\mu$.

\begin{proposition}\label{prop:rank_differential} Let $\mu$ be the parameterization map of an LCN with stride one. Given  filter weights $\theta = (w_1,\ldots,w_L)$, we write $\overline w = \mu(\theta) \in \RR^k$ for the end-to-end filter. 
 Using the identification $\Pi$, the image of the differential $d \mu$ at the point $\theta$ is the set of polynomials
\begin{equation}\label{eq:image_differential_ideal}
H(\theta) := \left\{h \, g \,\colon h \in \RR[\x,\y]_{k-1-\deg(g)}\right\} \subset \RR[\x,\y]_{k-1},
\end{equation}
where $g = {\rm gcd}\left(\frac{\Pi(\overline w)}{\Pi(w_1)},\ldots,\frac{\Pi(\overline w)}{\Pi(w_L)}\right) \in \RR[\x,\y].$
In particular, $d\mu (\theta)$ is surjective if and only if ${\rm gcd}(\Pi(w_i),\Pi(w_j)) = 1$ for every pair $i,j \in \{1,\ldots,L\}$ with $i \ne j$.
\end{proposition}

To prove this result, we make use of the following homogeneous version of Bezout's identity.

\begin{lemma}\label{lemma:ideal_gcd} Let $p_1,\ldots,p_n$ be homogeneous polynomials in $\RR[\x,\y]$ of arbitrary degree and let
$g = \gcd(p_1,\ldots,p_n)$ and $l={\rm lcm}(p_1,\ldots,p_n)$. For any $d \ge \deg(l)-1$ the following sets of homogeneous polynomials in $\RR[\x,\y]_d$ are equal:
\[
\begin{aligned}
I_d := \{\alpha_1 p_1 + \cdots + \alpha_n p_n \colon \alpha_i \in \RR[\x,\y]_{d - \deg(p_i)} \}, \quad J_d := \{h g \colon h \in \RR[\x,\y]_{d-\deg(g)}\}.\\
\end{aligned}
\]
\end{lemma}
\begin{proof}
If $g = l$ (\ie \, $p_1=\cdots=p_n$), the assertion is clear. Hence, we assume that $\deg(g) < \deg(l)$.
We always have that $I_d \subset J_d$, since any algebraic combination of $p_1, \ldots, p_n$ is a multiple of $g$. For the converse, we observe that it is enough to prove that $J_{\deg(l)-1} \subset I_{\deg(l)-1}$ (if $hg$ belongs to the ideal generated by $p_1,\ldots,p_n$, then so do $\x hg$ and $\y hg$). 
Moreover, it is sufficient to assume 
that the linear form $\y$ is not a factor of $l$ (otherwise, we apply a general invertible linear transformation to the input variables, \eg\, $\x'=\x, \y' = \y + t\x$ for appropriate $t$).
We can thus safely \emph{dehomogenize} and obtain univariate polynomials $\tilde p_i = p_i(\x,1)$, $\tilde g = g(\x,1)$, $\tilde l = l(\x,1)$ in $\RR[\x]$. Note that these polynomials have the same degrees as their homogeneous counterparts and also that $\tilde g = \gcd(\tilde p_1,\ldots,\tilde p_n)$ and $\tilde l = {\rm lcm}(\tilde p_1,\ldots,\tilde p_n)$. 
By Bezout's identity, we can write $\tilde \beta_1 \tilde p_1 + \cdots + \tilde \beta_n \tilde p_n = \tilde g$ for some $\tilde \beta_i \in \RR[\x]$.
Thus, for any $\tilde h \in \RR[\x]$, we get a decomposition
$\tilde \alpha_1 \tilde p_1 + \cdots + \tilde \alpha_n \tilde p_n = \tilde h \tilde g$.
We now claim that we can assume $\deg(\tilde \alpha_i) < \deg(\tilde l) - \deg(\tilde p_i)$ whenever $\deg(\tilde h \tilde g) \le \deg(\tilde l)-1$. 
Indeed, if this is not the case, then we write $\tilde \alpha_i = \tilde q_i \tilde p_i^* + \tilde r_i$ with $\tilde p_i^* = \tilde l / \tilde p_i$ and $\deg(\tilde r_i) < \deg(\tilde p_i^*) = \deg(\tilde l) - \deg(\tilde p_i)$, and we obtain $(\sum_{i=1}^n \tilde q_i) \tilde l + \tilde r_1 \tilde p_1 + \cdots + r_n \tilde p_n = \tilde h \tilde g$. 
Since $\deg(\tilde h \tilde g) < \deg (\tilde l)$, we deduce that $\sum_{i=1}^n \tilde q_i = 0$, which yields a decomposition with the desired degree bounds. 
We can now \emph{homogenize} this relation so that the total degree is $\deg(l) - 1 = \deg(\tilde l)-1$ to obtain $\alpha_1 p_1 + \cdots +\alpha_n p_n = hg$, where
\[
\begin{aligned}
&\alpha_i(\x,\y) = \y^{\deg(\tilde l) - 1 -\deg(\tilde p_i)}\tilde \alpha_i\left(\frac{\x}{\y}\right), &&h(\x, \y) = \y^{\deg(\tilde l) - 1 -\deg(\tilde g)}\tilde h\left(\frac{\x}{\y}\right),\\
&p_i(\x,\y)=\y^{\deg(\tilde p_i)}\tilde p_i\left(\frac{\x}{\y}\right), &&g(\x,\y)=\y^{\deg(\tilde g)}\tilde g\left(\frac{\x}{\y}\right).
\end{aligned}
\]
Note that $p_i$ and $g$ are the original homogeneous polynomials from the statement. Since $\tilde h$ can be chosen arbitrarily in $\RR[\x]_{\le \deg(\tilde l) - 1 - \deg(\tilde g)}$, we have that $h(\x,\y)$ can be an arbitrary polynomial in $\RR[\x,\y]_{\deg(l) - 1 - \deg(g)}$. We thus conclude that $J_{\deg(l)-1} \subset I_{\deg(l)-1}$.
\end{proof}

\begin{proof}[Proof of \Cref{prop:rank_differential}]
For notational simplicity, we omit writing the map $\Pi$ and identify $\RR^k$ with $\RR[\x,\y]_{k-1}$. The differential of the parameterization can be written as
\begin{equation}\label{eq:differential_general}
d \mu(\theta)\colon \RR^{k_1 + \cdots + k_L} 
\longrightarrow  \RR^{k}; \quad 
(\dot w_1, \ldots, \dot w_L) 
\longmapsto \dot w_1\frac{\overline w}{w_1} + \cdots + \dot w_L\frac{\overline w}{w_L}, 
\end{equation}
where the product means polynomial multiplication and each $\dot w_i$ is simply a symbol for the coordinates of the tangent vector corresponding to the variable $w_i$. Since $\overline w$ is a common multiple of $\frac{\overline w}{w_1},\ldots, \frac{\overline w}{w_L}$, using~\Cref{lemma:ideal_gcd} we deduce that the image of~\eqref{eq:differential_general} is equal to $H(\theta)$ in~\eqref{eq:image_differential_ideal}. This shows that the differential is surjective if and only if $\gcd\left(\frac{\overline w}{w_1}, \ldots, \frac{\overline w}{w_L}\right)=1$.
The latter is equivalent to  ${\rm gcd}(w_i,w_j) = 1$ for every pair $i \neq j$. Indeed, any common factor of $w_i, w_j$ is a factor of each $\frac{\overline w}{w_m}$; conversely, if $q$ is an irreducible common factor of $\frac{\overline w}{w_1}, \ldots, \frac{\overline w}{w_L}$, then $q | w_i$ for some $i \in \{1,\ldots,L\}$, and from $q | \frac{\overline w}{w_i}$ we deduce $q | w_j$ for some $j \ne i$.
\end{proof}

We are finally ready to prove the main theorems of this section. 
As in the proof of \Cref{prop:rank_differential}, we omit writing the map $\pi$ and identify filters with polynomials. 

\begin{proof} [Proof of \Cref{prop:critical-points}]
Let $\overline W = \mu(w_1,\ldots,w_L)$. If $(d \ell \circ d \mu)(\theta) = 0$ then either $d \ell (\overline{W}) = 0$ or $d \mu(\theta)$ must not have full rank. The assumption that $\ell$ is convex means that if $d\ell(\overline W) = 0$ then $\overline W$ is a global minimum for $\ell$ and thus $\theta$ is a global minimum for $\L$. Alternatively, if $d \mu(\theta)$ does not have full rank, then by \cref{prop:rank_differential} the second condition holds.
\end{proof}

\begin{proof}[Proof of \Cref{thm:criticalPointsAreOnMultipleRootLoci}]
If $\lambda = (1,\ldots,1)$, then $\overline w$ has $k-1$ distinct roots. By \Cref{prop:rank_differential}, the differential $d\mu(\theta)$ has full rank, \ie, $\theta$ is a pure critical point. This in turn implies that $\overline w$ is a critical point for $\ell$ in $\Delta_{\lambda} = \RR[x,y]_{k-1}$. 

We next assume that $\lambda = (\lambda_1,\ldots,\lambda_r)$ is such that $\lambda_1 > 1$. Let $\overline w = q_1^{\lambda_1} \cdots q_r^{\lambda_r}$ be the decomposition of $\overline w$ into pairwise distinct complex linear forms $q_i \in \CC[\x, \y]_1$. Note that if $\overline w = \prod_{i=1}^r q_i^{\lambda_i}$ is real, then so is $\prod_{i=1}^r q_i^{\lambda_i - 1}$. We claim that the (real) tangent space of $\Delta_{\lambda}$ at the smooth point $\overline w$ is given by
\begin{equation}\label{eq:tangent_multiple_root_locus}
T_{\overline w} \Delta_{\lambda} = \left\{h \prod_{i=1}^r q_i^{\lambda_i - 1} \,\, \colon \,\, h \in \RR[\x,\y]_r\right\} \subset \RR[\x,\y]_{k-1}.
\end{equation}
Indeed, let us assume without loss of generality that $q_{i_1},\ldots,q_{i_{r'}}$ are real linear forms, while $q_{j_1}, q_{j_1+1}, \ldots,\allowbreak q_{j_{r''}}, q_{j_{r''}+1}$ are such that each $(q_{j_l}, q_{j_l+1})$ is a conjugate pair of complex linear forms, with $r'+2r'' = r$ (in particular we must also have $\lambda_{j_l} = \lambda_{j_l+1}$ for all $l=1,\ldots,r''$). We write $t_{j_l} = q_{j_l} q_{j_l+1}$ for the irreducible real quadratic form that is the product of the pair $q_{j_l}, q_{j_l+1}$.
In a neighborhood of $\overline w$, the multiple root locus $\Delta_{\lambda}$ is the image of
\begin{equation}\label{eq:parameterization_multi_root_locus}
(q_{i_1},\ldots,q_{i_{r'}}, t_{j_1},\ldots,t_{j_{r''}}) \mapsto \prod_{m=1}^{r'} q_{i_m}^{\lambda_{i_m}} \prod_{l=1}^{r''} t_{j_l}^{\lambda_{j_l}},
\end{equation}
as $q_{i_m} \in \RR[\x,\y]_1$ and $t_{j_l} \in \RR[\x,\y]_2$ vary. Arguing as in the proof of \Cref{prop:rank_differential} (based on \Cref{lemma:ideal_gcd}), we have the image of the differential of~\eqref{eq:parameterization_multi_root_locus} is the set of homogeneous polynomials of degree $k-1$ that are multiples  of $g' = \gcd\left(\frac{\overline w}{q_{i_1}}, \ldots, \frac{\overline w}{q_{i_{r'}}},\frac{\overline w}{t_{j_1}},\ldots,\frac{\overline w}{t_{j_{r''}}} \right)$. We finally observe that $g' = \prod_{i=1}^r q_i^{\lambda_i - 1}$.

Having shown~\eqref{eq:tangent_multiple_root_locus}, we recall that the image of the differential of the parameterization  $\mu$ is  $\mathrm{Im}(d\mu(\theta)) = \left\{h g \,\colon h \in \RR[\x,\y]_{k-1-\deg(g)}\right\} \subset \RR[\x,\y]_{k-1}$, where $g$ is as in~\Cref{prop:rank_differential}. To conclude the proof, we note that $T_{\overline w} \Delta_{\lambda} \subset \mathrm{Im}(d\mu(\theta))$ since $g \, | \, \prod_{i=1}^r q_i^{\lambda_i - 1}$. Indeed, the linear form $q_i$ can appear in the factorization of $g$ with multiplicity at most $\lambda_i - 1$, for if $q_i | w_j$ then $q_i^{\lambda_i}\not|\frac{\overline w}{w_j}$.
The fact that $T_{\overline w} \Delta_{\lambda} \subset \mathrm{Im}(d\mu(\theta))$ implies the statement of the theorem since $0=d\mathcal L(\theta) = d \ell(\overline w) \circ d \mu(\theta)$ implies $d \ell(\overline w)|_{T_{\overline w} \Delta_{\lambda}} = 0$, \ie, $\overline w$ is a critical point of $\ell|_{\Delta_{\lambda}}$. 
\end{proof}

\begin{proof}[Proof of  \Cref{thm:MultipleRootLocigivecriticalpoints}]
Below we show the following claim for $\overline w$ being a smooth point of $\Delta_{\lambda}$: 
For every $\theta \in \mu^{-1}(\overline w)$, it holds $\mathrm{Im}(d\mu(\theta)) = T_{w}(\Delta_\lambda)$ if and only if each $w_i$ has no repeated roots. Hence, under the assumptions in the theorem, $d \mathcal L(\theta) = d \ell(\overline w) \circ d \mu(\theta) = d \ell(\overline w)|_{T_{\overline w} \Delta_{\lambda}} = 0$. 

To prove our claim, we use the fact that $T_{\overline w}(\Delta_\lambda) = \{h \prod_{i=1}^r q_i^{\lambda_i - 1} \,\, \colon \,\, h \in \RR[\x,\y]_r\}$  as in \cref{eq:tangent_multiple_root_locus} if $\overline w = \prod_{i=1}^r q_i^{\lambda_i}$ and that $\mathrm{Im}(d\mu(\theta)) = \left\{h \, g \,\colon h \in \RR[\x,\y]_{k-1-\deg(g)}\right\}$, where $g$ is the gcd defined in \Cref{prop:rank_differential}.
We observe that $w_i$ has no repeated roots if and only if $\prod_{i=1}^r q_i^{\lambda_i - 1} | (\nicefrac{\overline{w}}{w_i})$.
This happens for every $i=1, \ldots, L$ if and only if
 $\prod_{i=1}^r q_i^{\lambda_i - 1} | g$.
Since $g|\prod_{i=1}^r q_i^{\lambda_i - 1}$ always holds, as seen in the proof of \Cref{thm:criticalPointsAreOnMultipleRootLoci}, the latter is equivalent to $\mathrm{Im}(d\mu(\theta)) = T_{\overline w}(\Delta_\lambda)$.
\end{proof}

\begin{proof}[Proof of \Cref{thm:noDoubleRootsInFilters}]
As we will see in \cref{eq:objective,eq:squareLossFilter,eq:filterLoss}, if  $XX^\top$ has full rank, any square loss $\ell_{X,Y}(\overline w)$ is equivalent up to a constant additive term to a function of the form $\ell_{\Sigma,u}(\overline w) = \|u - \overline w\|^2_{\Sigma}$, where $u$ is a target filter and $\langle \cdot, \cdot \rangle_{\Sigma}$ is a non-degenerate inner product on $\RR^{k}$.
{Here $k$ is the size of the filters $u$ and $\overline{w}$.}
We now show that for \emph{arbitrary} inner products $\langle \cdot, \cdot \rangle_{\Sigma}$ and \emph{generic} target filters $u$, any critical point $\theta = (w_1,\ldots,w_L)$ of $\mathcal L_{\Sigma,u} = \ell_{\Sigma,u} \circ \mu$ is such that each polynomial $w_1,\ldots,w_L$ has no repeated roots.\footnote{In the setting of~\Cref{ssec:squareloss}, we need to observe that the inner product only depends on the input data matrix $X \in \RR^{d_0 \times N}$, and that as the output data matrix $Y \in \RR^{d_L \times N}$ varies the target matrix $U = YX^\top(XX^\top)^{-1}$ can be an arbitrary matrix of size $d_L \times d_0$ due to the assumption $N \geq d_0$. In particular, for any non-degenerate inner product $\langle \cdot, \cdot \rangle_{\Sigma}$, all possible target filters $u$ arise.} Fixing the inner product arbitrarily, we consider for all $i=1,\ldots,L$ the set
\[
Z_i = \{(\theta, \overline w, u) \, \colon \, \theta \in \Crit(\mathcal L_{\Sigma,u}), \, \mu(\theta) = \overline w, \,\, \theta=(w_1,\ldots,w_L),\,\, w_i \text{ has a repeated root}\}.
\]
Note that this is an algebraic set. We claim that $\dim({\rm pr}_3(Z_i)) \le k-1$ holds for all $i=1,\ldots,L$, where ${\rm pr}_3$ denotes the projection on the third factor. This implies the statement of the theorem since it shows that the set of target filters $u$ such that any of the corresponding critical points has a filter with repeated roots has positive codimension in $\RR^k$. To prove our claim on the dimension of $Z_i$, we further consider the sets
\[
Z_{i,\lambda} = \{(\theta, \overline w, u) \, \colon (\theta, \overline w, u) \in Z_i, \,\, \overline w \in \mathrm{Reg}( \Delta_{\lambda} )\},
\]
where $\lambda$ is a partition of $k-1$ and $\mathrm{Reg}( \Delta_{\lambda})$ denotes the locus of smooth points on $\Delta_\lambda$. 
Clearly, $Z_i = \bigcup_{\lambda} Z_{i,\lambda}$.
As shown in the proof of~\Cref{thm:MultipleRootLocigivecriticalpoints}, if $\overline w$ is a smooth point of $\Delta_{\lambda}$ and $w_i$ has a repeated root, then $T_{\overline w} \Delta_{\lambda} \subsetneq {\rm Im}(d\mu(\theta))$.
The fiber $\mu^{-1}(\overline{w})$ consists of finitely many scaling equivalence classes (see \Cref{lem:fiber}) and all $\theta$ in the same equivalence class yield the same ${\rm Im}(d\mu(\theta))$.
We note that $\theta\in\Crit(\mathcal L_{\Sigma,u})$ means equivalently that $u - \overline w \perp_{\Sigma}{\rm Im}(d\mu(\theta))$, or in other words $u\in\overline{w}+{\rm Im}(d\mu(\theta))^{\perp_{\Sigma}}$.
We write $N_{\Sigma,\overline{w}}$ for the finite union of the affine spaces $\overline{w}+{\rm Im}(d\mu(\theta))^{\perp_{\Sigma}}$ where $\theta=(w_1,\ldots,w_L)\in\mu^{-1}(\overline{w})$ and some $w_i$ has a repeated root.
Hence, $u\in{\rm pr}_3(Z_{i,\lambda})$ if and only if there is some smooth point $\overline{w}$ on $\Delta_\lambda$ such that $u\in N_{\Sigma,\overline{w}}$.
Since $\dim(N_{\Sigma,\overline{w}})<\dim((T_{\overline w} \Delta_{\lambda})^{\perp_{\Sigma}})=k-\dim(\Delta_\lambda)$, we see that $\dim({\rm pr}_3(Z_{i,\lambda}))\leq \dim(\Delta_\lambda)+\max_{\overline{w}\in\mathrm{Reg}(\Delta_\lambda)}\dim(N_{\Sigma,\overline{w}})<k$. 
This concludes the proof. 
\end{proof}

\subsection{Training dynamics}
\label{sec:training-dynamics}

In this section, we discuss some matters related to the training dynamics of linear convolutional networks. We begin by proving the existence of dynamical invariants in parameter space. Similar results have been shown for linear networks~\cite{DBLP:journals/corr/abs-1910-05505} and ReLU networks~\cite{williams2019gradient}.
We restrict ourselves to networks with stride one.

\begin{proposition}\label{prop:invariants} Let $\L= \ell \circ \mu$ be a loss function for an LCN with stride one, where $\ell$ is a smooth function on the space of convolutional matrices. Let $\theta(t)$ be an integral curve in parameter space for the negative gradient field of $\L$, \ie, $\dot
\theta(t) = - \nabla (\L(\theta(t)))$. If we write $\theta(t) = (w_1(t),\ldots,w_L(t))$, where $w_i$ are filter vectors, then the quantities $\delta_{ij}(t) = \|w_{i}(t)\|^2 - \|w_j(t)\|^2$ for $i,j \in \{1,\ldots,L\}$ remain constant for all $t \in \RR$.
\end{proposition}
\begin{proof} For any $\kappa \in \RR \setminus \{0\}$, let $\nu_{ij}(\theta,\kappa) = (w_1,\ldots,\kappa w_i, \ldots, \nicefrac{1}{\kappa} \, w_{j},\ldots, w_L)$. Because the map $\mu$ is multilinear in the filters, we have that $\L(\nu_{ij}(\theta,\kappa))$ does not depend on $\kappa$. This implies that $\frac{d}{d\kappa}\L(\nu_{ij}(\theta,\kappa)) \equiv 0$ and in particular 
\[
0=\left[\frac{d}{d\kappa}\L(\nu_{ij}(\theta,\kappa))\right]_{\kappa=1} = \frac{d}{d\theta}\L(\theta) \cdot \left[\frac{d}{d\kappa}\nu_{ij}(\theta,\kappa)\right]_{\kappa=1} =\frac{d}{d\theta}\L(\theta) \cdot [0,\ldots,w_i, \ldots, -w_j,\ldots 0]^\top. 
\]
We now observe that
\[
\frac{d}{dt}\delta_{ij}(t) = \frac{d}{d\theta} \delta_{ij} \cdot \frac{d}{dt} \theta(t) = -2 \, [0,\ldots,w_i,\ldots,-w_j,\ldots,0] \cdot \frac{d}{d\theta} \L(\theta) = 0. 
\]
This shows that $\delta_{ij}$ remains constant for all $t$.
\end{proof}

As discussed in \Cref{lem:fiber}, each connected component of the fiber $\mu^{-1}(p)$ of a polynomial $p$ under the paramaterization map $\mu$ consists of the different rescalings of a fixed factorization of $p$. 
\Cref{prop:invariants} can be used to eliminate the scaling ambiguity, assuming that the parameters (or the invariants) are known at initialization. 

\begin{corollary}\label{cor:invariants_system}
Let $\mu$ be the parameterization map of an LCN with stride one.
Given filters $\theta = (w_1,\ldots,w_L)$, the set of parameter values $\theta'$ such that $\mu(\theta) = \mu(\theta')$ and $\delta_{i, i+1} = \delta'_{i,i+1}$ for all $i  = 1, \ldots, L-1$ is finite, where $\delta_{ij}$ and $\delta'_{ij}$ are the invariants associated with $\theta$ and $\theta'$, respectively, as defined in \Cref{prop:invariants}.
\end{corollary}

\begin{proof} Let $q_i = \pi(w_i)$ be the polynomials associated with the filters $w_i$, and let $p = q_L \cdots q_1 = \pi(\mu(\theta))$. 
As in Lemma~\ref{lem:fiber}, the factors $q_i$ are determined by $p$ up to a finite number of root permutations and also up to rescalings:
$
q_i' = \kappa_i q_i,$ where $\prod_{i=1}^L \kappa_i = 1.
$
It remains to show that there are only finitely many choices of the scale factors $\kappa_i$ such that 
$\delta_{i+1,i} = \|\kappa_{i+1} q_{i+1}\|^2 - \|\kappa_i q_{i}\|^2$ for $i=1,\ldots,L-1$. 
We can recover the possible $\kappa_i$ by solving a simple polynomial system. Setting $\beta_i = \|\kappa_i q_i\|^2 = \kappa_i^2 \| q_i \|^2$ and $\delta_i = \delta_{i+1,i}$, we have that
\begin{equation*}
\beta_{i+1} - \beta_i = \delta_i, \,\, (i=1,\ldots,L-1), \quad 
\prod_{i=1}^L \beta_i = H, 
\end{equation*}
where $H = \prod_{i=1}^L \|q_i\|^2$. 
In particular, all $\beta_i$ are determined by $\beta_1$:  $\beta_{i+1} = \beta_1 + \delta_1 + \cdots + \delta_i$ for all $i=1,\ldots,L-1$.
Hence, we obtain one polynomial equation of degree $L$ in $\beta_1$, namely
\begin{equation}\label{eq:1eq_1var}
\beta_1 \cdot (\beta_1 + \delta_1) \cdots (\beta_1 + \delta_1 + \cdots + \delta_{L-1}) - H = 0. 
\end{equation}
There are finitely many real solutions for $\beta_1$ in~\cref{eq:1eq_1var} such that $\beta_1, \ldots,\beta_1 + \delta_1 + \cdots +\delta_{L-1}$ are all positive. Each such solution  gives finitely many choices for the  $\kappa_i$. 
\end{proof}

The previous statement can be used to study the impact of the initialization (within the fiber) on the training dynamics. More precisely, we can compute the Riemannian metric that locally characterizes the gradient dynamics in function space. In general if $\L = \ell \circ \mu$ and $\theta(t)$ is an
integral curve for the negative gradient field of $\L$, then setting $\overline W(t) = \mu(\theta(t))$ 
we have 
\begin{equation}\label{eq:function_dynamics}
\begin{aligned}
\frac{d }{dt}\overline W(t) 
&= - K(\theta(t)) \,\cdot \, \nabla \ell (\overline W(t)),
\end{aligned}
\end{equation}
where $K(\theta) = \Jac_{\mu}(\theta)\Jac_{\mu}(\theta)^\top$ and $\Jac_{\mu}(\theta)$ is the Jacobian matrix of $\mu$ at $\theta$. 
The matrix $K(\theta)$ corresponds to the \emph{neural tangent kernel} \cite{NEURIPS2018_5a4be1fa} popular in the analysis of neural networks and can be viewed as a ``pre-conditioning'' matrix that modifies the gradient of $\ell$ in function space. If $K(\theta)$ could be written as a function of $\overline W = \mu(\theta)$, then we could view~\cref{eq:function_dynamics} as a Riemannian gradient flow for the metric tensor $K^{-1}$ (this assumes $K$ is not singular, which is generically the case; note that $K$ is always positive semi-definite). In general, however, the matrix $K(\theta)$ depends on the parameter $\theta$, and different parameters within the fiber of $\overline W$ will yield different matrices. In particular, the trajectories in function space are susceptible to the relative scalings of the filters at initialization. Using the dynamical invariants, however, we can bypass this problem by restricting ourselves to the parameters that are dynamically reachable, defining a local metric in function space associated with the initialization. 
We show this idea with a simple example.

\begin{example} \rm
Let us consider the factorization of a cubic polynomial into a linear and a quadratic factor. We write $C = A \cdot B$, where
\[
C = c_3 \x^3 + c_2 \x^2\y + c_1 \x\y^2 + c_0\y^3,\quad 
A = a_2 \x^2 + a_1 \x\y + a_0\y^2,\quad 
B = b_1 \x + b_0\y.
\]
The Jacobian matrix of the parameterization map $\mu:(a_2,a_1,a_0,b_1,b_0) \mapsto (c_3,c_2,c_1,c_0) = (a_2 b_1, a_2 b_0 + a_1 b_1, a_1 b_0 + a_0 b_1, a_0b_0)$ is
\begin{equation*}
\Jac_{\mu}(a_2,a_1,a_0,b_1,b_0) =
\begin{bmatrix}
b_1 & 0 & 0 & a_2 & 0\\
b_0 & b_1 & 0 & a_1 & a_2\\
0 & b_0 & b_1 & a_0 & a_1\\
0 & 0 & b_0 & 0 & a_0\\
\end{bmatrix} \in \RR^{4 \times 5}.
\end{equation*}
The matrix $K(\theta) = \Jac_{\mu} \Jac_{\mu}^\top \in \RR^{4 \times 4}$ can be written as $K(\theta)=K_B+K_A$, where 
\begin{equation*}
K_B = 
\begin{bmatrix}
b_{1}^{2} & b_{0} b_{1} & 0 & 0 \\
b_{0} b_{1} & b_{0}^{2} + b_{1}^{2} & b_{0} b_{1} & 0 \\
0 & b_{0} b_{1} & b_{0}^{2} + b_{1}^{2} & b_{0} b_{1} \\
0 & 0 & b_{0} b_{1} & b_{0}^{2}
\end{bmatrix},\;\; 
K_A =
\begin{bmatrix}
a_{2}^{2} & a_{1} a_{2} & a_{0} a_{2} & 0 \\
a_{1} a_{2} & a_{1}^{2} + a_{2}^{2} & a_{0} a_{1} + a_{1} a_{2} & a_{0} a_{2} \\
a_{0} a_{2} & a_{0} a_{1} + a_{1} a_{2} & a_{0}^{2} + a_{1}^{2} & a_{0} a_{1} \\
0 & a_{0} a_{2} & a_{0} a_{1} & a_{0}^{2}
\end{bmatrix}.
\end{equation*}
This matrix clearly depends on the coefficients of $A$ and $B$; however, we can recover it given only $C$ and the value of $\delta = \|A\|^2 - \|B\|^2$, which is fixed from the initialization. To resolve the scale ambiguity, we replace $A,B$ with $\gamma A, 1/\gamma B$ with $\gamma \in \RR\setminus \{0\}$ (here $A$ and $B$ can be fixed arbitrarily so that $A \cdot B = C$). We see that $K(\theta)$ becomes a one-dimensional (quadratic) family of matrices $K(\gamma) = (1/\gamma)^2 K_B + \gamma^2 K_A$. Using $\|\gamma A\|^2 - \|\nicefrac{1}{\gamma} \, B\|^2 = \delta$ and writing $\beta = \|\nicefrac{1}{\gamma} \, B\|^2$ and $H = \|A\|^2\|B\|^2$ we deduce that
\[
\beta^2 + \delta \beta - H = 0, \quad \beta > 0, \beta+\delta > 0, H > 0,
\]
which is the same as~\cref{eq:1eq_1var}.
This yields $\beta = \frac{-\delta + \sqrt{\delta^2 + 4H}}{2}$ and $\gamma^2 = \frac{2 \|B\|^2}{-\delta + \sqrt{\delta^2 + 4H}}$. 
\end{example}

\section{Optimization of the square loss}
\label{sec:squareloss}
In this section, we investigate the minimization of quadratic losses and the Euclidean distance degrees for LCNs. 
\subsection{The square loss}
\label{ssec:squareloss}

Given data $X\in\mathbb{R}^{d_0\times N}$ and $Y\in \mathbb{R}^{d_L\times N}$ (where $N$ is the number of data samples), and assuming that 
$XX^\top$ has full rank, 
we write the quadratic loss as:
\begin{equation}
\begin{aligned}
\ell(W) =\ell_{X,Y}(W) &= \|W X - Y\|^2 = 
\langle Y, Y \rangle - 2 \langle WX,Y \rangle  + \langle WX, WX \rangle \\
&= const. + 2 \langle  W (XX^\top), YX^\top(XX^\top)^{-1} \rangle + \langle WX, WX \rangle\\
&= const. + 2\langle W, U \rangle_{XX^\top}  + \langle W, W \rangle_{XX^\top} \\
&= const. + \|W - U\|^2_{XX^\top},
\end{aligned}
\label{eq:objective}
\end{equation}
where $U$ is the unconstrained optimizer of this quadratic problem, over the set of all matrices: 
\begin{equation*}
U = \operatorname{argmin} \limits_{V\in\mathbb{R}^{d_L\times d_0}}\|V X - Y\|^2 = YX^\top (XX^\top)^{-1}.  
\end{equation*}

As discussed in the previous sections, an LCN corresponds to a semi-algebraic set $\mathcal{M}$ of convolutional matrices $\overline{W}$. 
By~\cref{eq:objective}, the optimization of the square error $\ell(\overline W)$ over the function space $\mathcal{M}$ corresponds to a constrained linear regression problem, which is equivalent to solving 
$$
\min_{\overline W\in \mathcal{M}} \|U-\overline W\|_{XX^\top}^2. 
\label{eq:optimization_model}
$$
If $U\in\mathcal{M}$, then the solution is $U$. 
If $U\not\in \mathcal{M}$, we can optimize the objective over the boundary of the function space. We also may assume that the matrix $U$ is convolutional of the same format as the matrices in $\mathcal{M}$, for if not, we can consider the orthogonal projection (with respect to $\| \cdot \|_{XX^\top}$) to the appropriate linear space of convolutional matrices.

\begin{remark}
\label{rem:populationLoss}
The same expressions as above hold for population losses, by replacing $\frac{1}{N} XX^\top = \frac{1}{N}\sum_{i=1}^N x_i x_i^\top$ with $\mathbb E_{x \sim \mathcal D} [xx^\top]$ where $\mathcal D$ is the distribution over the input data. 
\end{remark}

If we assume that $U$ is a convolutional matrix with filter $u$, then we can also write the square loss in terms of the filter $\overline w$ of $\overline{W}$:
\begin{equation}
\label{eq:squareLossFilter}
\begin{aligned}
\ell( \overline w) =  \|\overline{W} - U\|_{XX^\top}^2 
&= \mathrm{tr} [(\overline W - U) XX^\top (\overline W - U)^\top]\\
&= \sum_{m=0}^{d_L-1} [(\overline W - U)_{m,:} XX^\top (\overline W - U)_{m,:}^\top],
\end{aligned}
\end{equation}
where we write $M_{m,:}$ for the $m$-th row of a given matrix $M$.
We denote by $k$ and $s$ the filter size and stride of the convolutional matrices in the function space $\mathcal{M}$.
Since the $m$-th row of $\overline{W}-U$ is obtained by shifting $\overline{w}-u$ by $sm$ positions to the right, we see that
\begin{equation}
\label{eq:filterLoss}
\begin{aligned}
\ell(\overline{w})= (\overline w - u)^\top \tau({XX^\top}) (\overline w - u),
\end{aligned}
\end{equation}
where $\tau: \RR^{d_0 \times d_0} \rightarrow \RR^{k \times k}$ is an operator that sums all shifts of rows and columns of a 
matrix: 
\begin{equation}
\label{eq:tauoperator}
\tau(M)_{ij} = \sum_{m=0}^{d_L-1} M_{i+sm,j+sm}. 
\end{equation}

\begin{remark}\label{rmk:circulant_tau}
If we consider circulant matrices instead of convolutional matrices, then all matrices appearing in~\cref{eq:squareLossFilter} would be square matrices of size $d_0 \times d_0$ and the trace would sum from $m=0$ to $d_0-1$.
The sum in~\cref{eq:tauoperator} would change accordingly and would be cyclic. 
\end{remark}

\begin{lemma}\label{lem:euclidean_dist}
Let $XX^\top$ be a multiple of the identity matrix. For any stride $s$, the square loss $\ell$ on the filters is equivalent to the Euclidean distance between $\overline w$ and $u$, \ie, 
\mbox{$
\ell(\overline w) \propto \|\overline w-u\|^2.
$}
\end{lemma} 

\begin{proof}
If $XX^\top$ is a multiple of the identity, 
then $\tau(XX^\top) = d_L XX^\top$ for convolutional matrices
resp. $\tau(XX^\top) = d_0 XX^\top$ for circulant matrices.
Now the claim follows from \cref{eq:filterLoss}.
\end{proof}

\begin{remark}
\label{rem:diagonalCirculant}
The lemma above holds both for convolutional and circulant matrices.
If we restrict ourselves to the circulant setting, then the statement holds more generally if $XX^\top$ is an arbitrary diagonal matrix as long as $s$ and $d_0$ are coprime. This follows from \Cref{rmk:circulant_tau}. Indeed,
if $s$ and $d_0$ are coprime, $\tau(XX^\top)_{ii}$ sums over all entries of the diagonal matrix $XX^\top$ and so  $\tau(XX^\top)$ is a multiple of the identity matrix.
\end{remark}

We can apply \Cref{lem:euclidean_dist} also for population losses (see \Cref{rem:populationLoss}) where the data distribution $\mathcal D$ is such that the covariance matrix $\mathbb E_{x \sim \mathcal D}[xx^\top]$ is a multiple of the identity (\eg, Gaussian data), and conclude that using the square loss on the LCN is equivalent to using a standard Euclidean distance at the level of filters. If we make use of circulant convolutions and restrict ourselves to a stride that is coprime to $d_0$, then this is also the case if the covariance matrix is simply diagonal.

Next we give a detailed example illustrating the critical points of the quadratic loss for an LCN and how these depend on the particular training data. In the next subsection we will discuss the number of critical points in more generality. 

\begin{example}\rm
\label{ex:tinyexample}
We return to \Cref{example:minimal} to consider minimizing the square loss over the set of functions $\mathcal{M} = \{ W = (A,B,C) \colon B^2 \geq 4AC \}$ represented by the two layer LCN with architecture specified by $\bm d = (3,2,1)$, $\bm k = (2,2)$, $\bm s = (1,1)$. 
Given data matrices $X,Y$, we are minimizing an objective $\ell(W)=\|WX-Y\|^2$ which is convex over the space of all matrices $W\in\mathbb{R}^{1\times 3}$. 
Hence, if the unconstrained minimizer is outside $\mathcal{M}$ and $XX^\top$ is full rank, then the set of minimizers over $\mathcal{M}$ will be at the boundary of $\mathcal{M}$. 
The situation is illustrated in \Cref{fig:zwei}. 
The boundary $\partial \mathcal{M} = \{W=(A,B,C) \colon B^2 = 4AC\}$ can be parametrized by $(A,C) \mapsto W = (A,\pm 2\sqrt{AC},C)$ with $(A,C)$ either both non-negative or both non-positive. 
In the following we consider the case $A,C\geq0$ (the case $A,C\leq 0$ is similar). 
The gradient of the objective function with respect to $A,C$ is 
\begin{align*}
\nabla_{A,C} \ell(A,C) =&
(-2YX^\top + 2WXX^\top) \nabla_{A,C}W^\top . 
\end{align*}
The first order criticality condition $\nabla_{A,C}\ell(A,C)=0$ means that 
\begin{align}
- YX^\top + W XX^\top  
= \lambda WJ, 
\label{eq:ex1storder}
\end{align}
where we introduce the matrix $J=\left[\begin{smallmatrix} 0&0&1\\
0&-1/2&0\\
1&0&0 \end{smallmatrix}\right]$, so that $W J = [C, \mp\sqrt{AC}, A]$ is a basis of the left null space of $\nabla_{A,C}W^\top$. 
We can solve \cref{eq:ex1storder} by writing 
\begin{equation}
W= YX^\top (XX^\top - \lambda J)^{-1}, 
\label{eq:tiny-critical}
\end{equation}
and plugging this into the defining equation of the boundary, $B^2=4AC$, which gives us an equation in $\lambda$. 
As discussed in~\cite[Section~4.4]{1027497}, computing a square distance from a point to a quadric surface results in a degree $6$ polynomial in $\lambda$. 
Since our boundary equation is homogeneous, we can clear the determinant in Cramer's rule for the matrix inverse and use simply the adjugate matrix, which results in a degree $4$ polynomial in $\lambda$: 
\begin{equation}
(YX^\top \operatorname{adj}(XX^\top-\lambda J)_{:,2})^2 - 4(YX^\top \operatorname{adj}(XX^\top-\lambda J)_{:,1})(YX^\top \operatorname{adj}(XX^\top-\lambda J)_{:,3}). 
\label{eq:ex-lambda}
\end{equation}
Solving this, we typically obtain $4$ values of $\lambda$, which when inserted in \cref{eq:tiny-critical} give us the critical points in the boundary of the function space. 

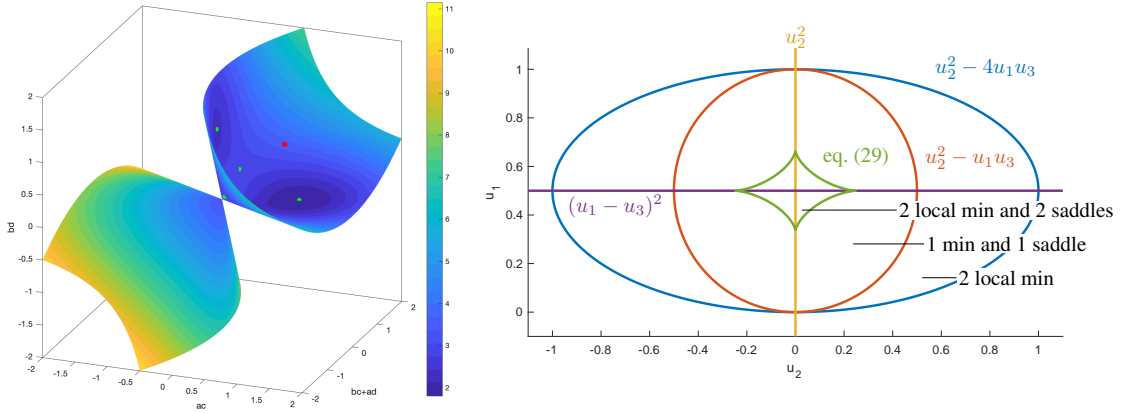
\begin{figure}
\centering
\begin{tikzpicture}
\node at (0,0) {\includegraphics[width=6cm]{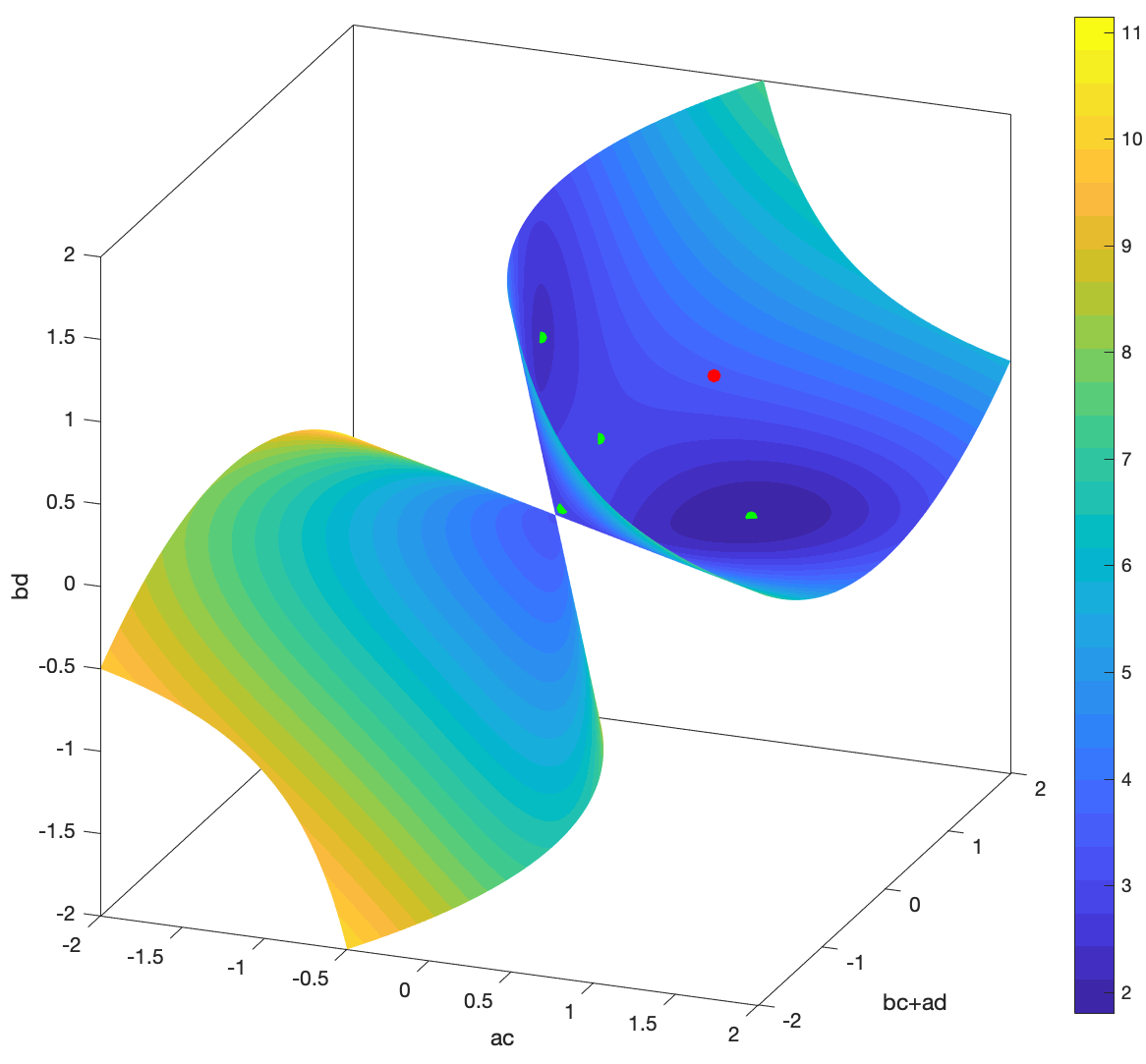}}; 
\node at (7.5,0) { 
\scalebox{.8}{ 
    \begin{tikzpicture}[x=1cm,y=1cm]
\definecolor{col1}{rgb}{0, 0.4470, 0.7410}
\definecolor{col2}{rgb}{0.8500, 0.3250, 0.0980}
\definecolor{col3}{rgb}{0.8361, 0.6246, 0.1125}
\definecolor{col4}{rgb}{0.4940, 0.1840, 0.5560}
\definecolor{col5}{rgb}{0.4660, 0.6740, 0.1880}
    \hypersetup{linkcolor=col5}
    
    \node at (0,0) {\includegraphics[clip=true, trim=0cm 0cm 0cm 1cm, width=10cm]{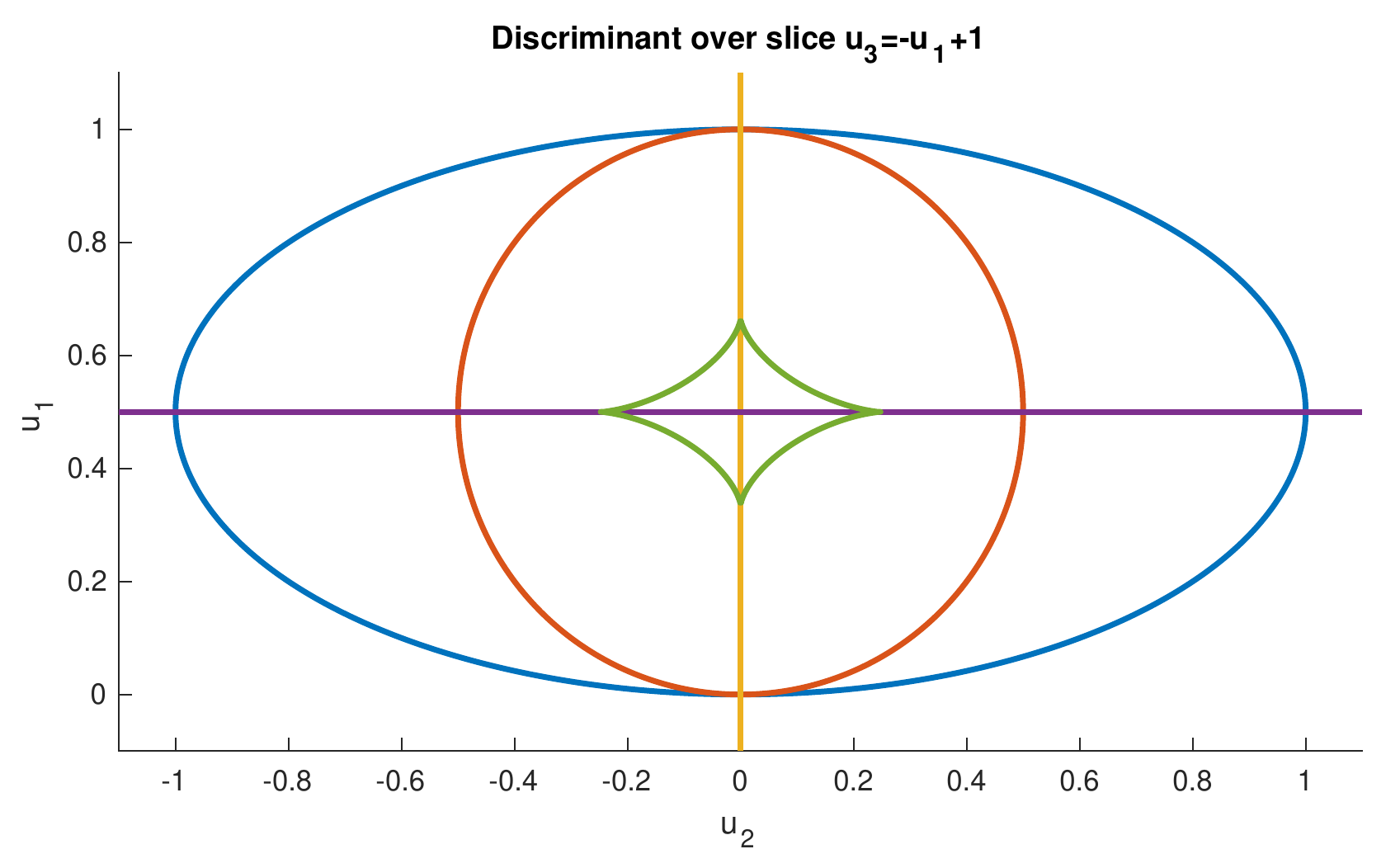}};
    \node at (3.5,2.5) {\textcolor{col1}{$u_2^2 - 4 u_1 u_3$ 
    }};
    \node at (3.2,1) {\textcolor{col2}{$u_2^2 - u_1 u_3$}};
    \node at (.3,3) {\textcolor{col3}{$u_2^2$}};
    \node at (-2.7,0.2) {\textcolor{col4}{$(u_1-u_3)^2$}};
    \node at (1.3,1) {\textcolor{col5}{\cref{eq:discriminant_long_tiny_example}}};
    
    \node[
    fill=white, fill opacity=1, text opacity=1, inner sep = 1] (N1) at (3.8,.125) {2 local min and 2 saddles};
    \node[
    fill=white, fill opacity=1, text opacity=1, inner sep = 1] (N2) at (3.8,-.4375) {1 min and 1 saddle};
    \node[
    fill=white, fill opacity=1, text opacity=1, inner sep = 1] (N3) at (3.8,-1) {2 local min}; 
    
    \draw (.4,.125) -- (N1.west);
    \draw (1.25,-.4375) -- (N2.west);
    \draw (2.4,-1) -- (N3.west);
    \end{tikzpicture}
}
};
\end{tikzpicture}

\caption{
Left: 
Boundary of the function space in $\mathbb{R}^{1\times 3}$ from \Cref{example:minimal}. 
The complement of the function space is inside the elliptic double cone with axis $(1,0,1)$. 
The boundary is colored by values of the square loss for random data $X,Y$ with $U=YX^\top(XX^\top)^{-1}$ shown as a red dot. 
For this data there are two local minima and two saddles shown as green dots. 
Right: Discrimination of datasets with $XX^\top=I$ in terms of $U=(u_1,u_2,u_3)=YX^\top$ by the types of critical points of the square loss as discussed in \Cref{ex:tinyexample}. 
    Shown is the slice $u_3=-u_1+1$. 
    For points with $u_2^2-4u_1u_3<0$ outside the function space, 
    we have cases with two local minima and two saddles, one minimum and one saddle, and two local minima. 
}
\label{fig:zwei}
\label{fig:discriminant_tiny_example}
\end{figure}

Now we take a closer look at the special case with $XX^\top=I$. In this case, the unconstrained optimal solution is $U = (u_1, u_2, u_3)= YX^\top$ and \cref{eq:ex-lambda} simplifies to 
\begin{multline}
(u_2^2-u_1u_3)\lambda^4+(-u_1^2-4u_1u_3-u_3^2)\lambda^3+(-4u_1^2-2u_2^2-5u_1u_3-4u_3^2)\lambda^2\\
+(-4u_1^2-4u_1u_3-4u_3^2)\lambda + (u_2^2-4u_1u_3).
\label{eq:ex-lambda1}
\end{multline}
We can classify the points $(u_1,u_2,u_3)$ for which the problem has different numbers of solutions. 
When $u_2^2-u_1u_3\neq 0$, \cref{eq:ex-lambda1} is a degree 4 polynomial in $\lambda$ which generally has 4 or 2 real solutions (the global minimum corresponds to a real root). 
The cases are separated by the discriminant of \cref{eq:ex-lambda1} with factors 
$u_2^2$, $(u_1-u_3)^2$, $(u_1+u_3)^2$, 
and
\begin{multline}
32u_1^6 + 435u_1^4u_2^2 + 384u_1^2u_2^4+256u_2^6 - 240u_1^5u_3 - 960u_1^3u_2^2u_3 - 960u_1u_2^4u_3 + 696u_1^4u_3^2\\
+ 1098u_1^2u_2^2u_3^2 + 384u_2^4u_3^2 - 980u_1^3u_3^3 - 960u_1u_2^2u_3^3 + 696u_1^2u_3^4 + 435u_2^2u_3^4 - 240u_1u_3^5 + 32u_3^6. 
\label{eq:discriminant_long_tiny_example}
\end{multline}
We visualize the situation in Figure~\ref{fig:discriminant_tiny_example}. 
The figure shows the slice $u_3 = -u_1+1$. 
Since the discriminant is homogeneous, each point in the slice is representative of a line through the origin in $\mathbb{R}^3$. The zeros of the factor $(u_1+u_3)^2$ are not visible over this slice. 
We see that the discriminant separates $12$ generic cases outside of the function space (\ie, inside the blue ellipse, where $u_2^2 - 4u_1u_3<0$). 
By evaluating individual instances in each of these cases, 
we find that the points inside the green curvy diamond (caustic) have two local minima and two saddles, 
the points between the green curvy diamond and the red circle have one minimum and one saddle, 
and the points between the red circle and the blue ellipse have two local minima (whereby we observed that only one of them satisfies $A,C\geq 0$).  
\end{example}

\subsection{Euclidean distance degrees} 
\label{sec:EDdegree} 

We now discuss the number of critical points  of the optimization problem in function space using the square loss~\cref{eq:objective}.
This problem is given by 
\begin{align*}
    \min_{W \in \mathcal{M}} \| U-W  \|^2_{XX^\top},
\end{align*}
where $\mathcal{M}$ is the function space of an LCN and $U$ and $X$ are fixed data matrices. 
We focus on one-dimensional convolutions with stride one. 
We first assume that $XX^\top$ is the identity matrix (\ie, we are interested in minimizing the standard Euclidean distance of the data $U$ to the function space $\mathcal{M}$) and comment on other choices of $X$ later.

By \Cref{thm:criticalPointsAreOnMultipleRootLoci}, each critical point in the parameter space of the LCN corresponds in function space to a critical point of the Euclidean distance from the data $U$ to a multiple root locus $\Delta_\lambda$. 
For almost every choice of $U$ the number of \emph{complex} critical points of minimizing the Euclidean distance from $U$ to $\Delta_\lambda$ is the same. 
This number is known as the \emph{Euclidean distance degree} (\emph{ED degree} for short)  of $\Delta_\lambda$~\cite{draisma2016euclidean}.
Hence, if we assume that the data $U$ is generic, \Cref{cor:rrmp} shows that 
the number of $\mu(\theta)$, where $\theta$ is a critical point in the parameter space of the given LCN, is upper bounded by the sum of the ED degrees of all multiple root loci  $\Delta_\lambda \subset \RR[\x,\y]_{k-1}$, where $k$ is the filter size in the function space $\mathcal{M}$ and $\lambda=\lambda_{\rho|\gamma}$ comes from an \rrmp\, $(\rho \mid \gamma)$ that is compatible with the LCN architecture.
\begin{example}
\label{ex:upperEDdegreeBound}
For generic data $U$, \Cref{ex:partitionVsArchitecture} shows the following:
The number of $\mu(\theta)$, where $\theta$ is a critical point of an LCN with architecture $\bm k = (3,2,2)$ resp.\ $\bm {k'} = (4,2)$ is at most the sum of the ED degrees of
$\Delta_{(1,1,1,1)}, \Delta_{(1,1,2)},\Delta_{(2,2)}, \Delta_{(1,3)}$ resp.\ $\Delta_{(1,1,1,1)}, \Delta_{(1,1,2)}$. 
\end{example}

The state-of-the-art on ED degrees of multiple root loci is presented in \cite{lee2016duality}.
For instance, the ED degree of the discriminant hypersurface of degree-$(k-1)$ polynomials is $3(k-1)-2$~\cite[Example 5.12]{draisma2016euclidean}. 
More generally, the ED degree of the multiple root locus $\Delta_{(\alpha, 1, \ldots, 1)}$ (with $\alpha >1$ and  $k-1-\alpha$ ones) is $(2\alpha-1)(k-1)-2(\alpha-1)^2$~\cite[Theorem 5.1]{lee2016duality}.
Many more ED degrees are computed in \cite[Table 1]{lee2016duality}, but a closed formula for general partitions is not known. 

\begin{example}\rm
\label{ex:EDdegreesGeneric}
$(k=5)$
The discriminant of quartic polynomials has ED degree $10$.
Its iterated singular loci are the following 
multiple root loci: $\Delta_{(2,2)}$ with ED degree $13$, $\Delta_{(1,3)}$ with ED degree $12$, and $\Delta_{(4)}$ with ED degree $10$. 
Hence, the upper bounds on the number of $\mu(\theta)$ for critical points $\theta$ given in \Cref{ex:upperEDdegreeBound} are $1+10+13+12=36$ resp.\ $1+10=11$. 
\end{example}

\begin{remark}\label{rmk:bombieri}
The ED degrees discussed in this section are often referred to as \emph{generic ED degrees} in the algebraic literature.
The \emph{special ED degree} refers to the following slightly different norm to be optimized: for two polynomials $P = \sum_{i=0}^{k-1} w_i \x^i$ and $Q = \sum_{i=0}^{k-1} u_i \x^i$, 
\begin{equation}\label{eq:bombieri}
\langle P, Q \rangle_{B}:=\sum_{i=0}^{k-1} \frac{i!(k-i-1)!}{(k-1)!} u_i w_i. 
\end{equation}
The letter ``$B$'' stands for \emph{Bombieri norm}. For example, if $P$ and $Q$ are quadratic polynomials, then $\langle P,Q \rangle_B = u_2 w_2 + \frac{1}{2}u_1 w_1 + u_0 w_0$. 
The Euclidean inner product that arises for LCNs if $XX^\top$ is the identity is instead simply $\langle P,Q \rangle_2 = u_2 w_2 + u_1 w_1 + u_0 w_0$. This slight difference has an impact on the ED degrees. Indeed, the ED degree of the discriminant hypersurface $\{(w_0,w_1,w_2) \colon w_1^2 = w_0 w_2\}$ is $2$ for the Bombieri norm, and $4$ for the standard Euclidean norm (see Example~\ref{ex:tinyexample} for a computation of the four critical points).
More generally, if we consider the multiple root loci for polynomials of degree $k-1$, then the special ED degree of the discriminant hypersurface is $k-1$~\cite[Corollary 8.7]{draisma2016euclidean} (whereas the generic ED degree is $3(k-1) - 2$, as noted above). The special ED degree of the multiple root locus  $\Delta_{(\alpha, 1, \ldots, 1)}$ (with $\alpha >1$ and $k-1-\alpha$ ones) is $k-1$, independently of $\alpha$~\cite[Theorem 5.1]{lee2016duality}.
\end{remark}

\begin{example} \label{ex:specialEDdegree} \rm $(k=5)$
The generic ED degrees of the multiple root loci of quartic polynomials are listed in Example~\ref{ex:EDdegreesGeneric}.
The special ED degrees are smaller:
for $\Delta_{(1,1,2)}$ it is $4$, 
for $\Delta_{(2,2)}$ it is $7$,
for $\Delta_{(1,3)}$ it is $4$,
and for $\Delta_{(4)}$ it is $4$. 
\end{example}

\begin{remark}
The above discussion shows that the Bombieri norm can be thought of as a ``special'' norm that leads to lower ED degrees and thus makes the optimization landscape simpler. Perhaps counterintuitively, the standard Euclidean distance is instead \emph{generic}, in the sense that almost all pairings $\langle\cdot,\cdot\rangle$ will lead to the same ED degrees for multiple-root loci.
In other words, for  \emph{almost all} data matrices the ED degrees of the multiple root loci would be the same as assuming $XX^\top$ is the identity. 
\end{remark}

A natural question in this setting is if there exists a ``special'' data distribution that induces the Bombieri norm~\cref{eq:bombieri} in polynomial space. For circulant convolutions this is not possible. Indeed, for any matrix $M$, the matrix $\tau(M)$ when $d_0 = d_L$ will always have constant diagonal entries by the same reasoning as in Remark~\ref{rem:diagonalCirculant}. 
In the convolutional (Toeplitz) case, however, such special data distributions can exist, as the following simple  example demonstrates.

\begin{example}\rm
We consider a single convolutional layer with $k=3$, $d_0 = 4$ and $d_1 = 2$. If the data covariance matrix is such that
\[
XX^\top = \left[\begin{smallmatrix}
3/4 &   &   &  \\
  & 1/4 &   &  \\
  &   & 1/4 &  \\
  &   &   & 3/4\\
\end{smallmatrix}\right],
\]
then the induced metric in the space of quadratic polynomials is indeed the Bombieri norm $\langle P,Q \rangle_B = u_2 w_2 + \frac{1}{2} u_1 w_1 + u_0 w_0$. This example suggests that non-standard normalizations of the data might lead to more favorable optimization landscapes for convolutional networks.
\end{example}

\begin{example}\rm
We return to Example~\ref{ex:tinyexample}, an LCN with $\mathcal{M} = \{W=(A,B,C)\colon B^2\geq4AC\}$, and investigate the structure of critical points of the square loss over the boundary $\partial \mathcal{M} = \{W=(A,B,C)\colon B^2=4AC\}$ when $XX^\top$ is compatible with the Bombieri norm, 
\begin{equation}
XX^\top = \left[\begin{smallmatrix}
1&    &  \\
 & 1/2&  \\
 &    & 1
\end{smallmatrix}\right]. 
\label{eq:tiny-example-bombieri-norm}
\end{equation}
A geometric interpretation of this setting is that the balls $\{W \colon \|W-U\|_{XX^\top}\leq c\}$ are ellipsoids with the same $1\times 2\times 1$ aspect ratio as the boundary of the function space. 
As before, we write $U=(u_1,u_2,u_3) = YX^\top (XX^\top)^{-1}$ for the unconstrained least squares solution. 
With $XX^\top$ taking the form \cref{eq:tiny-example-bombieri-norm}, equation \cref{eq:ex-lambda} for the case $A,C\geq0$ becomes 
\begin{multline}
(-1) (\lambda + 1) (\lambda + 1) ((u_1 u_3 - 4 u_2^2)\lambda^2  + (u_1^2 + 8 u_2^2 +  u_3^2)\lambda  + (u_1 u_3 - 4 u_2^2)). 
\label{eq:ex-lambda-Bombieri}
\end{multline}
This quartic polynomial in $\lambda$ has a repeated root $\lambda = -1$ for any values of $U$ and hence its discriminant vanishes identically. 
The distinction in terms of $U$ is determined by the last factor alone. 
We discuss the zeros of the individual factors in turn. 
\begin{itemize}[leftmargin=*]
\item 
Consider the factor $(\lambda+1)$ of \cref{eq:ex-lambda-Bombieri}. 
The root $\lambda=-1$ plugged into \cref{eq:ex1storder}
gives equation $W(XX^\top+J)=YX^\top$ for $W$. 
Since the matrix $XX^\top+J = \left[\begin{smallmatrix}
1 &0 &1\\
0&0&0 \\
1&0&1
\end{smallmatrix}\right]$ is not of full rank, for generic $YX^\top$ there is no solution and this case can be ignored. 
Solutions only exist in the non-generic case that $YX^\top = c \begin{bmatrix} 1& 0& 1\end{bmatrix}$, namely $\{W = (A, B, C) \colon A+C = c\}$, which is a hyperplane. 
The intersection of this hyperplane with the boundary of the function space gives us the set of critical points $\{(A,B,C)\colon A+C=c, B^2=4AC\}$. For the assumed case $A,C\geq0$, this is empty over the reals for $c<0$, a single point for $c=0$, and an ellipse for $c>0$. All critical points have the same Bombieri distance $c$ from $U=\begin{bmatrix} c& 0& c\end{bmatrix}$.  
\item 
Consider now the last factor of \cref{eq:ex-lambda-Bombieri}, which is quadratic in $\lambda$. 
Its discriminant has factors $(u_1 + u_3)^2$ and $u_1^2 - 2 u_1 u_3 + 16 u_2^2 + u_3^2$. 
The former has zero set $\{u_1=-u_3\}$, 
which also appeared in the case $XX^\top=I$ and is inside the function space. 
The latter has real zero set $\{u_2=0, u_1=u_3\}$, which is the center line of our boundary cone $\partial \mathcal{M}$ and can be interpreted as a degeneration of the caustic that appeared in the case $XX^\top=I$. It also corresponds to the non-generic case discussed above. 
\end{itemize}
Hence, when $XX^\top$ takes the form \cref{eq:tiny-example-bombieri-norm} compatible with the Bombieri norm, there is a single generic case for the datasets. 
We find that generic datasets $U=(u_1,u_2,u_3)$ outside the function space have two critical points over the boundary of the function space: 1 global minimum and 1 saddle. 
This is in contrast with the more complicated case distinction in Example~\ref{ex:tinyexample} and Figure~\ref{fig:discriminant_tiny_example} when $XX^\top=I$, where some generic datasets led to 4 critical points. 
On the other hand, for the Bombieri norm there are special datasets, with $\{u_2=0, u_1=u_3 >0\}$, with infinitely many critical points, namely entire ellipses of global minima. 
\end{example}

\section{Examples and numerical evaluation} 
\label{sec:experiments}

We present computational experiments on LCNs with 1D convolutions with stride one. 
In \cref{prop:critical-points} we saw that critical parameters are global minimizers of the objective or are contained in the discriminant hypersurface. In fact, in \Cref{thm:criticalPointsAreOnMultipleRootLoci} we observed that any critical parameter corresponds to a critical point over the locus with the same root structure. 
We discuss the relevant root loci and evaluate their appearance in  gradient optimization using different LCN architectures. 
As observed in the previous sections, the order of the filters does not affect the function space nor the gradient dynamics. For this reason, we consider architectures with filter widths $k_1 \ge \cdots \ge k_L$. 

We perform our experiments with the five non-filling architectures from \Cref{ex:boundaries}. 
The left hand side of \Cref{fig:zwei} visualizes the function space $\mathcal M_{(2,2)}$.
The function spaces representing cubic and quartic polynomials are depicted in \Cref{fig:cubic_quartic_discriminants}:
On the left, $\mathcal M_{(2,2,2)}$ is the ``inside'' of the ``cuspidal wedge.''
On the right, 
the complement of $\mathcal M_{(3,2,2)} = \mathcal M_{(4,2)}$ are the points above the blue ``bowl.'' 
These points correspond to one of the two convex cones described in \cref{prop:twoConvexCones}. 
The other cone is not visible since the figure fixes the highest degree coefficient to $1$.
Moreover, $\mathcal M_{(2,2,2,2)}$ 
is the set inside the triangular region. 

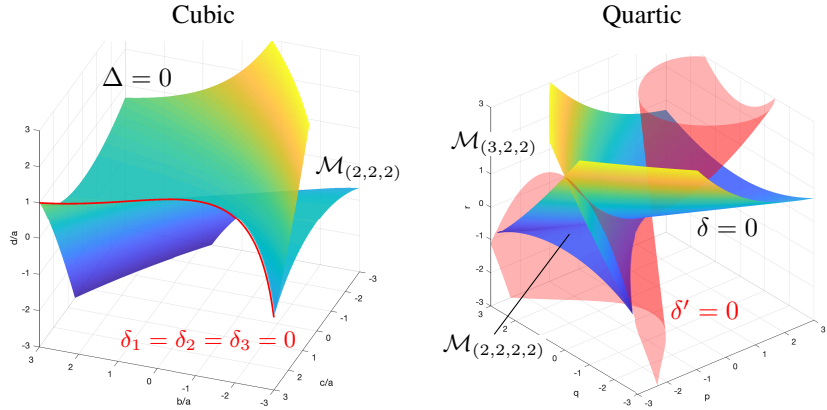
\begin{figure}[htb]
    \centering
\begin{tabular}{cc}    
Cubic & Quartic\\
    \begin{tikzpicture}
    \node at (0,0) {\includegraphics[width=5cm,clip=true,trim=0cm 0cm 0cm 1cm]{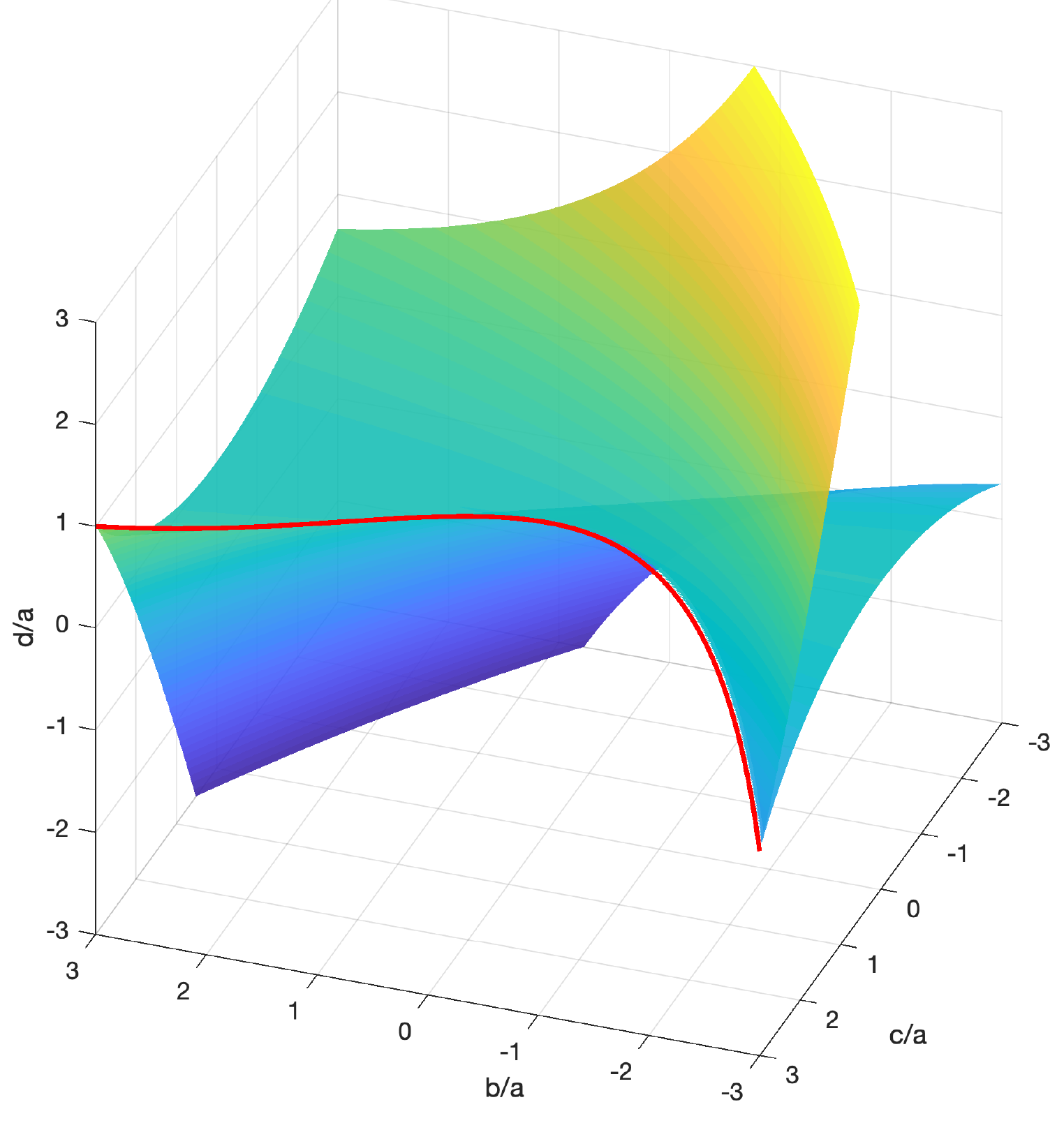}}; 
    \node at (-.75,2) {\textcolor{black}{$\Delta=0$}};
    \node at (0.2,-1.5) {\textcolor{red}{\small $\delta_1=\delta_2=\delta_3=0$}};
    \node at (2.2,0.78) {{\small $\mathcal{M}_{(2,2,2)}$}};
    \end{tikzpicture}
&
    \begin{tikzpicture}
    \node at (0,0) {\includegraphics[width=5cm,clip=true,trim=0cm 0cm 0cm 1cm]{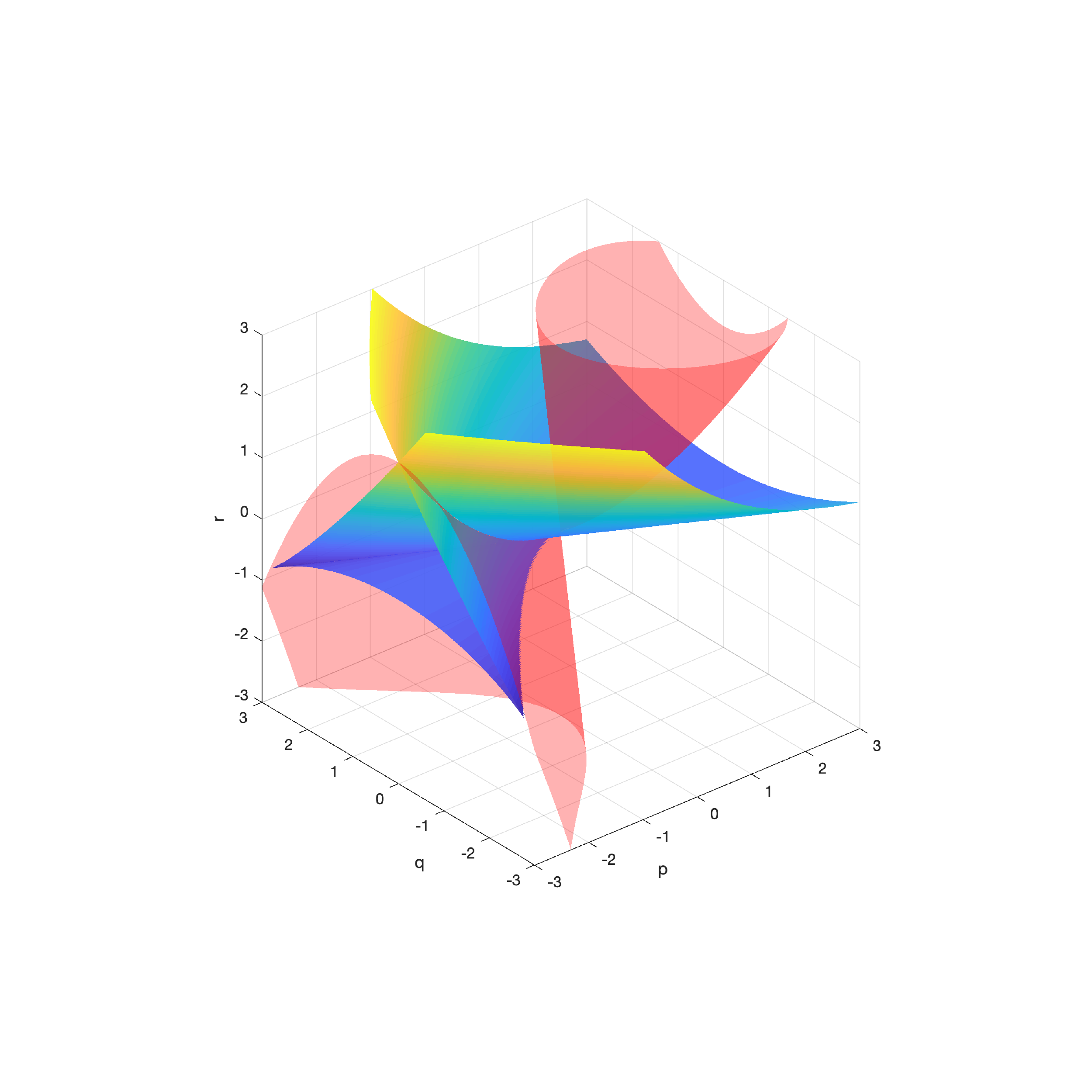}}; 
    \node at (1.1,0.1) {\textcolor{black}{$\delta=0$}};
    \node at (.8,-1) {\textcolor{red}{$\delta'=0$}};
    \node[fill=white, fill opacity=1, text opacity=1, inner sep = 0] 
    at (-2,1.2) {{\small $\mathcal{M}_{(3,2,2)}$}};    
    \node[fill=white, fill opacity=1, text opacity=1, inner sep = 0] 
    (M2222) at (-2,-1.5) {{\small $\mathcal{M}_{(2,2,2,2)}$}};    
    \draw (-1,0) -- (M2222.north);
    \end{tikzpicture}
\end{tabular}    
    \caption{
    Left: The surface $\Delta=0$ of coefficients $(a,b,c,d)$, $a\neq0$, for which a cubic polynomial $a\x^3+b \x^2\y+c\x\y^2+d\y^3$ has a double root. The coefficients with triple roots are shown in red. 
    The coefficients with $\Delta>0$ 
    (the interior of $\mathcal{M}_{(2,2,2)}$)
    are on the side with smaller values of $c$. 
    Right: The surfaces $\delta=0$ and $\delta'=0$ (as in \Cref{ex:rrmps}) of coefficients $(p,q,r)$, that separate the real root multiplicities of a depressed quartic polynomial $\x^4+p\x^2\y^2+q\x\y^3+r\y^4$. 
    The positive values of $\delta$ and $\delta'$ are on the sides that do not contain the labels. 
    The complement of $\mathcal{M}_{(3,2,2)}$ is above the blue bowl. 
    The interior of $\mathcal{M}_{(2,2,2,2)}$ is within the triangular blue region. 
    }
    \label{fig:cubic_quartic_discriminants}
\end{figure}

\subsection{Critical points reached by gradient descent}
For each of the architectures listed above, we generated $10,000$ random data sets $X,Y$ each consisting of $10$ i.i.d.\ input-output pairs from a standard normal distribution, with inputs in $\mathbb{R}^{d_0}$ and outputs in $\mathbb{R}$. Since $d_L=1$, we have from \Cref{prop:nonzerodiagsExtension} that $d_0  = k$. 
For each dataset and architecture, we initialized the parameters at random from a standard normal distribution and ran gradient descent optimization of the square loss $\L$ 
with fixed step size $0.01$ and stopping criterion $\|\nabla \L\|^2 \leq 10^{-14}$. 
Instances which did not reach the stopping criterion within a maximum of $15,000$ iterations were discarded. 
Since the final coefficients are only approximately critical points, they rarely correspond to polynomials with exactly repeated roots. To assess the \rrmp's of solutions, we compute the roots of all factor filters and consider a root $r$ as real if $|r-\operatorname{real}(r)| \leq \text{\texttt{tol}} \cdot |r|$, and consider two roots $r,r'$ as same if $|r-r'| \leq \text{\texttt{tol}}\cdot \max\{|r|,|r'|\}$ with $\text{\texttt{tol}} = 10^{-4}$. We use a similar procedure for the initializations and for the target. 
This classification scheme is relatively robust, but is not exact. 
The results are shown in Table~\ref{table:solution_types}. 

\begin{itemize}[leftmargin=*] 
\item 
For the architecture $\bm k =(2,2)$, about 66\% of the datasets correspond to targets with \rrmp\, {$11|0$} inside the function space, and gradient descent finds a solution with nearly zero training error. 
About 34\% of the datasets correspond to targets with \rrmp\, {$0|1$} outside the function space and gradient descent finds solutions with \rrmp\, {$2|0$} at its boundary. 
\item 
For $\bm k=(2,2,2)$, 
about 25\% of the datasets have \rrmp\, {$111|0$} inside the function space, and for most of these, gradient descent finds a solution of the same \rrmp\, with nearly zero loss. 
About 75\% of the datasets have \rrmp\, {$1|1$} outside the function space. Of these, about 70\% led to a solution with \rrmp\, {$12|0$}, \ie, a smooth point of the discriminant hypersurface, and about 30\% led to a solution with \rrmp\, {$3|0$}, \ie, a point on the red curve in Figure~\ref{fig:cubic_quartic_discriminants}. 

\item 
For the architectures $\bm k = (4,2)$ or $(3,2,2)$, which have the same function space, the different parametrizations lead to different optimization behaviors. 
First, the initialization procedure leads to different distributions in function space, 
with $\bm k = (4,2)$ being more frequently initialized with \rrmp\, {$11|1$} 
than $\bm k = (3,2,2)$. 
For targets with \rrmp\, {$0|11$}, outside the function space, 
the solutions often correspond to zeros of the discriminant and are critical points in parameter space but they do not lie on the Euclidean boundary of $\mathcal{M}_{\bm k}$.
Hence, they are spurious critical points. 
The effect is more pronounced in the deeper architecture, which has a more intricate parametrization. The optimization landscapes are illustrated in Figure~\ref{fig:landscapes}. 
For targets with \rrmp\, {$11|1$}, inside the function space, a good fraction of the optimization instances with $\bm k = (4,2)$ converged to sub-optimal solutions with \rrmp\, {$112|0$} or {$2|1$}, and for $\bm k = (3,2,2)$ an even larger fraction converged to sub-optimal solutions with \rrmp\, {$112|0$}, {$22|0$}, {$13|0$}, or {$2|1$}. 
This is in line with our theoretical discussion, by which $\bm k = (3,2,2)$, but not $\bm k = (4,2)$, can have critical points in parameter space with \rrmp's {$22|0$} and {$13|0$} for generic data (see \Cref{cor:rrmp} and \Cref{ex:partitionVsArchitecture}).

\item For $\bm k = (2,2,2,2)$, in contrast to the previous architectures, a fraction of the targets outside the function space (\ie, with \rrmp\, {$11|1$} or {$0|11$}) led to a solution with \rrmp\, {$4|0$}, \ie, the codimension 3 point at the origin in Figure~\ref{fig:cubic_quartic_discriminants}, which was predicted by \Cref{cor:rrmp}. 
\end{itemize}

\begin{table}
\caption{Rrmp's of solutions returned by gradient descent optimization of the square loss 
in parameter space for random datasets and initializations. 
The table omits a small fraction of targets and initializations numerically classified in \rrmp's of positive codimension. }
\label{table:solution_types}
\centering 
\small 
\begin{tabular}{ll}
$\mathcal{M}_{(2,2)}$&
\begin{tabular}[t]{l S[table-format=2.2] | l S[table-format=3.3] l}
       &      & \multicolumn{3}{c}{initialization {$11|0$} \;\; 100\% } \\
target & \%   & solution & \%& mean loss\\  
\hline  
{$11|0$\phantom{11}} & 65.8  & {$11|0$} & 100 & 1.15e-15\\
\hline 
{$0|1$}  & 34.2 &  {$2|0$} & 100 & 0.185  \\
\bottomrule 
\end{tabular}\\ 

$\mathcal{M}_{(2,2,2)}$& 
\begin{tabular}[t]{l S[table-format=2.2] | l S[table-format=3.3] l}
             &&\multicolumn{3}{c}{initialization {$111|0$} \;\; 100\% }\\
target  & \% & solution & \%& mean loss\\  
\hline  
{$111|0$\phantom{1}}& 24.9 &  {$111|0$} & 99.7 & 2.32e-15\\
                        && {$12|0$} & 0.051 & 0.945\\
                        && {$3|0$} & 0.205 & 1.19\\
\hline 
{$1|1$}  & 75  &  {$12|0$} & 69.8 & 0.224  \\
             &&  {$3|0$}  & 30.2 & 0.709 \\
\bottomrule 
\end{tabular}\\ 

$\mathcal{M}_{(4,2)}$& 
\begin{tabular}[t]{l S[table-format=2.2] | l S[table-format=3.3] l|l S[table-format=3.3] l}
    &&\multicolumn{3}{c|}{initialization {$1111|0$}\;\; 24.6\%} & \multicolumn{3}{c}{initialization {$11|1$}\;\; 75.3\%}\\
target & \%   & solution & \%& mean loss & solution & \%& mean loss\\  
\hline 
{$1111|0$}& 5.28 & {$1111|0$} & 100   & 3.04e-15 & {$1111|0$} & 100   & 3.1e-15 \\
\hline 
{$11|1$}  & 72.6 &  {$112|0$} & 15.5  & 0.228    & {$112|0$} & 12.4  & 0.193 \\
              &&  {$11|1$}  & 83.2  & 1.94e-15 & {$11|1$}  & 86.7  & 1.84e-15 \\
              &&  {$2|1$}   & 1.36 & 0.54      & {$2|1$}   & 0.886 & 0.397 \\
\hline 
{$0|11$}   & 22.1 &  {$112|0$} & 7.85 & 0.347     & {$112|0$} & 4.81  & 0.35 \\
              &&  {$2|1$}   & 92.2 & 0.231     & {$2|1$}   & 95.2  & 0.229 \\
\bottomrule 
\end{tabular}\\ 

$\mathcal{M}_{(3,2,2)}$& 
\begin{tabular}[t]{l S[table-format=2.2] | l S[table-format=3.3] l| l S[table-format=3.3] l}
                && \multicolumn{3}{c|}{initialization {$1111|0$}\;\; 65.1\%}  & \multicolumn{3}{c}{initialization {$11|1$}\;\; 34.9\%} \\
target & \%     & solution  & \%& mean loss & solution & \%& mean loss\\  
\hline 
{$1111|0$}& 4.82   &  {$1111|0$} & 99.6   & 4.68e-15 & {$1111|0$} & 100 & 4e-15\\
                &&  {$13|0$}   & 0.429   & 0.71 &&&\\
\hline 
{$11|1$}& 72.9    &    {$112|0$} & 27.1 & 0.221     & {$112|0$}  & 21.8 & 0.207\\
                &&   {$22|0$} & 1.28  & 0.992     & {$22|0$}   & 0.663& 0.879\\
                &&   {$13|0$} & 25.8  & 0.798     & {$13|0$}   & 15.6 & 0.71\\
                &&   {$11|1$} & 45.5  & 1.78e-15  & {$11|1$}   & 61.7 & 1.7e-15\\
                &&   {$2|1$}  & 0.381 & 0.446     & {$2|1$}    & 0.306& 0.418 \\
\hline 
{$0|11$}& 22.3    &    {$112|0$} & 11.2 & 0.374      & {$112|0$} & 9.09 & 0.365 \\
               &&   {$22|0$}  & 25.5 & 0.855      & {$22|0$}  & 17.5 & 0.882 \\
               &&   {$13|0$}  & 7.1  & 0.895      & {$13|0$}  & 4.21 & 0.937 \\
               &&   {$2|1$}   & 56.2 & 0.224      & {$2|1$}   & 69.2 & 0.208 \\
\bottomrule 
\end{tabular}\\ 

$\mathcal{M}_{(2,2,2,2)}$& 
\begin{tabular}[t]{l S[table-format=2.2] | l S[table-format=3.3] l}
             && \multicolumn{3}{c}{initialization {$1111|0$}\;\; 99.8\%}\\
target & \%   & solution & \%& mean loss \\  
\hline   
{$1111|0$}& 4.79 & {$1111|0$} & 99.7   & 5.41e-15 \\
              && {$13|0$}   & 0.274  & 0.297   \\
\hline 
{$11|1$}& 72.7   &  {$112|0$} & 33.3  & 0.259 \\
              &&  {$22|0$}  & 2.24  & 0.908 \\
              &&  {$13|0$}  & 52.1  & 0.774 \\
              &&  {$4|0$}   & 12.4  & 1.53 \\
\hline 
{$0|11$}& 22.5    &  {$112|0$} & 13.2 & 0.411 \\ 
              &&  {$22|0$}  & 40.7 & 0.779 \\ 
              &&  {$13|0$}  & 16.1 & 1.02 \\ 
              &&  {$4|0$}   & 30   & 1.48 \\ 
\bottomrule  
\end{tabular}
\end{tabular}
\end{table}

\begin{figure}
    \centering
\setlength\tabcolsep{0pt} 
\begin{tabular}{ccccc}
$\mathcal{M}_{(2,2)}$&$\mathcal{M}_{(2,2,2)}$&$\mathcal{M}_{(4,2)}$&
$\mathcal{M}_{(3,2,2)}$&$\mathcal{M}_{(2,2,2,2)}$\\
\includegraphics[clip=true,trim=0cm 3.5cm 1cm 3.6cm,width=.2\textwidth]{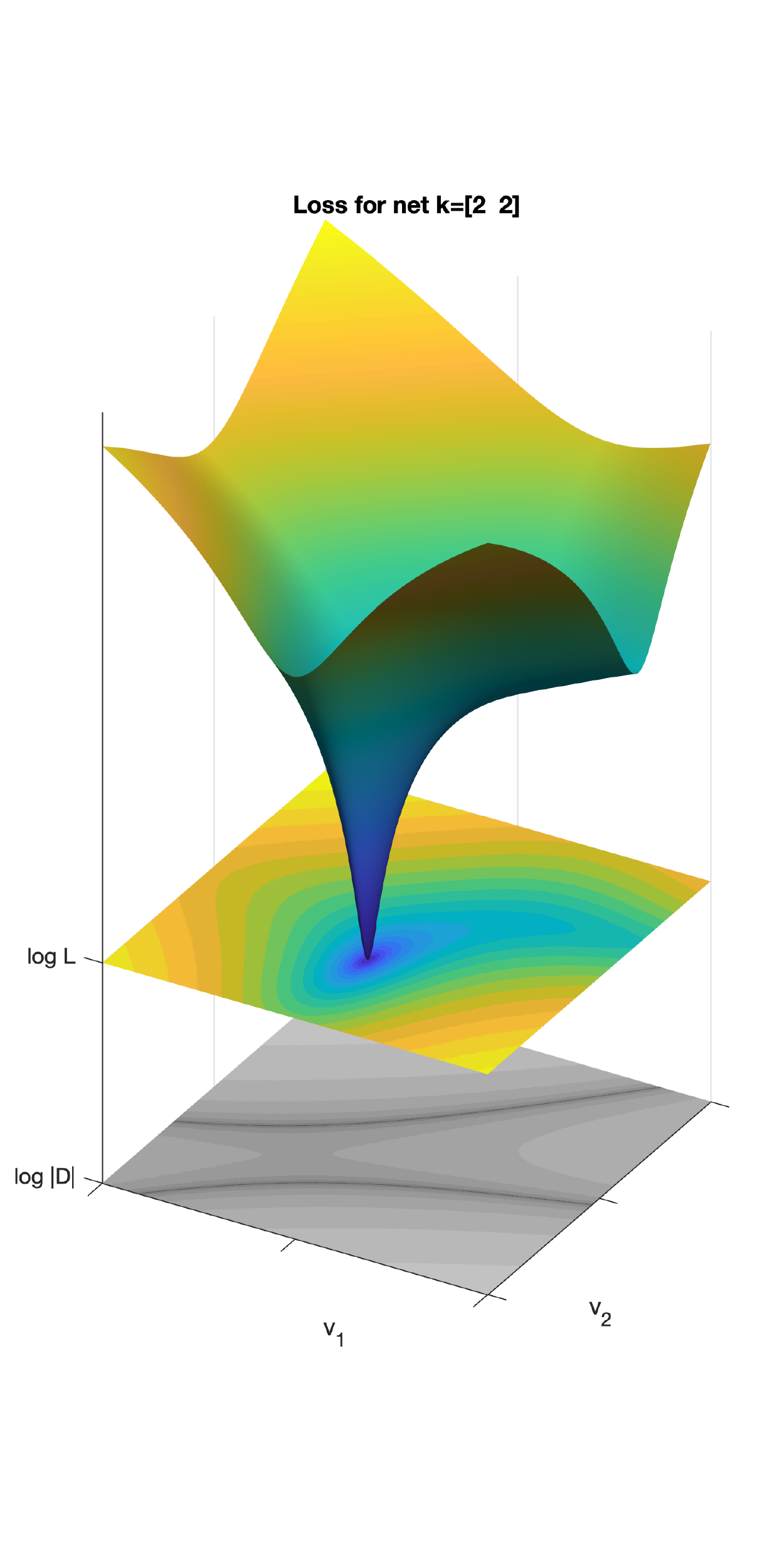}& 
\includegraphics[clip=true,trim=0cm 3.5cm 1cm 3.6cm,width=.2\textwidth]{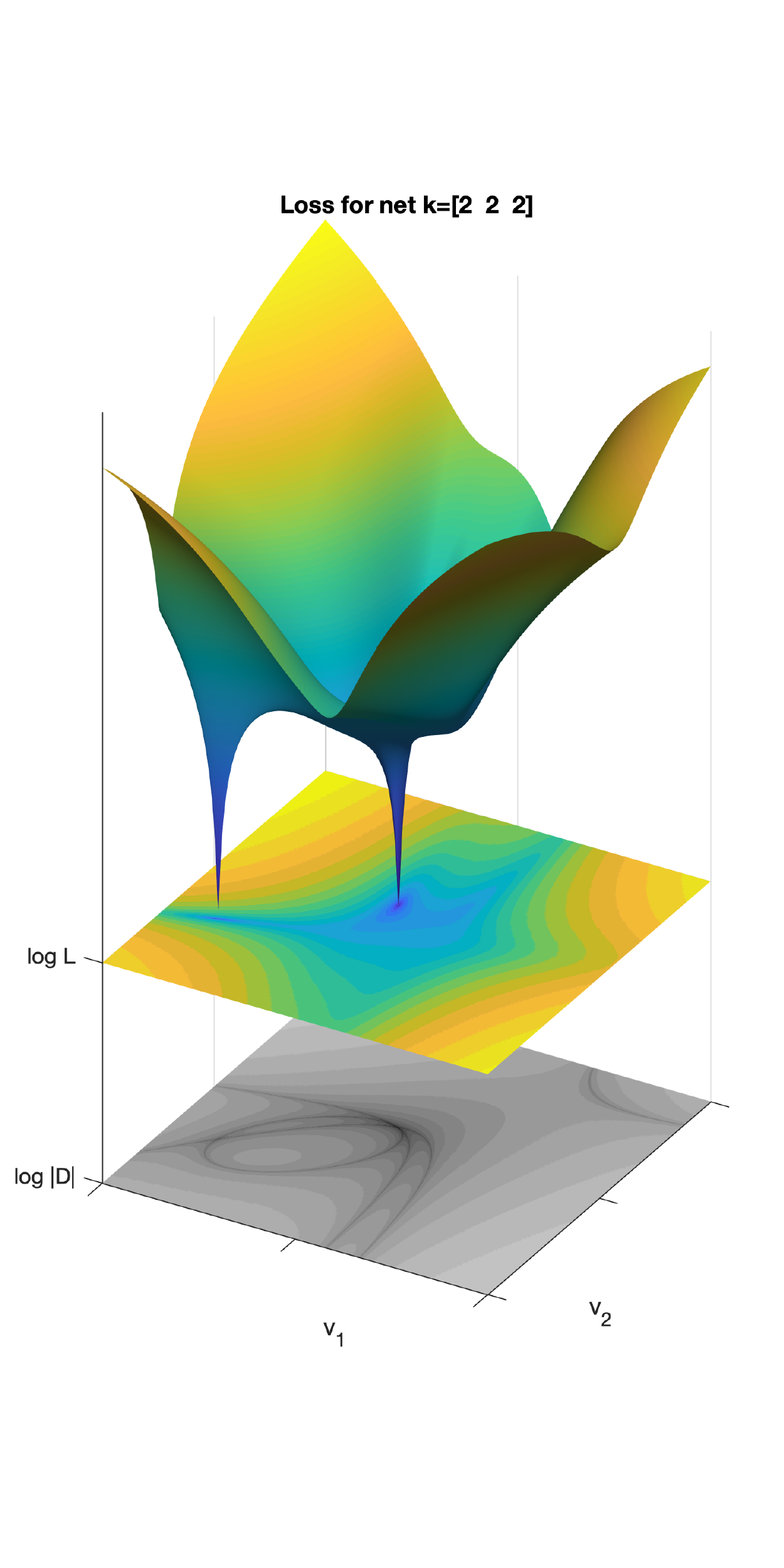}& 
\includegraphics[clip=true,trim=0cm 3.5cm 1cm 3.6cm,width=.2\textwidth]{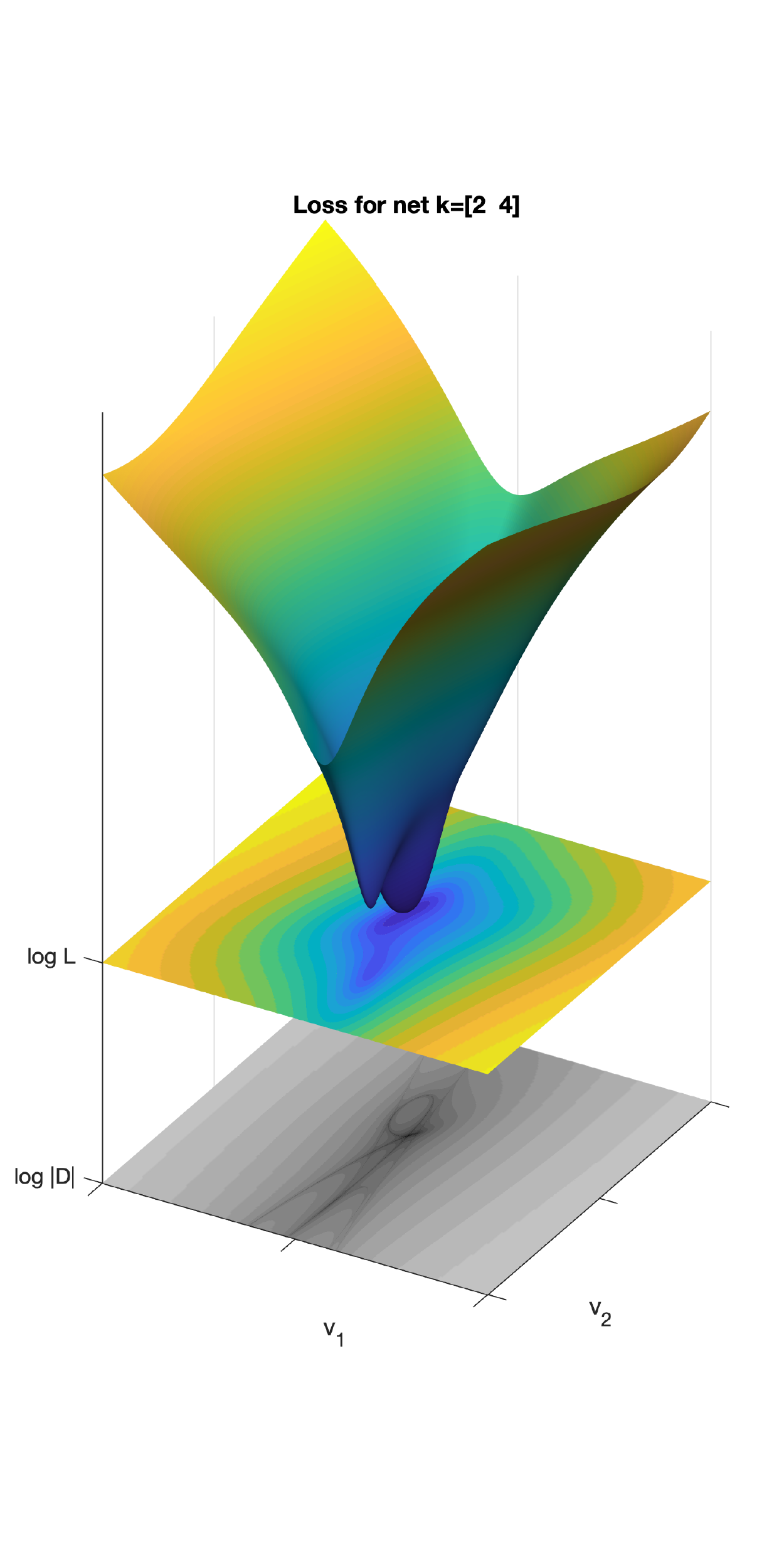}&
\includegraphics[clip=true,trim=0cm 3.5cm 1cm 3.6cm,width=.2\textwidth]{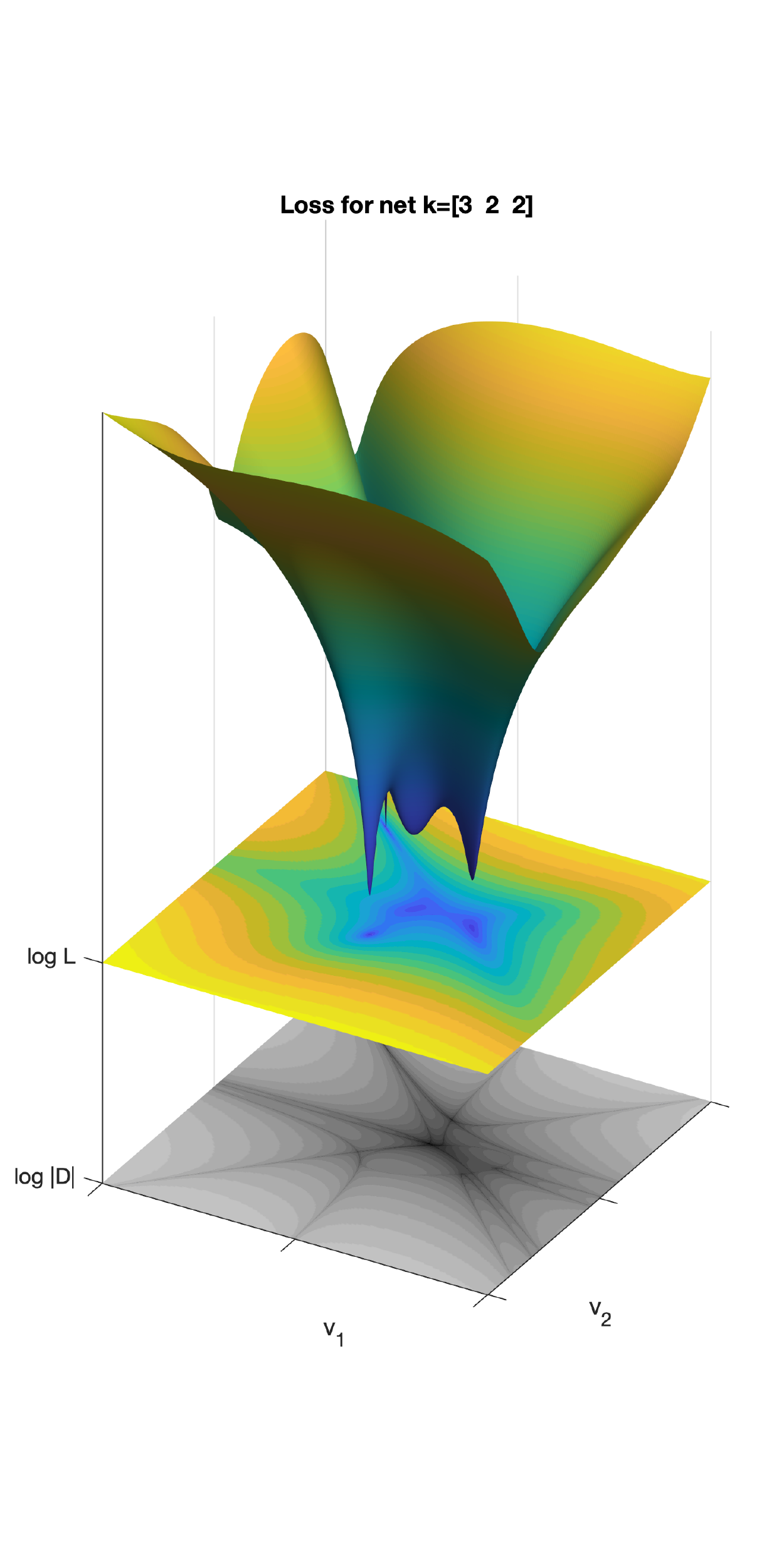}&
\includegraphics[clip=true,trim=0cm 3.5cm 1cm 3.6cm,width=.2\textwidth]{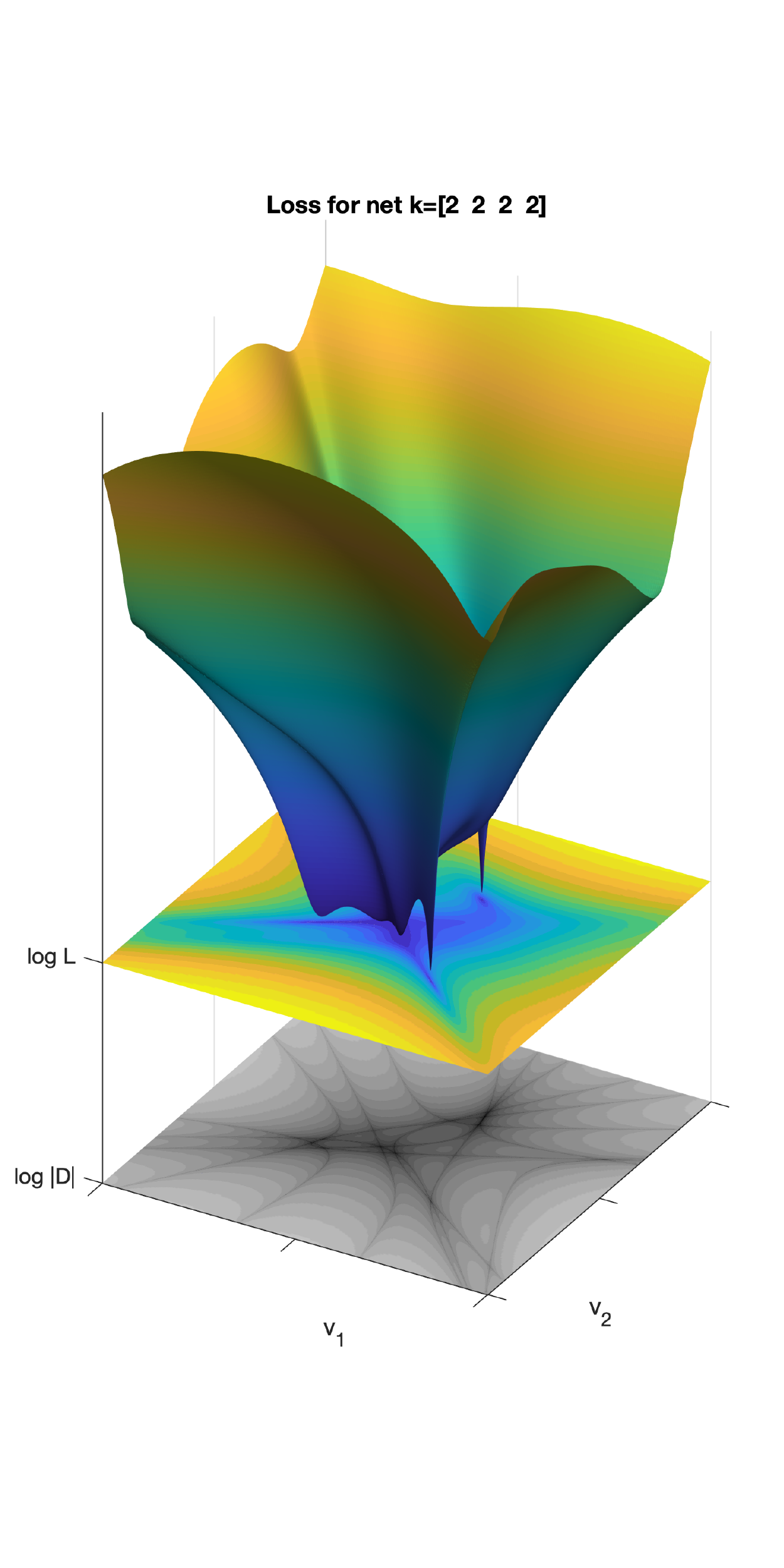}
\end{tabular}
\caption{The square loss $(w_1,\ldots, w_L) \mapsto \|U - \mu(w_1,\ldots, w_L)\|^2_{XX^\top}$, $U = YX^\top (XX^\top)^{-1}$, for random data $X,Y$, plotted in log-scale over random 2D affine subsets of parameter values. 
The bottom shows the discriminant at the corresponding filters, $(w_1,\ldots, w_L)\mapsto |\Delta(\pi(\mu(w_1,\ldots, w_L)))|$. 
As can be seen, deeper networks tend to have a more intricate loss surface and critical parameters often correspond to functions where the discriminant vanishes. 
}
\label{fig:landscapes}
\end{figure}

\subsection{Euclidean vs.\ Bombieri norm} 
In this experiment we generated 500 random targets $U$ for each architecture. 
For each of them, we ran gradient descent with 50 random initializations for the loss with Euclidean and  Bombieri norms. 
For each target we recorded the list of converged solutions, using the settings of the previous experiment. 
To determine the number of distinct solutions we considered two filters $w_i$ and $w_{i'}$ as equal if on each entry 
$|w_{ik}-w_{i'k}| \leq \text{\texttt{tol}}\cdot \max\{\max\{|w_{jk}|\colon j\},1\}$ with $\text{\texttt{tol}}=10^{-4}$. 
The results are shown in Table~\ref{table:compare_metric}. 
In line with our theoretical discussion in \Cref{sec:EDdegree} we observed fewer distinct solutions for the Bombieri norm than for the generic norm. 

\begin{table}
\caption{Percentages of targets for which gradient descent with different initializations converged to a given number of distinct solutions. 
For the Bombieri norm $B$ we usually observe a smaller number of distinct solutions compared with the generic Euclidean norm $I$. }
\label{table:compare_metric}
\centering 
\small 
\begin{tabular}{c c| S[table-format=3.1] S[table-format=2.2] S[table-format=2.2] S[table-format=1.3]}
            && \multicolumn{4}{c}{number of distinct solutions}\\
$\bm k$                 & metric & 1    & 2   & 3 & 4\\  
\hline 
$(2,2)$ & $I$    & 95.8 & 4.2 &   \\
                      & $B$    & 100  &     &   \\
\hline 
$(2,2,2)$ & $I$   & 84.7 & 14.5 & 0.84 \\
                        & $B$   & 95.6 & 1.1  & 3.29 \\
\hline 
$(4,2)$   & $I$   & 66.3 & 33.7 & \\
                        & $B$   & 73.8 & 26.2 & \\
\hline 
$(3,2,2)$ & $I$   & 22.5 & 58.8 & 17.7 & 1.01 \\
                        & $B$   & 32   & 48.3 & 19.5 & 0.209\\
\hline 
$(2,2,2,2)$ & $I$  & 70.7 & 26.5 & 2.61 & 0.217 \\
                          & $B$  & 84.8 & 8.21 & 6.97 &       \\
                          \hline 
\end{tabular}
\end{table}

\subsection{Symbolic computations}
We illustrate several of our results with a final example. We consider linearly generated data $\mathcal D_N = \{(x_i,y_i) \in \RR^5 \times \RR \colon i=1,\ldots,N, y_i = u^\top x_i\}$ where 
$u=[2,0,5,0,2]^\top \in \RR^5$
is a ``target filter'' and $x_i$ is sampled from a standard Gaussian distribution. We study the square loss regression problem for different 1D LCN architectures with compatible dimensions, as described in~\Cref{sec:squareloss}. In light of~\Cref{lem:euclidean_dist}, if the data is normalized so that the empirical covariance $XX^\top$ is the identity, or if $N \rightarrow \infty$, then this corresponds to an objective function $\mathcal L(\theta) = \ell(\mu(\theta)) =\|\mu(\theta) - u\|^2$, where $\overline w = \mu(\theta)$ is the end-to-end filter and $\|\cdot \|$ is the standard Euclidean norm. 

The possible LCN architectures with stride one, $d_0=5, d_L=1$ and {non-trivial filter sizes $1 < k_i < 5$ are} $\bm k = (4,2)$, $(3,3)$, $(3,2,2)$, $(2,2,2,2)$ (not considering permutations). By \Cref{thm:fillingCircular}, only the architecture $\bm k = (3,3)$ is filling. The target filter corresponds to the polynomial 
$\pi(u) = 2\x^4 + 5\x^2\y^2 + 2\y^4 = (\x^2+2\y^2)(2\x^2+\y^2)$, 
which does not have any real roots and is thus not contained in 
$\mathcal M_{(4,2)}$, $\mathcal M_{(3,2,2)}$, nor $\mathcal M_{(2,2,2,2)}$. 

Following our results in \Cref{sec:critPts}, for all architectures listed above, any critical point $\theta \in \Crit(\mathcal L)$ with $\mu(\theta) \ne u$ corresponds to polynomials with repeated roots. 
More precisely, if $\theta$ is a critical point for $\mathcal L$, then  $\mu(\theta)$ is a critical point for $\ell|_{\Delta_{\lambda}}$ where $\ell(\overline w) = \|\overline w - u\|^2$ and $\Delta_{\lambda}$ is the multiple root locus where $\lambda$ describes the {(complex)} root pattern of $\mu(\theta)$. Based on this, we can compute the critical points for $\ell|_{\Delta_{\lambda}}$ for any  $\lambda$ by solving algebraic systems.
Recall that the number of complex critical points was predicted by the ED degrees in \Cref{ex:EDdegreesGeneric}.

\begin{itemize}[leftmargin=*]
    \item $\lambda = (2,1,1)$: 
    out of the $10$ complex critical points on $\Delta_{\lambda}$, four are real and in fact rational, namely
    $[\frac{1}{5},\pm \frac{9}{5},\frac{16}{5},\pm \frac{9}{5},\frac{1}{5}]^\top {\equiv \frac{1}{5}(\x\pm \y)^2(\x^2 \pm 7\x\y +\y^2)}$, \,\, $ [0,0,5,0,2]^\top {\equiv \y^2(5\x^2+2\y^2)}$, and  $[2,0,5,0,0]^\top  \equiv \allowbreak \x^2(5\y^2+2\x^2)$.
    \item $\lambda = (3,1)$: 
    out of the $12$ complex critical points on $\Delta_{\lambda}$, four are real and none are rational. 
    \item $\lambda = (2,2)$: out of the $13$ complex critical points on $\Delta_{\lambda}$, five are real and three are rational:
    $[-1,0,2,0,-1]^\top \allowbreak  \equiv  \allowbreak  -(\x-\y)^2(\x+\y)^2$, 
    $[0,0,5,0,0]^\top \equiv  5 \x^2\y^2$, 
    $[\frac{7}{3},0,\frac{14}{3},0,\frac{7}{3}]^\top \equiv  \frac{7}{3} (\x^2+\y^2)^2$.
    \item $\lambda = (4)$: out of the $10$ complex critical points on $\Delta_{\lambda}$, four are real and also rational: $[0,0,0,0,2]^\top \equiv 2\y^4$, $[2,0,0,0,0]^\top \equiv  2\x^4$, $[\frac{17}{35},\pm \frac{68}{35},\frac{102}{35},\pm \frac{68}{35},\frac{17}{35}]^\top \equiv  \frac{17}{35}(\x\pm \y)^4$. 
\end{itemize}

For any choice of quartic architecture, all critical points are mapped to {either $u$ or} one of the $17$ real critical points listed above. 
We have verified numerically that gradient descent indeed always converges to one of these critical points. In fact, it always seems to converge to one of the rational critical points. For example, for $\bm k= (4,2)$, gradient descent converges to either
$[2,0,5,0,0]^\top$ or $[0,0,5,0,2]^\top$
which arise with $\lambda = (2,1,1)$.
Note that in this case we do not expect critical points with $\lambda = (3,1), (2,2), (4)$ due to~\Cref{ex:partitionVsArchitecture,cor:rrmp}.

Finally, we remark that the previous list only describes the critical points in function space but not the actual parameters. 
However, as discussed in~\Cref{sec:training-dynamics}, it is possible recover these parameters up to a finite ambiguity given the initialization of gradient descent. For example, assume that $\bm k = (4,2)$ and we initialize gradient descent at $\theta(0) = (w_1(0), w_2(0)) = ([1,6,11,6],[4,1])$, \ie, a factorization of the polynomial $(\x+\y)(\x+2\y)(\x+3\y) \cdot {(4\x+\y)}$.  
Then, following the approach outlined in \Cref{cor:invariants_system}, we can derive that if the gradient descent path $\theta(t)$ converges to $\mu(\theta)=[2,0,5,0,0]^\top$ in function space, then in parameter space $\theta(t)$ converges to one of $(w_1,w_2) = (1/\kappa[2,0,5,0]^\top, \kappa [1,0]^\top)$
with $\kappa = \pm \sqrt{\frac{1}{2} \left(\sqrt{31445} - {177}\right)} \approx \pm 0.4045867$. 
Numerically, we verify that gradient descent converges to these filters with $\kappa > 0$. 

\section{Conclusion}

We presented a semi-algebraic description of the families of functions that can be represented by linear convolutional networks. Using the fact that compositions of convolutions correspond to polynomial multiplication, we showed that the function space can be described as sets of polynomials that admit certain factorizations. We then investigated the optimization of an objective function over such sets and described the critical points in function space and in parameter space. Unlike fully-connected linear networks, convolutional linear networks can have non-global local optimizers in function space as well as spurious local optimizers in parameter space. In more detail, our analysis showed that all critical points that are not global minima correspond to polynomials with repeated roots, according to multiplicity patterns that depend on the network's architecture.
Possible extensions of this work might include a more detailed study of the function space for linear convolutions in higher dimensions or with larger strides (which we discussed briefly in Sections~\ref{sec:higher_dimension} and \ref{sec:largerStrides}) and, naturally, carrying out a similar analysis for \emph{non-linear} convolutional networks. In particular, one might first consider the special case of polynomial activation functions, since in this setting the function space is still a semi-algebraic set. 

\section*{Acknowledgment}
Kathl\'en Kohn was partially supported by the Knut and Alice Wallenberg Foundation within their WASP
(Wallenberg AI, Autonomous Systems and Software Program) AI/Math initiative. 
Guido Mont\'ufar and Thomas Merkh have been supported by the European Research Council (ERC) under the European Union’s Horizon 2020 research and innovation programme (grant n\textsuperscript{o}~757983). While at New York University, Matthew Trager was supported in part by Samsung Electronics.

\bibliographystyle{alpha}
\bibliography{M144118-literatur}

\begin{thebibliography}{10}

\bibitem{pmlr-v97-allen-zhu19a}
{\sc Z.~Allen-Zhu, Y.~Li, and Z.~Song}, {\em A convergence theory for deep
  learning via over-parameterization}, in Proceedings of the 36th International
  Conference on Machine Learning, K.~Chaudhuri and R.~Salakhutdinov, eds.,
  vol.~97 of Proceedings of Machine Learning Research, PMLR, 09--15 Jun 2019,
  pp.~242--252, \url{http://proceedings.mlr.press/v97/allen-zhu19a.html}.

\bibitem{722fd1ac5816467ebadbd10f15ff2a74}
{\sc E.~Allman, H.~Banos, R.~Evans, S.~Hosten, K.~Kubjas, D.~Lemke, J.~Rhodes,
  and P.~Zwiernik}, {\em Maximum likelihood estimation of the latent class
  model through model boundary decomposition}, Journal of Algebraic Statistics,
  10 (2019), pp.~51--84, \url{https://doi.org/10.18409/jas.v10i1.75}.

\bibitem{ARNON198837}
{\sc D.~S. Arnon}, {\em Geometric reasoning with logic and algebra}, Artificial
  Intelligence, 37 (1988), pp.~37--60,
  \url{https://doi.org/https://doi.org/10.1016/0004-3702(88)90049-5},
  \url{https://www.sciencedirect.com/science/article/pii/0004370288900495}.

\bibitem{arora2018a}
{\sc S.~Arora, N.~Cohen, N.~Golowich, and W.~Hu}, {\em A convergence analysis
  of gradient descent for deep linear neural networks}, in International
  Conference on Learning Representations, 2019,
  \url{https://openreview.net/forum?id=SkMQg3C5K7}.

\bibitem{pmlr-v80-arora18a}
{\sc S.~Arora, N.~Cohen, and E.~Hazan}, {\em On the optimization of deep
  networks: Implicit acceleration by overparameterization}, vol.~80 of
  Proceedings of Machine Learning Research, Stockholmsmässan, Stockholm
  Sweden, 10--15 Jul 2018, PMLR, pp.~244--253,
  \url{http://proceedings.mlr.press/v80/arora18a.html}.

\bibitem{DBLP:journals/corr/abs-1910-05505}
{\sc B.~Bah, H.~Rauhut, U.~Terstiege, and M.~Westdickenberg}, {\em Learning
  deep linear neural networks: Riemannian gradient flows and convergence to
  global minimizers}, CoRR, abs/1910.05505 (2019),
  \url{http://arxiv.org/abs/1910.05505},
  \url{https://arxiv.org/abs/1910.05505}.

\bibitem{NIPS1988_123}
{\sc P.~Baldi}, {\em Linear learning: Landscapes and algorithms}, in Advances
  in Neural Information Processing Systems 1, D.~S. Touretzky, ed.,
  Morgan-Kaufmann, 1989, pp.~65--72,
  \url{http://papers.nips.cc/paper/123-linear-learning-landscapes-and-algorithms.pdf}.

\bibitem{BALDI198953}
{\sc P.~Baldi and K.~Hornik}, {\em Neural networks and principal component
  analysis: Learning from examples without local minima}, Neural Networks, 2
  (1989), pp.~53--58,
  \url{https://doi.org/https://doi.org/10.1016/0893-6080(89)90014-2},
  \url{https://www.sciencedirect.com/science/article/pii/0893608089900142}.

\bibitem{10.1109/72.392248}
{\sc P.~Baldi and K.~Hornik}, {\em Learning in linear neural networks: A
  survey}, Trans. Neur. Netw., 6 (1995), p.~837–858,
  \url{https://doi.org/10.1109/72.392248},
  \url{https://doi.org/10.1109/72.392248}.

\bibitem{10.5555/1197095}
{\sc S.~Basu, R.~Pollack, and M.-F. Roy}, {\em Algorithms in Real Algebraic
  Geometry (Algorithms and Computation in Mathematics)}, Springer-Verlag,
  Berlin, Heidelberg, 2006.

\bibitem{BlinnHowto}
{\sc J.~F. Blinn}, {\em How to solve a cubic equation, part 1: The shape of the
  discriminant}, IEEE Computer Graphics and Applications, 26 (2006),
  pp.~84--93, \url{https://doi.org/10.1109/MCG.2006.60}.

\bibitem{10.1007/978-3-030-43120-4_29}
{\sc T.~{\"O}. {\c{C}}elik, A.~Jamneshan, G.~Mont{\'u}far, B.~Sturmfels, and
  L.~Venturello}, {\em Optimal transport to a variety}, in Mathematical Aspects
  of Computer and Information Sciences, D.~Slamanig, E.~Tsigaridas, and
  Z.~Zafeirakopoulos, eds., Cham, 2020, Springer International Publishing,
  pp.~364--381.

\bibitem{ccelik2020wasserstein}
{\sc T.~{\"O}. {\c{C}}elik, A.~Jamneshan, G.~Mont\'{u}far, B.~Sturmfels, and
  L.~Venturello}, {\em Wasserstein distance to independence models}, Journal of
  Symbolic Computation, 104 (2021), pp.~855--873,
  \url{https://doi.org/https://doi.org/10.1016/j.jsc.2020.10.005},
  \url{https://www.sciencedirect.com/science/article/pii/S0747717120301152}.

\bibitem{8811622}
{\sc Y.~Chi, Y.~M. Lu, and Y.~Chen}, {\em Nonconvex optimization meets low-rank
  matrix factorization: An overview}, IEEE Transactions on Signal Processing,
  67 (2019), pp.~5239--5269, \url{https://doi.org/10.1109/TSP.2019.2937282}.

\bibitem{draisma2016euclidean}
{\sc J.~Draisma, E.~Horobe{\c{t}}, G.~Ottaviani, B.~Sturmfels, and R.~R.
  Thomas}, {\em The {E}uclidean distance degree of an algebraic variety},
  Foundations of computational mathematics, 16 (2016), pp.~99--149.

\bibitem{du2018gradient}
{\sc S.~S. Du, X.~Zhai, B.~Poczos, and A.~Singh}, {\em Gradient descent
  provably optimizes over-parameterized neural networks}, in International
  Conference on Learning Representations, 2019,
  \url{https://openreview.net/forum?id=S1eK3i09YQ}.

\bibitem{pmlr-v119-dukler20a}
{\sc Y.~Dukler, Q.~Gu, and G.~Mont\'{u}far}, {\em Optimization theory for
  {R}e{LU} neural networks trained with normalization layers}, in Proceedings
  of the 37th International Conference on Machine Learning, H.~{Daumé III} and
  A.~Singh, eds., vol.~119 of Proceedings of Machine Learning Research, PMLR,
  13--18 Jul 2020, pp.~2751--2760,
  \url{http://proceedings.mlr.press/v119/dukler20a.html}.

\bibitem{Gonzalez2015RootCO}
{\sc E.~Gonzalez and D.~Weinberg}, {\em Root configurations of real univariate
  cubics and quartics}, arXiv: Commutative Algebra,  (2015).

\bibitem{m2}
{\sc D.~Grayson and M.~Stillman}, {\em Macaulay 2, a system for computation in
  algebraic geometry and commutative algebra}.
\newblock faculty.math.illinois.edu/Macaulay2/.

\bibitem{dft}
{\sc S.~Gunasekar, J.~D. Lee, D.~Soudry, and N.~Srebro}, {\em Implicit bias of
  gradient descent on linear convolutional networks}, in Advances in Neural
  Information Processing Systems 31, S.~Bengio, H.~Wallach, H.~Larochelle,
  K.~Grauman, N.~Cesa-Bianchi, and R.~Garnett, eds., Curran Associates, Inc.,
  2018, pp.~9461--9471,
  \url{http://papers.nips.cc/paper/8156-implicit-bias-of-gradient-descent-on-linear-convolutional-networks.pdf}.

\bibitem{DBLP:conf/iclr/HardtM17}
{\sc M.~Hardt and T.~Ma}, {\em Identity matters in deep learning}, in 5th
  International Conference on Learning Representations, {ICLR} 2017, Toulon,
  France, April 24-26, 2017, Conference Track Proceedings, OpenReview.net,
  2017, \url{https://openreview.net/forum?id=ryxB0Rtxx}.

\bibitem{hartshorne}
{\sc R.~Hartshorne}, {\em Algebraic geometry}, vol.~52, Springer Science \&
  Business Media, 2013.

\bibitem{NEURIPS2018_5a4be1fa}
{\sc A.~Jacot, F.~Gabriel, and C.~Hongler}, {\em Neural tangent kernel:
  Convergence and generalization in neural networks}, in Advances in Neural
  Information Processing Systems, S.~Bengio, H.~Wallach, H.~Larochelle,
  K.~Grauman, N.~Cesa-Bianchi, and R.~Garnett, eds., vol.~31, Curran
  Associates, Inc., 2018,
  \url{https://proceedings.neurips.cc/paper/2018/file/5a4be1fa34e62bb8a6ec6b91d2462f5a-Paper.pdf}.

\bibitem{NIPS2016_6112}
{\sc K.~Kawaguchi}, {\em Deep learning without poor local minima}, in Advances
  in Neural Information Processing Systems 29, D.~D. Lee, M.~Sugiyama, U.~V.
  Luxburg, I.~Guyon, and R.~Garnett, eds., Curran Associates, Inc., 2016,
  pp.~586--594,
  \url{http://papers.nips.cc/paper/6112-deep-learning-without-poor-local-minima.pdf}.

\bibitem{kurmann2012some}
{\sc S.~Kurmann}, {\em Some remarks on equations defining coincident root
  loci}, Journal of Algebra, 352 (2012), pp.~223--231.

\bibitem{pmlr-v80-laurent18a}
{\sc T.~Laurent and J.~von Brecht}, {\em Deep linear networks with arbitrary
  loss: All local minima are global}, in Proceedings of the 35th International
  Conference on Machine Learning, J.~Dy and A.~Krause, eds., vol.~80 of
  Proceedings of Machine Learning Research, Stockholmsmässan, Stockholm
  Sweden, 10--15 Jul 2018, PMLR, pp.~2902--2907,
  \url{http://proceedings.mlr.press/v80/laurent18a.html}.

\bibitem{lee2016duality}
{\sc H.~Lee and B.~Sturmfels}, {\em Duality of multiple root loci}, Journal of
  Algebra, 446 (2016), pp.~499--526.

\bibitem{1027497}
{\sc C.~{Lennerz} and E.~{Schomer}}, {\em Efficient distance computation for
  quadratic curves and surfaces}, in Geometric Modeling and Processing. Theory
  and Applications. GMP 2002. Proceedings, 2002, pp.~60--69.

\bibitem{DBLP:journals/corr/LuK17}
{\sc H.~Lu and K.~Kawaguchi}, {\em Depth creates no bad local minima}, CoRR,
  abs/1702.08580 (2017), \url{http://arxiv.org/abs/1702.08580},
  \url{https://arxiv.org/abs/1702.08580}.

\bibitem{9397294}
{\sc D.~Mehta, T.~Chen, T.~Tang, and J.~Hauenstein}, {\em The loss surface of
  deep linear networks viewed through the algebraic geometry lens}, IEEE
  Transactions on Pattern Analysis and Machine Intelligence,  (2021), pp.~1--1,
  \url{https://doi.org/10.1109/TPAMI.2021.3071289}.

\bibitem{DBLP:journals/corr/SaxeMG13}
{\sc A.~M. Saxe, J.~L. McClelland, and S.~Ganguli}, {\em Exact solutions to the
  nonlinear dynamics of learning in deep linear neural networks}, in 2nd
  International Conference on Learning Representations, {ICLR} 2014, Banff, AB,
  Canada, April 14-16, 2014, Conference Track Proceedings, Y.~Bengio and
  Y.~LeCun, eds., 2014, \url{http://arxiv.org/abs/1312.6120}.

\bibitem{SeigalMontufar}
{\sc A.~Seigal and G.~Mont\'{u}far}, {\em {Mixtures and products in two
  graphical models}}, Journal of algebraic statistics, 9 (2018), pp.~1--20,
  \url{https://doi.org/10.18409/jas.v9i1.90}.

\bibitem{JMLR:v19:18-188}
{\sc D.~Soudry, E.~Hoffer, M.~S. Nacson, S.~Gunasekar, and N.~Srebro}, {\em The
  implicit bias of gradient descent on separable data}, Journal of Machine
  Learning Research, 19 (2018), pp.~1--57,
  \url{http://jmlr.org/papers/v19/18-188.html}.

\bibitem{sullivant2018algebraic}
{\sc S.~Sullivant}, {\em Algebraic Statistics}, Graduate Studies in
  Mathematics, American Mathematical Society, 2018,
  \url{https://books.google.de/books?id=EXj1uQEACAAJ}.

\bibitem{geometryLinearNets}
{\sc M.~Trager, K.~Kohn, and J.~Bruna}, {\em Pure and spurious critical points:
  a geometric study of linear networks}, in International Conference on
  Learning Representations, 2020,
  \url{https://openreview.net/forum?id=rkgOlCVYvB}.

\bibitem{williams2019gradient}
{\sc F.~Williams, M.~Trager, D.~Panozzo, C.~Silva, D.~Zorin, and J.~Bruna},
  {\em Gradient dynamics of shallow univariate {ReLU} networks}, in Advances in
  Neural Information Processing Systems, H.~Wallach, H.~Larochelle,
  A.~Beygelzimer, F.~d\textquotesingle Alch\'{e}-Buc, E.~Fox, and R.~Garnett,
  eds., vol.~32, Curran Associates, Inc., 2019,
  \url{https://proceedings.neurips.cc/paper/2019/file/1f6419b1cbe79c71410cb320fc094775-Paper.pdf}.

\bibitem{47812}
{\sc L.~Zhang}, {\em Depth creates no more spurious local minima}, 2019,
  \url{https://arxiv.org/pdf/1901.09827.pdf}.

\bibitem{zhou2018critical}
{\sc Y.~Zhou and Y.~Liang}, {\em Critical points of linear neural networks:
  Analytical forms and landscape properties}, in International Conference on
  Learning Representations, 2018,
  \url{https://openreview.net/forum?id=SysEexbRb}.

\end{thebibliography}

\newpage 

\appendix

\section{Gradients}
For convenience we give a description of the gradient for LCNs. 
\begin{proposition}
Consider a loss $\ell\colon \mathbb{R}^{d_L\times d_0}\to\mathbb{R}$ on the space of matrices. 
For an LCN parametrization $\mu\colon (w_1,\ldots, w_L)\mapsto \overline{W}$ as in \cref{eqn:Wbar}, denote the loss on the space of 
filters by $\L = \ell \circ \mu$. 
The gradient is 
\begin{align*} 
(\nabla_{w_l}\L(w_1,\ldots, w_L))_\alpha =& \sum_{\beta, \gamma} (\nabla_{w_l}W_l)_{\alpha,\beta,\gamma}  
(W_{l+1}^\top\cdots W_L^\top\nabla_{\overline{W}}\ell(\overline{W}) W_1^\top\cdots W_{l-1}^\top)_{\beta,\gamma}, 
\end{align*}
where $(\nabla_{w_l}W_l)_{\alpha,\beta,\gamma} = \partial_{w_{l,\alpha}} (W_l)_{\beta,\gamma}$. 
For the square loss $\ell(\overline{W}) = \|\overline{W}X - Y\|^2$ 
with data $X,Y$, $\nabla_{\overline{W}}\ell(\overline{W}) = \overline{W} XX^\top - YX^\top$. 
\end{proposition} 

The previous result describes the gradient of the loss $\L$ in terms of convolutional matrices $W_l$, however the expression does not depend on the actual matrices but only on the filters involved. We can make this fact explicit in the case of stride $1$ as follows. 

\begin{proposition}
Let $\L = \ell \circ \mu$ be a loss function for an LCN with stride $1$ and let $\theta = (w_1,\ldots,w_L)$ denote the filters. For any vector $w \in \RR^k$, let $\pi_{-1}(w):=w_0 \x^{-k+1} + w_1  \x^{-k+2} +\cdots + w_{k-2} \x^{-1} + w_{k-1}$ be the Laurent polynomial (polynomial with negative exponents) whose coefficients are the elements of $w$. For any $i=1,\ldots,L$, the gradient $\nabla_{w_i} \L(\theta) \in \RR^{k_i}$ is given by the coefficients of degree $(0,\ldots,k_i-1)$ in
\[ \tilde
\pi(\nabla \ell) \cdot \pi_{-1}(w_1) \cdots \pi_{-1}(w_{i-1}) \cdot \pi_{-1}(w_{i+1}) \cdots \pi_{-1}(w_L) \in \RR[\x,\x^{-1}],
\]
where $\tilde \pi(\nabla \ell)$ is the univariate polynomial representation of the gradient $\nabla \ell(\mu(\theta))$ obtained from \eqref{eq:polynomial_identification} by dehomogenizing  $\y=1$.
\end{proposition}
\begin{proof} The Laurent polynomial $\pi_{-1}(w)$ is a representation of the \emph{adjoint filter} $w^* := (w_{k-1},\ldots,w_0)$. Convolution with the adjoint filter $w^*$ is the adjoint operator of the convolution with $w$, for the standard Euclidean inner product. In terms of polynomials, this means that
\[
\langle \tilde \pi(v), \tilde\pi(w_2) \cdot \tilde\pi(w_1) \rangle = \langle \tilde\pi(v) \cdot \pi_{-1}(w_2), \tilde\pi(w_1)\rangle,
\]
for any vectors $v, w_1, w_2$, where $\langle \cdot, \cdot \rangle$ is the coefficient-wise inner-product and $\cdot$ is product operation for polynomials (missing coefficients are treated as zero so we need not constrain the sizes of $v, w_1, w_2$). The claim now follows from the fact that $\nabla_{w_i} \L$ is such that
\[
\begin{aligned}
\langle \tilde\pi(\nabla_{w_i} \L), \tilde\pi(\dot w_i) \rangle &= \langle \tilde \pi(\nabla \ell), \tilde \pi(w_1) \cdots \tilde \pi(\dot w_i) \cdots \tilde \pi(w_L) \rangle\\ 
&=  \langle \tilde \pi(\nabla \ell) \cdot \pi_{-1}(w_1) \cdots \pi_{-1}(w_{i-1}) \cdot \pi_{-1}(w_{i+1}) \cdots \pi_{-1}(w_L),\tilde \pi(\dot w_i)\rangle. 
\end{aligned}
\]
\end{proof}

\end{document}